\documentclass[letterpaper]{article}
\usepackage[preprint]{aaai2027}
\usepackage[hyphens]{url}
\usepackage{graphicx}
\usepackage{natbib}
\usepackage{caption}
\usepackage{algorithm}
\usepackage{algorithmic}
\usepackage{multirow}
\usepackage{mathtools}
\usepackage{amsmath}
\usepackage{flafter}
\usepackage{subcaption}
\usepackage{enumitem}
\usepackage{amssymb}
\usepackage{array}
\usepackage[table]{xcolor}
\usepackage{amsthm}
\usepackage{enumitem}
\usepackage{xcolor}
\usepackage[table]{xcolor}

\usepackage{listings}

\lstdefinestyle{cordpython}{
  language=Python,
  basicstyle=\small\ttfamily,
  identifierstyle=\color{black},
  keywordstyle=\color{blue!55!black}\bfseries,
  commentstyle=\color{green!35!black}\itshape,
  stringstyle=\color{orange!50!black},
  columns=fullflexible,
  keepspaces=true,
  showstringspaces=false,
  breaklines=true,
  breakatwhitespace=false,
  tabsize=4,
  numbers=left,
  numberstyle=\tiny\color{black!45},
  numbersep=6pt,
  stepnumber=1,
  backgroundcolor=\color{black!2},
  frame=single,
  rulecolor=\color{black!20},
  framerule=0.4pt,
  framesep=5pt,
  xleftmargin=1.5em,
  framexleftmargin=1em,
  aboveskip=0.7\baselineskip,
  belowskip=0.7\baselineskip
}

\newtheorem{proposition}{Proposition}

\usepackage{newfloat}
\usepackage{listings}
\DeclareCaptionStyle{ruled}{labelfont=normalfont,labelsep=colon,strut=off}
\floatstyle{ruled}
\newfloat{listing}{tb}{lst}{}
\floatname{listing}{Listing}

\usepackage{booktabs}

\title{Let Confidence Change, Not the Prediction: Prediction-Preserving Repair for Post-hoc Calibration}

\author{
    Daehwan Kim\textsuperscript{\rm 1},
    Ikbeom Jang\textsuperscript{\rm 2\textdagger}
}
\affiliations{
    \textsuperscript{\rm 1}Hanyang University, Seoul, Republic of Korea\\
    \textsuperscript{\rm 2}Hankuk University of Foreign Studies, Yongin, Republic of Korea\\
    officialhwan@hanyang.ac.kr, ijang@hufs.ac.kr
}

\begin{document}

\maketitle
\begingroup
\renewcommand{\thefootnote}{\textdagger}
\footnotetext{Corresponding author.}
\endgroup

\begin{abstract}
Post-hoc calibration corrects reported confidence, yet a multiclass calibrator can also change the associated top-1 prediction. Accuracy captures only the net effect of these changes on correctness, not how often predictions change; the Top-1 Prediction Change Rate (TPCR) instead measures this frequency. We propose Calibrator-Output Repair for Top-1 Decision Preservation (CORD), the first post-fit adapter to impose exact prediction preservation by repairing the full calibrated probability vector. From the original and calibrated outputs alone, CORD determines the mass assigned to the original top-1. The calibrated conditional distribution allocates the remaining mass over the other classes, yielding a repaired vector whose own argmax recovers the original prediction. On the calibration split, CORD coordinates the repaired masses to retain the calibrated outputs' mean mass on original predictions whenever attainable. The adapter alters neither the fitted calibrator nor its direct output, fits no additional supervised map, and requires no user- or validation-tuned hyperparameter. Across CIFAR-10/100 and ImageNet-1K, CORD attains zero TPCR by construction and lowers mean ECE, NLL, and Brier relative to the corresponding direct outputs in every dataset; paired gains persist under distribution shift and across calibration-set sizes. CORD thus removes the preservation constraint from calibrator fitting and assigns exact recovery of the original decision to subsequent output repair. Our code is available at \url{https://github.com/labhai/CORD}.
\end{abstract}

\section{Introduction}
\label{sec:introduction}

\begin{figure}[t!]
    \centering
    \includegraphics[width=0.9\columnwidth]{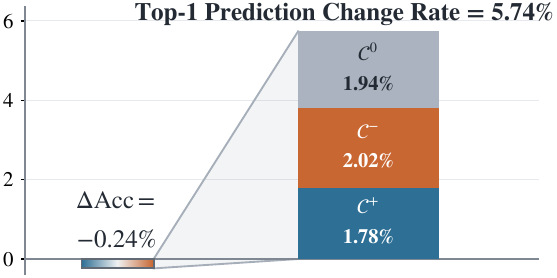}
    \caption{Accuracy hides the extent of top-1 prediction changes. For Vector Scaling on ImageNet-1K with ResNet-50, accuracy change reflects the net balance of accuracy-improving ($\mathcal{C}^{+}$) and accuracy-degrading ($\mathcal{C}^{-}$) revisions, whereas TPCR in Equation~\eqref{eq:tpcr} counts all revisions, including accuracy-neutral changes between incorrect classes ($\mathcal{C}^{0}$).}
    \label{fig:decision-revisions}
\end{figure}

Post-hoc calibration leaves the parameters of a trained classifier fixed and fits a separate map to its outputs. Its goal is confidence calibration, whereby predictions reported with confidence $p$ are correct with frequency $p$~\citep{guo2017calibration}. In standard multiclass prediction, a single probability vector both selects the predicted class through its argmax and encodes the corresponding confidence in the value at that coordinate. A post-hoc map that modifies this vector can therefore change not only the reported confidence but also the final top-1 prediction. Some calibration maps retain the original classifier's top-1~\citep{guo2017calibration,zhang2020mix,rahimi2020intra,tomani2022parameterized}, whereas others may change it~\citep{guo2017calibration,kull2019beyond}. A calibrated vector used as the final predictive distribution associates its confidence with the class selected by its own argmax; when the calibrated top-1 changes, that class differs from the original classifier's choice. The resulting report is coherent for the composed predictor, but the object of calibration has shifted. 

Preventing this shift at fit time requires constraining the calibration map to preserve the original top-1, thereby narrowing the class of admissible probability transformations. Mix-n-Match~\citep{zhang2020mix} makes this tension explicit by identifying accuracy preservation and high expressive power as calibration desiderata and illustrating how greater expressive power can come at the cost of classification accuracy. Measured as an accuracy change, the resulting compromise can appear small, yet accuracy records only the net effect on correctness of changes to the original top-1 predictions, not their extent. Applying Vector Scaling~\citep{guo2017calibration} to a pretrained ResNet-50 on ImageNet-1K reduces accuracy by only $0.24$ percentage points yet changes the original top-1 prediction for $5.74\%$ of examples (Figure~\ref{fig:decision-revisions}). Only ground-truth labels reveal how each revision affects correctness. Here, revisions with opposing effects nearly offset one another, while those between incorrect classes make no contribution to the accuracy change. The net effect is small; the extent is not.

We therefore ask a different question---\emph{What if the preservation constraint were removed from fitting and imposed afterward?} We open this route with \textbf{C}alibrator-\textbf{O}utput \textbf{R}epair for Top-1 \textbf{D}ecision Preservation (\textbf{CORD}), the first post-fit adapter to impose exact prediction preservation by repairing a fitted calibrator's full probability vector. Using only the original and calibrated outputs, CORD constructs a repaired probability vector whose argmax matches the original prediction. It leaves both the fitted calibrator and its direct output unchanged, fits no additional supervised map, and introduces no user- or validation-tuned hyperparameter. The calibrator is thus fitted without a preservation constraint; prediction preservation is imposed only afterward in constructing the repaired output and is no longer determined by the choice of calibrator.

Moving preservation outside fitting requires a principled repair, as the calibrated vector is the output of a fitted calibration map and an unrestricted rewrite would be indistinguishable from an arbitrary override. CORD therefore changes only how probability mass is split between the originally predicted class and all remaining classes, restoring the original top-1 while leaving the calibrated vector's relative allocation among those classes unchanged. Because independent pointwise repairs can collectively shift the mean mass assigned to the original predictions, CORD coordinates them over the calibration split to retain the calibrated outputs' mean mass on those predictions whenever attainable. One fitted calibrator thereby yields two normalized reports, with the direct output describing the prediction induced by its own argmax and the repaired output describing the original classifier's prediction. CORD thus allows the reported confidence to change while keeping the prediction it describes fixed. Overall, we make the following contributions:
\begin{itemize}
\item We expose how multiclass calibration can change the top-1 prediction that the reported confidence describes; accuracy records only the net effect of such changes on correctness, whereas the Top-1 Prediction Change Rate (TPCR) captures their total incidence.

\item We introduce CORD, the first post-fit adapter to impose exact prediction preservation by repairing a fitted calibrator's full probability report. It preserves the calibrated conditional distribution over the remaining classes without auxiliary supervised fitting or a user- or validation-tuned hyperparameter.

\item We demonstrate across diverse datasets, classifiers, and calibrator families that CORD attains zero TPCR by construction while lowering ECE, NLL, and Brier on average relative to the direct outputs, with gains persisting under distribution shift and across calibration-set sizes.
\end{itemize}

\section{Related Work}
\label{sec:related-work}
\noindent\textbf{Fit-time prediction preservation.}
Fit-time approaches build argmax or order preservation into the fitted calibration mechanism. Temperature Scaling (TS)~\citep{guo2017calibration} retains the complete class ordering through a single positive temperature; Mix-n-Match~\citep{zhang2020mix} introduces multiclass isotonic regression (IRM) with one strictly isotonic map shared across classes; Intra Order-Preserving Functions~\citep{rahimi2020intra} learn order-preserving neural calibration maps. Parameterized Temperature Scaling (PTS)~\citep{tomani2022parameterized} uses an input-dependent positive scalar temperature; Sample-Dependent Adaptive Temperature Scaling (AdaTS)~\citep{joy2023sample} predicts it from class-conditional latent likelihoods produced by a variational autoencoder over the classifier's feature space, whereas concurrent Quantile-Adaptive Temperature Scaling (QaTS)~\citep{chakraborty2026quantile} conditions it on the empirical quantile of the original confidence. Probability Bounding (PB)~\citep{atarashi2025box} fits uniform lower and upper probability bounds for Box-Constrained Softmax (BCSoftmax), retaining input-logit ordering but potentially introducing top-class ties. MCCT-I~\citep{zhang2025instance} fits rank-dependent inverse scales and biases to sorted logits under monotonicity constraints; recent work in semantic segmentation~\citep{kirscher2026rethinking} fits class-conditional affine calibrators under argmax- or order-preservation constraints. Outside calibrator fitting, accuracy-preserving Truth Discovery Ensemble (aTDE)~\citep{ma2021improving} projects truth-discovery iterates during aggregation onto a simplex region preserving the ensemble's top-1; CORD imposes exact top-1 preservation only when constructing the repaired output.

\noindent\textbf{Calibration maps without preservation guarantees.}
Other multiclass calibrators produce normalized vectors that need not retain the original prediction. Vector Scaling (VS) and Matrix Scaling (MS)~\citep{guo2017calibration} use diagonal and dense logit-affine maps, respectively; Structured Vector Scaling (SVS) and Structured Matrix Scaling (SMS)~\citep{berta2025structured} apply hierarchical regularization to the corresponding affine families; Dirichlet Calibration with Off-Diagonal and Intercept Regularization (Dir-ODIR)~\citep{kull2019beyond} fits a regularized affine map in log-probability space. Beyond these affine families, one-versus-all isotonic regression (IROvA) and IROvA-TS~\citep{zhang2020mix} fit classwise monotone maps before normalization, the latter after Temperature Scaling. Meta-Cal~\citep{ma2021meta} combines a base calibrator with a ranking model to control miscoverage or coverage accuracy, without guaranteeing pointwise prediction preservation. CORD instead complements each direct output with a repaired probability vector that restores the original top-1 without altering the calibrated conditional distribution over the remaining classes.

\noindent\textbf{Confidence calibration for fixed predictions.}
Another line fixes the predicted label; top-label calibration~\citep{gupta2022top} requires confidence calibration conditional on that label. Top-versus-All (TvA)~\citep{le2024confidence} treats the original prediction's correctness as a binary calibration problem and, with a binary calibrator, acts after class selection; its TS instantiation (TS--TvA) retains class ordering through a shared positive temperature. Reduced confidence calibration~\citep{panchenko2022class} lifts a calibrated top confidence to the simplex. Simplex Temperature Scaling (STS)~\citep{esaki2024accuracy} fits an input-dependent temperature while fixing the Concrete distribution's location parameter inherited from the pretrained classifier. These methods attach confidence to a preselected class or separate it from class selection during fitting; CORD instead operates on the full probability vector from an already fitted calibrator, returning a normalized vector whose own argmax recovers the original prediction.

\section{CORD: Post-Fit Prediction Preservation}
\label{sec:method}

\subsection{Problem Setup and Design Requirements}
\label{sec:problem-setup}

\noindent\textbf{Setup.}
Let $K\ge 2$ and $\Delta^{K-1}=\{u\in\mathbb R_{\ge0}^K:\sum_j u_j=1\}$. On a calibration split $\{x_i\}_{i=1}^n$, let $p_i^0,q_i\in\Delta^{K-1}$ denote the original classifier output and the direct output of an already fitted calibration map. All $\arg\max$ operations use the same fixed deterministic tie rule specified in Appendix, and $a_i=\arg\max_j p_{ij}^0$ is the originally predicted class. The repaired probability vector must recover this prediction through its own argmax.

\begin{equation}
\arg\max_j\widetilde p_{ij}=a_i.
\label{eq:preservation}
\end{equation}
The derivation below assumes $p_i^0,q_i\in\operatorname{ri}\Delta^{K-1}$ for all $i$. For boundary outputs, the Appendix defines an order-preserving stabilization using a fixed numerical constant; when required, the same symbols below denote the stabilized vectors used by CORD, while the supplied outputs remain unchanged.

\noindent\textbf{Design requirements.}
Separating repair from calibrator fitting requires CORD to use only the original and calibrated outputs (\textbf{R1}), fit no additional supervised prediction map (\textbf{R2}), and introduce no user- or validation-tuned hyperparameter (\textbf{R3}); making preservation a property of the repaired vector itself requires Equation~\eqref{eq:preservation} to hold for every input (\textbf{R4}). CORD returns $\widetilde p$ alongside the direct output, leaving that output and the fitted calibrator unchanged.

Equation~\eqref{eq:preservation} still admits infinitely many repaired vectors; CORD retains the conditional distribution induced by $q_i$ over classes $j\ne a_i$, leaving each vector repair with one degree of freedom. Within each input, CORD uses $q_i$, supplemented by $p_i^0$ only when the calibrated prediction changes, as a reference for the mass assigned to the original prediction; across the calibration split, it coordinates these masses to retain the mean encoded by $q$ when attainable. The adapter retains only one scalar computed on that split.

\subsection{A One-Dimensional Repair Family}
\label{sec:repair-family}

\noindent\textbf{Preserving the conditional distribution.}
CORD fixes the conditional distribution over classes $j\ne a_i$ by renormalizing the corresponding entries of $q_i$. Let $b_i=q_{i,a_i}$ and define $\alpha_{i,a_i}=0$ and $\alpha_{ij}=q_{ij}/(1-b_i)$ for $j\ne a_i$. Once $s_i$, the repaired mass assigned to $a_i$, is chosen, every coordinate of the repaired vector is fixed.
\begin{equation}
\widetilde p_i(s_i)
=s_i e_{a_i}+(1-s_i)\alpha_i,
\qquad s_i\in(0,1),
\label{eq:reconstruction}
\end{equation}
where $e_{a_i}$ is the corresponding standard basis vector. This reconstruction changes only the mass split between $a_i$ and the remaining classes, preserving $q_{ij}/q_{ik}$ for all $j,k\ne a_i$.

For any candidate $u_i$ assigning mass $s_i$ to $a_i$, write $\alpha_i^u$ for its conditional distribution over classes $j\ne a_i$. The KL chain rule decomposes $D_{\mathrm{KL}}(q_i\|u_i)$---the excess expected log loss relative to the calibrated report $q_i$---into the required change in mass on $a_i$ and any additional change in this conditional distribution.
\begin{equation}
D_{\mathrm{KL}}(q_i\|u_i)
=d_{\mathrm B}(b_i\|s_i)
+(1-b_i)D_{\mathrm{KL}}(\alpha_i\|\alpha_i^u),
\label{eq:kl-decomposition}
\end{equation}
where $d_{\mathrm B}(r\|s)=r\log(r/s)+(1-r)\log\{(1-r)/(1-s)\}$. For fixed $s_i$, the first term is fixed, whereas the second is nonnegative and vanishes only when $\alpha_i^u=\alpha_i$; hence the reconstruction in Equation~\eqref{eq:reconstruction} uniquely minimizes $D_{\mathrm{KL}}(q_i\|\cdot)$ on this simplex slice. Appendix gives the derivation.

\noindent\textbf{Prediction-preserving interval.}
Prediction preservation now constrains the sole remaining degree of freedom. With $\rho_i=\max_{j\ne a_i}\alpha_{ij}$, the largest repaired probability among classes $j\ne a_i$ is $(1-s_i)\rho_i$; comparison with $s_i$ yields the top-rank threshold and CORD's strictly prediction-preserving interval.
\begin{equation}
\begin{aligned}
\widetilde p_{i,a_i}(s_i)\ge
\max_{j\ne a_i}\widetilde p_{ij}(s_i)
&\Longleftrightarrow
s_i\ge\frac{\rho_i}{1+\rho_i},\\
I_i
&=\left[
\frac{\rho_i}{1+\rho_i}+\epsilon_{\mathrm{num}},
1-\epsilon_{\mathrm{num}}
\right].
\end{aligned}
\label{eq:feasible-range}
\end{equation}
With the numerical offset $\epsilon_{\mathrm{num}}$ fixed at $10^{-12}$, $\rho_i\le1$ ensures that $I_i$ is nonempty, and every $s_i\in I_i$ makes $a_i$ uniquely top-ranked. CORD thus reduces prediction-preserving repair to one scalar $s_i\in I_i$ per input.

\subsection{Coordinating Repairs on the Calibration Split}
\label{sec:aggregate-coordination}

\noindent\textbf{Local reference for the original prediction.}
The interval $I_i$ constrains $s_i$ but does not select its value. If $q_i$ keeps $a_i$ top-ranked, CORD retains $b_i$ as its local reference. Otherwise, the unconstrained choice $s_i=b_i$ would reproduce $q_i$, while $p_{i,a_i}^0$ supplies the second available output-level probability assigned to $a_i$. CORD assigns these two probabilities equal weight, rendering the Bernoulli--KL objective symmetric in them. The resulting average divergence to a candidate $s$ equals $d_{\mathrm B}((b_i+p_{i,a_i}^0)/2\|s)$ up to an $s$-independent constant, making the arithmetic mean the unique unconstrained minimizer.
\begin{equation}
g_i=
\begin{cases}
b_i,
&\arg\max_j q_{ij}=a_i,\\[2pt]
\dfrac{b_i+p_{i,a_i}^0}{2},
&\arg\max_j q_{ij}\ne a_i.
\end{cases}
\label{eq:local-reference}
\end{equation}
Both branches therefore yield the common per-input objective $d_{\mathrm B}(g_i\|s_i)$ used below. Appendix gives the corresponding derivation and local-reference sensitivity analysis.

\noindent\textbf{Retaining mean mass on the original predictions.}
The quantity $n^{-1}\sum_i b_i$ is the mean probability mass that the fitted calibrator assigns to the original predictions. Independent projection, $s_i=\Pi_{I_i}(g_i)$, satisfies pointwise prediction preservation but can shift this quantity, so CORD projects it onto the attainable mean interval:
\begin{equation}
\mu=
\Pi_{\left[
\frac1n\sum_i\min I_i,
\frac1n\sum_i\max I_i
\right]}
\left(\frac1n\sum_i b_i\right),
\label{eq:target-mean}
\end{equation}
where $\Pi_J$ denotes projection onto a closed interval $J$; thus $\mu$ retains $n^{-1}\sum_i b_i$ when attainable and otherwise selects the nearest attainable mean.

CORD minimizes total Bernoulli--KL departure from the local references over repairs with mean $\mu$.
\begin{equation}
\begin{aligned}
s^\star=\arg\min_{s_1,\ldots,s_n}\quad
&\sum_{i=1}^{n}d_{\mathrm B}(g_i\|s_i)\\
\textnormal{subject to}\quad
&s_i\in I_i\quad\forall i,
\qquad \frac1n\sum_{i=1}^{n}s_i=\mu.
\end{aligned}
\label{eq:coordinated-repair}
\end{equation}
The equality fixes the aggregate mass; the objective allocates the required adjustment. Equation~\eqref{eq:target-mean} ensures feasibility; joint strict convexity yields a unique solution.

\subsection{A Shared-Scalar Repair Rule}
\label{sec:shared-scalar}

Although Equation~\eqref{eq:coordinated-repair} involves $n$ scalar variables, they are coupled only by the aggregate equality, so a single Lagrange multiplier coordinates all repairs. For a candidate multiplier $\eta$, let $\psi(\eta;g,I)$ denote the interval-constrained response for a local reference $g$ and feasible interval $I$, and write $\psi_i(\eta)=\psi(\eta;g_i,I_i)$ on the calibration split. Interval clipping can leave the multiplier nonunique without changing the unique primal repair; CORD selects the valid multiplier closest to zero.
\begin{equation}
\begin{aligned}
\psi(\eta;g,I)
&=\arg\min_{s\in I}
\left\{d_{\mathrm B}(g\|s)-\eta s\right\},\\
\eta^\star
&=\arg\min_{\eta}\ |\eta|
\quad\textnormal{subject to}\quad
\frac1n\sum_{i=1}^{n}\psi_i(\eta)=\mu,\\
s_i^\star&=\psi_i(\eta^\star).
\end{aligned}
\label{eq:shared-scalar}
\end{equation}
Stationarity gives $(s-g)/\{s(1-s)\}=\eta$, whose unique interior solution is available in closed form; clipping this solution to $I$ evaluates $\psi(\eta;g,I)$. Each response is continuous and nondecreasing in $\eta$, and Equation~\eqref{eq:target-mean} places $\mu$ in the range of their mean, so the scalar equation can be solved by bisection. The nearest-zero rule leaves the unique calibration-split repair unchanged while defining a unique repair rule for new inputs. Appendix gives the closed-form response, its cancellation-safe evaluation, finite bracketing, and the plateau-aware solver.

Algorithm~\ref{alg:cord} summarizes CORD's construction on the calibration split and its application to a new output pair; Appendix provides a Python implementation, the KKT derivation, and existence arguments.

\begin{algorithm}[t!]
\caption{Construction and application of CORD}
\label{alg:cord}
\begin{algorithmic}[1]
\renewcommand{\algorithmicrequire}{\textbf{Input}}
\renewcommand{\algorithmicensure}{\textbf{Output}}
\REQUIRE Calibration-split output pairs
$\{(p_i^0,q_i)\}_{i=1}^{n}$ and a new output pair $(p^0,q)$
\ENSURE A repaired probability vector $\widetilde p$ whose argmax
recovers the original top-1 while preserving the calibrated
conditional distribution over the remaining classes

\STATE \textbf{Construction:} For each $i$, compute
$a_i=\arg\max_j p_{ij}^0$, $b_i=q_{i,a_i}$, and $\alpha_i$
as defined above; form $I_i$ and $g_i$ using
\eqref{eq:feasible-range} and~\eqref{eq:local-reference}

\STATE Compute the inherited feasible mean $\mu$
from~\eqref{eq:target-mean}

\STATE Solve~\eqref{eq:shared-scalar} for $\eta^\star$
and retain this scalar

\STATE \textbf{Application:} With $\eta^\star$ fixed, compute
$a$, $b$, $\alpha$, $I$, and $g$ for $(p^0,q)$ as above and set
$s=\psi(\eta^\star;g,I)$

\STATE \textbf{return}
$\widetilde p=s e_a+(1-s)\alpha$
\end{algorithmic}
\end{algorithm}

\begin{proposition}[Structural guarantees]
\label{prop:cord-guarantees}
Assume that the original and calibrated outputs, after the fixed stabilization when needed, lie in $\operatorname{ri}\Delta^{K-1}$, that all $\arg\max$ operations use the fixed deterministic tie rule, and that $0<\epsilon_{\mathrm{num}}<1/4$. The program in Equation~\eqref{eq:coordinated-repair} has a nonempty feasible set and a unique minimizer recovered by Equation~\eqref{eq:shared-scalar}; the nearest-zero selection $\eta^\star$ exists and is unique. With this scalar fixed, form the pointwise quantities for any input $x$ as above, set $s(x)=\psi(\eta^\star;g(x),I(x))$, and reconstruct $\widetilde p(x)$ using Equation~\eqref{eq:reconstruction}, where $a(x)=\arg\max_j p_j^0(x)$. Then
\begin{enumerate}
\item $\widetilde p(x)\in\Delta^{K-1}$;
\item $\arg\max_j\widetilde p_j(x)=a(x)$, with $a(x)$ uniquely top-ranked;
\item for every $j\ne a(x)$,
\begin{equation}
\frac{\widetilde p_j(x)}{1-\widetilde p_{a(x)}(x)}
=
\frac{q_j(x)}{1-q_{a(x)}(x)};
\label{eq:conditional-allocation-guarantee}
\end{equation}
\item $n^{-1}\sum_i s_i^\star=\mu$ on the calibration split.
\end{enumerate}
\end{proposition}

Proposition~\ref{prop:cord-guarantees} couples pointwise prediction preservation with preservation of the calibrated conditional distribution and retention of the inherited feasible mean on the calibration split. Appendix provides the proof and characterizes when CORD reduces to the identity map.

CORD construction costs $O(nK)$ plus $O(n)$ per bisection step; each repair costs $O(K)$, and the adapter stores only $\eta^\star$.  With the tie rule and numerical constants fixed, CORD derives every data-dependent quantity from $p^0$ and $q$, fits no additional supervised prediction map, and uses no user- or validation-tuned hyperparameter, thereby satisfying \textbf{R1}--\textbf{R3}; Proposition~\ref{prop:cord-guarantees} establishes \textbf{R4}.

\section{Experiments}
\label{sec:experiments}

\begin{figure*}[t!]
    \centering
    \includegraphics[width=0.85\textwidth]{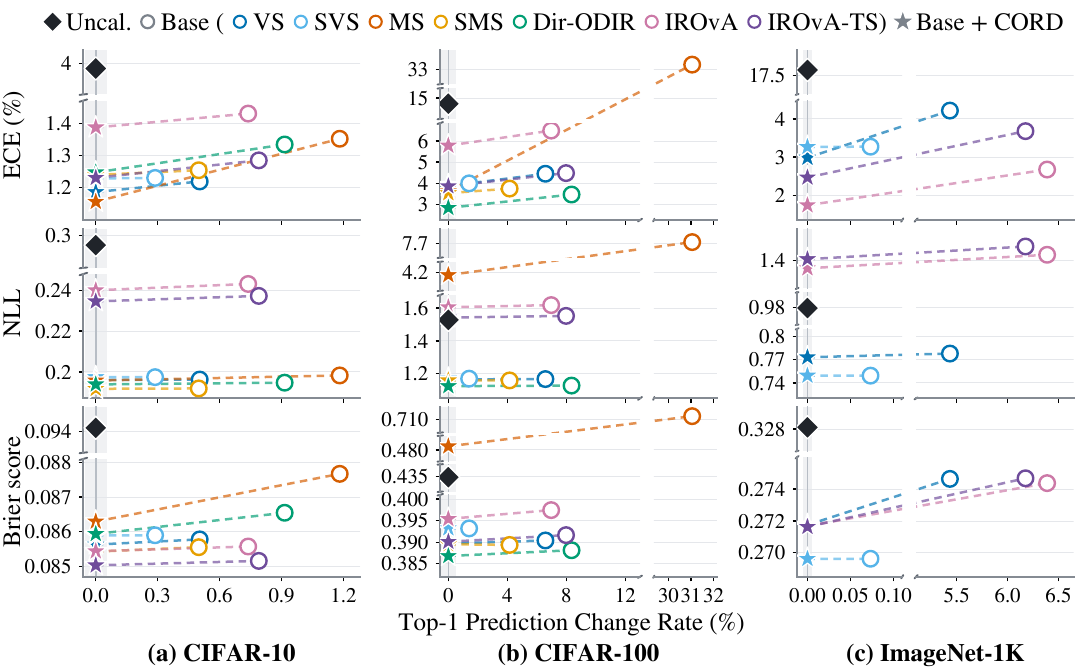}
    \caption{Post-fit prediction preservation. ECE, NLL, and Brier versus TPCR on CIFAR-10/100 and ImageNet-1K. Dashed lines pair each direct output with its CORD repair, and the gray band marks $\mathrm{TPCR}=0$. Markers are dataset-wise averages over classifiers and five splits.}
    \label{fig:main-evaluation}
\end{figure*}

\subsection{Experimental Setup}
\label{sec:experimental-setup}

\noindent\textbf{Datasets and classifiers.}
We evaluate CORD on CIFAR-10/100~\citep{krizhevsky2009learning} and ImageNet-1K~\citep{deng2009imagenet} using classifiers spanning convolutional networks, lightweight mobile architectures, and vision transformers: VGG-16-BN~\citep{simonyan2015very}, ResNet-56~\citep{he2016deep}, WRN-26-10~\citep{zagoruyko2016wide}, DenseNet-121~\citep{huang2017densely}, MobileNetV2-1.4$\times$~\citep{sandler2018mobilenetv2}, ShuffleNetV2-1.0$\times$~\citep{ma2018shufflenet}, and RepVGG-A1~\citep{ding2021repvgg} for CIFAR-10/100; ResNet-50~\citep{he2016deep}, ViT-B/16~\citep{dosovitskiy2020image}, Swin-T~\citep{liu2021swin}, and ConvNeXt-T~\citep{liu2022convnet} for ImageNet-1K. We use publicly released pretrained weights and keep all classifier parameters fixed. For each dataset, five distinct random seeds yield equal calibration and evaluation splits: $5{,}000$ examples per split from each CIFAR test set and $25{,}000$ per split from the ImageNet validation set. Within each dataset and seed, all classifiers and methods share the same split. For each split, we fit the calibrators and construct CORD on its calibration set, compute all measures on the corresponding evaluation set, and report five-split averages.

\noindent\textbf{Evaluation measures.}
We report Expected Calibration Error (ECE)~\citep{guo2017calibration}, which compares the mean maximum predicted probability with top-1 accuracy within 15 equal-width confidence bins; negative log-likelihood (NLL); and the multiclass Brier score~\citep{glenn1950verification}. Because CORD returns a full probability vector, we complement ECE with NLL and Brier, two strictly proper scoring rules for the full report; NLL evaluates the observed-class probability, whereas Brier aggregates squared errors across all coordinates relative to the one-hot outcome. These scores capture changes in the full predictive distribution that ECE and accuracy alone can miss~\citep{chidambaram2025reassessing}.

\begin{table}[t]
\renewcommand{\arraystretch}{0.5}
\centering
\footnotesize
\setlength{\tabcolsep}{1.0pt}
\resizebox{\columnwidth}{!}{%
    \begin{tabular}{
        l
        !{\color{black!18}\vrule width 0.25pt}
        c
        !{\color{black!18}\vrule width 0.25pt}
        c
        !{\color{black!18}\vrule width 0.25pt}
        c
    }
        \toprule
        \textbf{Dataset}
        & \textbf{$\Delta_{\mathrm R}$ECE (pp)}
        & \textbf{$\Delta_{\mathrm R}$NLL }
        & \textbf{$\Delta_{\mathrm R}$Brier ($\times10^{-3}$)} \\
        \midrule
        \shortstack[l]{CIFAR-10\\[-1pt]\strut}
        & \shortstack{0.061\\[-1pt]
            {[0.017, 0.125]}}
        & \shortstack{0.0014\\[-1pt]
            {[0.0009, 0.0019]}}
        & \shortstack{0.361\\[-1pt]
            {[0.215, 0.510]}} \\
        \addlinespace[2pt]
        \shortstack[l]{CIFAR-100\\[-1pt]\strut}
        & \shortstack{4.642\\[-1pt]
            {[2.822, 5.797]}}
        & \shortstack{0.5114\\[-1pt]
            {[0.2057, 0.7873]}}
        & \shortstack{33.691\\[-1pt]
            {[16.941, 45.355]}} \\
        \addlinespace[2pt]
        \shortstack[l]{ImageNet-1K\\[-1pt]\strut}
        & \shortstack{0.847\\[-1pt]
            {[0.594, 1.024]}}
        & \shortstack{0.0154\\[-1pt]
            {[0.0132, 0.0178]}}
        & \shortstack{2.159\\[-1pt]
            {[1.905, 2.470]}} \\
        \bottomrule
    \end{tabular}%
}
\caption{Mean paired Base--CORD metric reductions in Figure~\ref{fig:main-evaluation}, defined as $\Delta_{\mathrm R}M=M(\mathrm{Base})-M(\mathrm{Base}+\mathrm{CORD})$; positive values indicate improvement, and brackets denote Bonferroni-adjusted simultaneous 95\% bootstrap CIs.}
\label{tab:paired-effects}
\end{table}

\noindent\textbf{Top-1 prediction changes.}
To distinguish the extent of top-1 prediction changes from their net effect on accuracy, let $d_0$ denote the original top-1 rule and $d_1$ the rule induced by the evaluated output on an evaluation set of $N$ examples. We partition these changes into $\mathcal{C}^{+}$ (incorrect $\rightarrow$ correct), $\mathcal{C}^{-}$ (correct $\rightarrow$ incorrect), and $\mathcal{C}^{0}$ (incorrect $\rightarrow$ a different incorrect class). The net accuracy change is $\Delta\mathrm{Acc}:=\mathrm{Acc}(d_1)-\mathrm{Acc}(d_0)=(|\mathcal{C}^{+}|-|\mathcal{C}^{-}|)/N$: $\mathcal{C}^{+}$ and $\mathcal{C}^{-}$ enter with opposite signs, whereas $\mathcal{C}^{0}$ leaves accuracy unchanged. A small $\Delta\mathrm{Acc}$ can therefore coexist with many top-1 prediction changes. We report their total fraction as the Top-1 Prediction Change Rate (TPCR):
\begin{equation}
\label{eq:tpcr}
\mathrm{TPCR}
=\frac{1}{N}\sum_{i=1}^{N}
  \mathbf{1}\!\{d_1(x_i)\neq d_0(x_i)\}
=\frac{|\mathcal{C}^{+}|+|\mathcal{C}^{-}|+|\mathcal{C}^{0}|}{N}.
\end{equation}

\noindent\textbf{Calibrators.}
We apply CORD to calibrated outputs from seven maps that span parametric and nonparametric families and do not guarantee prediction preservation, namely VS~\citep{guo2017calibration}, SVS~\citep{berta2025structured}, MS~\citep{guo2017calibration}, SMS~\citep{berta2025structured}, Dir-ODIR~\citep{kull2019beyond}, IROvA~\citep{zhang2020mix}, and IROvA-TS~\citep{zhang2020mix}. We use Base to denote any fitted calibrator from this set and direct output to denote the probability vector it produces. On ImageNet-1K, we restrict this set to VS, SVS, IROvA, and IROvA-TS, omitting MS, SMS, and Dir-ODIR because each has more than $10^{6}$ fitted coefficients at $K=1000$~\citep{berta2025structured}. We include TS~\citep{guo2017calibration}, IRM~\citep{zhang2020mix}, AdaTS~\citep{joy2023sample}, TS--TvA~\citep{le2024confidence}, and MCCT-I~\citep{zhang2025instance} as fit-time prediction-preserving baselines. We run all calibrators with their default settings.

\begin{table}[t]
\renewcommand{\arraystretch}{1.0}
\centering
\resizebox{\columnwidth}{!}{%
    \begin{tabular}{
        l
        !{\color{black!18}\vrule width 0.25pt}
        l
        !{\color{black!18}\vrule width 0.25pt}
        l
        !{\color{black!18}\vrule width 0.25pt}
        l
        !{\color{black!18}\vrule width 0.25pt}
        l
    }
        \toprule
        \textbf{Dataset} & \textbf{Repair} & \textbf{ECE (\%)} & \textbf{NLL} & \textbf{Brier} \\
        \midrule
                \multirow{2}{*}{CIFAR-10}
            & Minimal repair & 1.285$^{**}$ & 0.208$^{***}$ & 0.0858$^{***}$ \\
            & \cellcolor{gray!12}CORD
            & \cellcolor{gray!12}\textbf{1.240}
            & \cellcolor{gray!12}\textbf{0.207}
            & \cellcolor{gray!12}\textbf{0.0857} \\
        \addlinespace[1pt]

        \multirow{2}{*}{CIFAR-100}
            & Minimal repair & 5.144$^{***}$ & 1.875$^{***}$ & 0.4055$^{***}$ \\
            & \cellcolor{gray!12}CORD
            & \cellcolor{gray!12}\textbf{3.912}
            & \cellcolor{gray!12}\textbf{1.705}
            & \cellcolor{gray!12}\textbf{0.4041} \\
        \addlinespace[1pt]

        \multirow{2}{*}{ImageNet-1K}
            & Minimal repair & 2.983$^{***}$ & 1.079$^{***}$ & 0.2714$^{**}$ \\
            & \cellcolor{gray!12}CORD
            & \cellcolor{gray!12}\textbf{2.614}
            & \cellcolor{gray!12}\textbf{1.077}
            & \cellcolor{gray!12}\textbf{0.2712} \\
        \bottomrule
    \end{tabular}%
}
\caption{CORD versus minimal pointwise repair, averaged over classifier--calibrator pairs and five splits. Both attain zero TPCR and preserve the calibrated conditional distribution. Asterisks mark significant reductions by CORD after Holm adjustment ($^{*}p<0.05$, $^{**}p<0.01$, $^{***}p<0.001$).}
\label{tab:minimal-repair}
\end{table}

\subsection{Prediction Preservation after Fitting}
\label{sec:post-fit-preservation}
Figure~\ref{fig:main-evaluation} demonstrates across datasets and calibrator families that prediction preservation need not be built into calibrator fitting; the Appendix provides complete classifier--calibrator tables and corresponding reliability diagrams. For every evaluated classifier--calibrator pair, CORD attains zero TPCR, recovering every original top-1 prediction and hence the original accuracy. Averaged over classifiers and five splits, the direct outputs generally attain lower ECE than their uncalibrated counterparts, while TPCR and top-1 accuracy change span $0.07\%$--$31.05\%$ and $-10.06$--$+0.22$ percentage points, respectively, across dataset--calibrator pairs. Accuracy records only these changes' net effect on correctness, whereas TPCR records their total incidence; a direct output can therefore improve calibration while changing the top-1 prediction whose correctness its confidence describes.

CORD also lowers mean ECE, NLL, and Brier relative to the corresponding direct outputs in every dataset; the NLL and Brier reductions show that its gains extend beyond confidence calibration to the full probability report. All corresponding confidence intervals lie above zero (Table~\ref{tab:paired-effects}). The paired trajectories further show that CORD's gains are largest when direct-output TPCR is high and remain small when few predictions are revised. This pattern is consistent with CORD's construction, which changes only how probability mass is split between the original prediction and the remaining classes while preserving the calibrated vector's relative allocation among those classes. The fitted calibrator thus continues to shape the repaired probability report, while the original classifier determines its top-1 prediction.

\begin{table}[t]
\renewcommand{\arraystretch}{1.0}
\centering
\resizebox{\columnwidth}{!}{%
    \Huge
    \begin{tabular}{
        l
        !{\color{black!18}\vrule width 0.25pt}
        l
        !{\color{black!18}\vrule width 0.25pt}
        l
        !{\color{black!18}\vrule width 0.25pt}
        l
        !{\color{black!18}\vrule width 0.25pt}
        l
    }
        \toprule
        \textbf{Mean target}
        & \textbf{ECE (\%)}
        & \textbf{NLL}
        & \textbf{Brier}
        & \textbf{$\lvert\bar{s}_{\mathrm{eval}}-\bar{b}_{\mathrm{eval}}\rvert$} \\
        \midrule
        None (independent)
            & 5.265$^{***}$ & 1.844$^{***}$ & 0.4056$^{***}$ & 2.808$^{***}$ \\
        Local-reference
            & 4.786$^{***}$ & 1.736$^{**}$ & 0.4050$^{**}$ & 2.277$^{***}$ \\
        Pointwise-projected Base
            & 5.000$^{***}$ & 1.841$^{***}$ & 0.4052$^{***}$ & 2.388$^{***}$ \\
        \rowcolor{gray!12}
        \textbf{CORD}
            & \textbf{3.912}
            & \textbf{1.705}
            & \textbf{0.4041}
            & \textbf{0.579} \\
        \bottomrule
    \end{tabular}%
}
\caption{Mean-target ablation on CIFAR-100, averaged over classifier--calibrator pairs and five splits. All variants attain zero TPCR; the final column is in pp. Asterisks mark significant reductions by CORD relative to each marked variant after Holm adjustment ($^{*}p<0.05$, $^{**}p<0.01$, $^{***}p<0.001$).}
\label{tab:mean-target}
\end{table}

\begin{figure}[t!]
    \centering
    \includegraphics[width=\columnwidth]{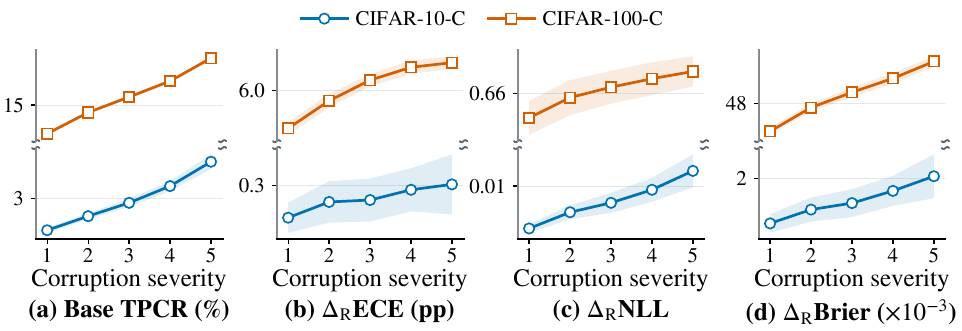}
    \caption{Base TPCR and paired metric reductions from CORD across corruption severity on CIFAR-10-C and CIFAR-100-C. For $M\in\{\mathrm{ECE},\mathrm{NLL},\mathrm{Brier}\}$, $\Delta_{\mathrm R}M=M(\mathrm{Base})-M(\mathrm{Base}+\mathrm{CORD})$, with $\Delta_{\mathrm R}M>0$ indicating improvement. Curves average over classifiers, calibrators, and 15 corruptions; shading shows pointwise 95\% CIs.}
    \label{fig:cifar-c}
\end{figure}

\subsection{Beyond Minimal Repair}
\label{sec:repair-ablations}
\noindent\textbf{Selecting the repaired mass.}
Exact prediction preservation defines a feasible family, not a unique probability report. We compare CORD with a minimal pointwise repair that leaves $q_i$ unchanged when $\arg\max_j q_{ij}=a_i$ and otherwise sets $s_i^{\mathrm{min}}=\Pi_{I_i}(b_i)$, reconstructing the vector with the calibrated conditional distribution $\alpha_i$ over classes $j\ne a_i$. Both repairs attain zero TPCR and preserve $\alpha_i$, so the comparison isolates the choice of $s_i$. CORD yields significantly lower ECE, NLL, and Brier on every dataset (Table~\ref{tab:minimal-repair}). Even with the original top-1 prediction and calibrated conditional distribution fixed, the remaining degree of freedom affects both confidence calibration and the quality of the full probability report.

\noindent\textbf{Inherited feasible mean.}
The aggregate mean constraint and its target both matter (Table~\ref{tab:mean-target}; corresponding results for CIFAR-10 and ImageNet-1K appear in the Appendix). All variants attain zero TPCR and otherwise share the intervals $I_i$, local references $g_i$, Bernoulli--KL objective, and reconstruction, yet the locally optimized independent variant performs worst in every column, whereas CORD yields the lowest value throughout. Prediction preservation fixes the original top-1 predictions and hence accuracy, not the aggregate mass assigned to them. CORD uses $g_i$ for the pointwise objective but derives the aggregate target from a single projection of the Base mean $\bar b=n^{-1}\sum_i b_i$ onto the attainable mean interval, retaining $\bar b$ when attainable and otherwise making the smallest mean change compatible with prediction preservation. The local-reference and pointwise-projected Base variants instead target $n^{-1}\sum_i g_i$ and $n^{-1}\sum_i\Pi_{I_i}(b_i)$, respectively; because the latter target is itself attainable, it cannot lie closer to $\bar b$ than CORD's target but can lie strictly farther. On held-out outputs, CORD also yields the smallest average $\lvert\bar{s}_{\mathrm{eval}}-\bar{b}_{\mathrm{eval}}\rvert$ among the variants, thereby most closely retaining the corresponding Base mean despite enforcing the aggregate equality only on the calibration split.

\begin{figure}[t!]
    \centering
    \includegraphics[width=\columnwidth]{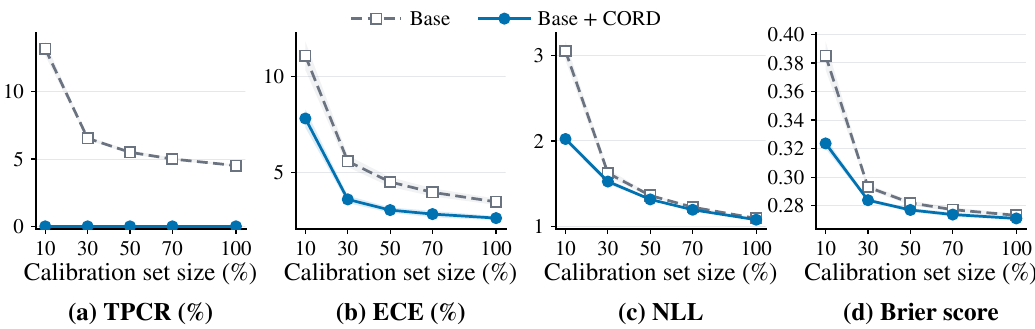}
    \caption{Calibration-size sensitivity on ImageNet-1K. Classifier- and calibrator-averaged TPCR, ECE, NLL, and Brier for Base and Base + CORD over 10\%--100\% calibration-set fractions, with shaded pointwise 95\% CIs.}
    \label{fig:calibration-size}
\end{figure}

\subsection{Robustness and Practicality}
\label{sec:robustness-practicality}

\noindent\textbf{Robustness under distribution shift.}
Under corruption, CORD maintains zero TPCR by construction and yields positive mean paired reductions in ECE, NLL, and Brier at every severity (Figure~\ref{fig:cifar-c}; the Appendix reports corresponding absolute results). We fit each calibrator and construct CORD on clean CIFAR data, holding both fixed for evaluation on CIFAR-10-C and CIFAR-100-C~\citep{hendrycks2019benchmarking}. Across severities 1--5, mean direct-output TPCR rises from $1.65\%$ to $4.58\%$ (CIFAR-10-C) and from $12.75\%$ to $18.70\%$ (CIFAR-100-C). Throughout, each repaired output retains its corresponding direct output's relative allocation among the remaining classes.

\begin{table*}[t!]
\renewcommand{\arraystretch}{0.9}
\centering
\setlength{\tabcolsep}{3.0pt}

\newcommand{\metricsep}{%
    \hspace{1pt}{\color{black!45}/}\hspace{1pt}}
\newcommand{\improved}{%
    \,{\color{green!50!black}\ensuremath{\blacktriangledown}}}
\newcommand{\degraded}{%
    \,{\color{red!75!black}\ensuremath{\blacktriangle}}}

\resizebox{\linewidth}{!}{%
    \begin{tabular}{
        l
        !{\color{black!18}\vrule width 0.25pt}
        c
        !{\color{black!18}\vrule width 0.25pt}
        c
        !{\color{black!18}\vrule width 0.25pt}
        c
        !{\color{black!18}\vrule width 0.25pt}
        c
        !{\color{black!18}\vrule width 0.25pt}
        c
        !{\color{black!18}\vrule width 0.25pt}
        c
        !{\color{black!18}\vrule width 0.25pt}
        c
        !{\color{black!18}\vrule width 0.25pt}
        >{\columncolor{gray!12}}c
    }
        \toprule
        \textbf{Classifier}
        & \textbf{Uncal.}
        & \textbf{TS}
        & \textbf{IRM}
        & \textbf{AdaTS}
        & \textbf{TS--TvA}
        & \textbf{MCCT-I}
        & \textbf{Base$^{\dagger}$}
        & \shortstack{\textbf{Base$^{\dagger}$ + CORD}} \\
        \midrule

        VGG-16-BN
        & 4.79\metricsep4.76\metricsep4.36
        & 1.54\metricsep2.52\metricsep1.95
        & 1.62\metricsep1.42\metricsep1.64
        & 1.47\metricsep1.39\metricsep1.66
        & 1.73\metricsep2.55\metricsep2.02
        & 1.74\metricsep2.88\metricsep2.16
        & 1.46\metricsep1.36\metricsep1.63
        & \textbf{1.41}\improved
          \metricsep
          \textbf{1.30}\improved
          \metricsep
          \textbf{1.61}\improved \\

        ResNet-56
        & 3.75\metricsep3.72\metricsep3.54
        & 0.94\metricsep1.24\metricsep1.23
        & 1.19\metricsep1.05\metricsep1.32
        & 1.59\metricsep1.43\metricsep1.66
        & 1.02\metricsep1.15\metricsep1.24
        & 1.02\metricsep1.22\metricsep1.21
        & 0.94\metricsep0.88\metricsep1.20
        & \textbf{0.92}\improved
          \metricsep
          \textbf{0.88}\improved
          \metricsep
          \textbf{1.13}\improved \\

        WRN-26-10
        & 3.36\metricsep3.35\metricsep2.88
        & 1.15\metricsep1.77\metricsep1.40
        & 1.19\metricsep0.90\metricsep1.27
        & 1.34\metricsep1.32\metricsep1.37
        & 1.16\metricsep1.78\metricsep1.45
        & 1.22\metricsep1.96\metricsep1.53
        & 1.13\metricsep0.91\metricsep1.21
        & \textbf{1.02}\improved
          \metricsep
          \textbf{0.88}\improved
          \metricsep
          \textbf{1.16}\improved \\

        DenseNet-121
        & 4.66\metricsep4.65\metricsep3.38
        & \textbf{1.24}\metricsep2.38\metricsep1.70
        & 1.25\metricsep\textbf{1.10}\metricsep1.31
        & 2.06\metricsep2.05\metricsep2.03
        & 1.31\metricsep2.38\metricsep1.72
        & 1.68\metricsep2.94\metricsep2.03
        & 1.32\metricsep1.19\metricsep1.37
        & 1.30\improved
          \metricsep
          1.14\improved
          \metricsep
          \textbf{1.27}\improved \\

        MobileNetV2-1.4$\times$
        & 3.73\metricsep3.68\metricsep3.58
        & 1.25\metricsep1.40\metricsep1.41
        & 1.38\metricsep1.28\metricsep1.41
        & 1.55\metricsep1.54\metricsep1.63
        & 1.22\metricsep1.33\metricsep1.38
        & 1.30\metricsep1.44\metricsep1.46
        & 1.20\metricsep1.23\metricsep1.36
        & \textbf{1.22}\degraded
          \metricsep
          \textbf{1.18}\improved
          \metricsep
          \textbf{1.35}\improved \\

        ShuffleNetV2-1.0$\times$
        & 4.03\metricsep3.97\metricsep3.86
        & 1.07\metricsep1.22\metricsep1.31
        & 1.46\metricsep1.38\metricsep1.48
        & 1.39\metricsep1.38\metricsep1.48
        & 1.02\metricsep1.15\metricsep1.29
        & \textbf{1.00}\metricsep1.09\metricsep1.29
        & 1.06\metricsep1.01\metricsep1.27
        & 1.04\improved
          \metricsep
          \textbf{1.05}\degraded
          \metricsep
          \textbf{1.25}\improved \\

        RepVGG-A1
        & 3.49\metricsep3.44\metricsep3.31
        & 1.14\metricsep1.46\metricsep1.30
        & 1.30\metricsep1.08\metricsep1.34
        & 1.34\metricsep1.42\metricsep1.49
        & 1.12\metricsep1.32\metricsep1.28
        & 1.13\metricsep1.33\metricsep1.28
        & 1.13\metricsep0.98\metricsep1.27
        & \textbf{1.10}\improved
          \metricsep
          \textbf{0.95}\improved
          \metricsep
          \textbf{1.23}\improved \\
        \bottomrule
    \end{tabular}%
}
\caption{Fit-time versus post-fit prediction preservation on CIFAR-10. Five-split means are reported as ECE / $\mathrm{ECE}_{\mathrm{EM}}$ / smECE (\%). Per classifier and metric, Base$^{\dagger}$ denotes the prediction-non-preserving Base with the lowest direct-output value, Base$^{\dagger}$+CORD its paired repair, and bold the lowest prediction-preserving value. Fit-time baselines and CORD repairs have zero TPCR; ${\color{green!50!black}\blacktriangledown}$ and ${\color{red!75!black}\blacktriangle}$ denote paired improvement and degradation from Base$^{\dagger}$, respectively.}
\label{tab:fit-time-preservation}
\end{table*}

\begin{table}[t]
\renewcommand{\arraystretch}{0.7}
\centering
\resizebox{\columnwidth}{!}{%
    \begin{tabular}{
        l
        !{\color{black!18}\vrule width 0.25pt}
        c
        !{\color{black!18}\vrule width 0.25pt}
        c
    }
        \toprule
        \textbf{Dataset} & \textbf{Construction (s)} & \textbf{Repair ($\mu$s/example)} \\
        \midrule
        CIFAR-10
        & 0.268\,[0.267,0.268]
        & 0.978\,[0.977,0.983] \\

        CIFAR-100
        & 0.280\,[0.280,0.281]
        & 1.996\,[1.947,2.089] \\

        ImageNet-1K
        & 1.384\,[1.383,1.385]
        & 25.64\,[25.45,25.84] \\
        \bottomrule
    \end{tabular}%
}
\caption{Runtime overhead of CORD. Single-threaded CPU wall-clock times for in-memory \texttt{float64} output pairs $(p^0,q)$, with $q$ from IROvA-TS. Entries are post-warm-up medians [$Q_1$, $Q_3$] over 9 construction and 21 vectorized full-split repair runs; repair time is amortized per example.}
\label{tab:runtime-overhead}
\end{table}

\noindent\textbf{Calibration-size sensitivity.}
On ImageNet-1K, CORD's mean paired reductions in ECE, NLL, and Brier are largest with limited calibration data yet persist across calibration-set fractions of 10\%--100\% (Figure~\ref{fig:calibration-size}; Appendix provides corresponding CIFAR-10/100 curves). At each fraction, we refit each calibrator and construct CORD on the same nested subset, holding the evaluation split fixed. Mean direct-output TPCR declines with calibration-set size but remains nonzero at 100\%, whereas CORD maintains zero TPCR throughout; ECE, NLL, and Brier decrease for both reports. Together, the corruption and calibration-size experiments extend the TPCR-dependent pattern in Figure~\ref{fig:main-evaluation}; paired gains widen as direct-output TPCR rises with corruption severity and narrow as it falls with additional calibration data.

\noindent\textbf{Runtime and memory.} CORD makes post-fit prediction preservation a lightweight addition to a fitted calibration pipeline without altering calibrator fitting or requiring an auxiliary model or further classifier or calibrator evaluation. Construction occurs once per fitted Base; amortized repair remains inexpensive even at ImageNet-1K scale (Table~\ref{tab:runtime-overhead}); at deployment, CORD's only data-dependent persistent state is the shared scalar $\eta^\star$ (8 bytes).

\subsection{Comparison with Prediction-Preserving Methods}
The calibration benefit of fitting without a prediction-preservation constraint persists when prediction preservation is imposed only afterward, even in the low-TPCR setting of CIFAR-10 (Table~\ref{tab:fit-time-preservation}). Its Base direct outputs exhibit the lowest TPCR among the evaluated datasets and hence the smallest observed departure from prediction preservation, yielding the most conservative decision-level comparison with the fit-time baselines. To test whether this finding depends on the fixed-width binning of standard ECE, we additionally report equal-mass $\mathrm{ECE}_{\mathrm{EM}}$~\citep{roelofs2022mitigating} and binning-free, kernel-smoothed smECE~\citep{blasiok2023smooth}. Despite revising fewer than $1\%$ of the original top-1 predictions in every case, the Base$^{\dagger}$ direct outputs often attain lower values than the fit-time baselines. Applied to those same outputs without altering the fitted calibrators, CORD lowers nearly every paired value while restoring zero TPCR; within this comparison, the repaired outputs remain broadly competitive with the fit-time baselines across all three calibration measures. The Appendix provides the NLL and Brier results and complete comparisons across all Base calibrators.

\section{Conclusion}
CORD relocates exact top-1 prediction preservation from calibrator fitting to post-fit output repair, opening post-fit preservation as a distinct calibration design space. Preservation remains a deployment choice: the unchanged direct output can be used when prediction changes are acceptable, whereas the repaired output applies when the original top-1 must be retained. Yet whether a change improves or degrades correctness is knowable only retrospectively from the label. As the TPCR analysis and its correctness-based decomposition show, calibrator-induced changes include improvements, degradations, and switches between incorrect classes, while accuracy records only their net balance. This calls into question whether prediction revision should be accepted merely as the price of greater expressiveness. The broader design problem is instead whether preservation is required, when it should be imposed, and how a probability report should be selected once exact preservation defines a feasible family. CORD establishes one lightweight post-fit route through this space without an auxiliary supervised map or a user- or validation-tuned hyperparameter. Richer post-fit objectives, alternative selection principles, and extensions beyond top-1 preservation remain open directions.

\bibliography{aaai2027}

@inproceedings{guo2017calibration,
  title={On calibration of modern neural networks},
  author={Guo, Chuan and Pleiss, Geoff and Sun, Yu and Weinberger, Kilian Q},
  booktitle={International conference on machine learning},
  pages={1321--1330},
  year={2017},
  organization={PMLR}
}

@article{kull2019beyond,
  title={Beyond temperature scaling: Obtaining well-calibrated multi-class probabilities with dirichlet calibration},
  author={Kull, Meelis and Perello Nieto, Miquel and K{\"a}ngsepp, Markus and Silva Filho, Telmo and Song, Hao and Flach, Peter},
  journal={Advances in neural information processing systems},
  volume={32},
  year={2019}
}

@article{zhang2020mix,
  title={Mix-n-match: Ensemble and compositional methods for uncertainty calibration in deep learning},
  author={Zhang, Jize and Kailkhura, Bhavya and Han, T},
  journal={arXiv preprint arXiv:2003.07329},
  year={2020}
}

@article{rahimi2020intra,
  title={Intra order-preserving functions for calibration of multi-class neural networks},
  author={Rahimi, Amir and Shaban, Amirreza and Cheng, Ching-An and Hartley, Richard and Boots, Byron},
  journal={Advances in neural information processing systems},
  volume={33},
  pages={13456--13467},
  year={2020}
}

@inproceedings{ma2021improving,
  title={Improving uncertainty calibration of deep neural networks via truth discovery and geometric optimization},
  author={Ma, Chunwei and Huang, Ziyun and Xian, Jiayi and Gao, Mingchen and Xu, Jinhui},
  booktitle={Uncertainty in Artificial Intelligence},
  pages={75--85},
  year={2021},
  organization={PMLR}
}

@inproceedings{ma2021meta,
  title={Meta-cal: Well-controlled post-hoc calibration by ranking},
  author={Ma, Xingchen and Blaschko, Matthew B},
  booktitle={International Conference on Machine Learning},
  pages={7235--7245},
  year={2021},
  organization={PMLR}
}

@inproceedings{gupta2022top,
  title={Top-label calibration and multiclass-to-binary reductions},
  author={Gupta, Chirag and Ramdas, Aaditya},
  booktitle={International Conference on Learning Representations},
  year={2022},
  organization={OpenReview}
}

@inproceedings{tomani2022parameterized,
  title={Parameterized temperature scaling for boosting the expressive power in post-hoc uncertainty calibration},
  author={Tomani, Christian and Cremers, Daniel and Buettner, Florian},
  booktitle={European conference on computer vision},
  pages={555--569},
  year={2022},
  organization={Springer}
}

@inproceedings{esaki2024accuracy,
  title={Accuracy-preserving calibration via statistical modeling on probability simplex},
  author={Esaki, Yasushi and Nakamura, Akihiro and Kawano, Keisuke and Tokuhisa, Ryoko and Kutsuna, Takuro},
  booktitle={International Conference on Artificial Intelligence and Statistics},
  pages={1666--1674},
  year={2024},
  organization={PMLR}
}

@article{le2024confidence,
  title={Confidence calibration of classifiers with many classes},
  author={Le Coz, Adrien and Herbin, St{\'e}phane and Adjed, Faouzi},
  journal={Advances in Neural Information Processing Systems},
  volume={37},
  pages={77686--77725},
  year={2024}
}

@article{berta2025structured,
  title={Structured Matrix Scaling for Multi-Class Calibration},
  author={Berta, Eug{\`e}ne and Holzm{\"u}ller, David and Jordan, Michael I and Bach, Francis},
  journal={arXiv preprint arXiv:2511.03685},
  year={2025}
}

@article{kirscher2026rethinking,
  title={Rethinking Post-Hoc Calibration in Semantic Segmentation},
  author={Kirscher, Tristan and Kahl, Kim-Celine and Kovacs, Balint and Rokuss, Maximilian R and Maier-Hein, Klaus and Coubez, Xavier and Meyer, Philippe and Faisan, Sylvain},
  journal={arXiv preprint arXiv:2607.01902},
  year={2026}
}

@article{krizhevsky2009learning,
  title={Learning multiple layers of features from tiny images},
  author={Krizhevsky, Alex and Hinton, Geoffrey and others},
  year={2009},
  publisher={Toronto, ON, Canada}
}

@inproceedings{deng2009imagenet,
  title={Imagenet: A large-scale hierarchical image database},
  author={Deng, Jia and Dong, Wei and Socher, Richard and Li, Li-Jia and Li, Kai and Fei-Fei, Li},
  booktitle={2009 IEEE conference on computer vision and pattern recognition},
  pages={248--255},
  year={2009},
  organization={Ieee}
}

@inproceedings{simonyan2015very,
  title={Very deep convolutional networks for large-scale image recognition},
  author={Simonyan, Karen and Zisserman, Andrew},
  booktitle={3rd international conference on learning representations (ICLR 2015)},
  year={2015},
  organization={Computational and Biological Learning Society}
}

@inproceedings{he2016deep,
  title={Deep residual learning for image recognition},
  author={He, Kaiming and Zhang, Xiangyu and Ren, Shaoqing and Sun, Jian},
  booktitle={Proceedings of the IEEE conference on computer vision and pattern recognition},
  pages={770--778},
  year={2016}
}

@article{zagoruyko2016wide,
  title={Wide residual networks},
  author={Zagoruyko, Sergey and Komodakis, Nikos},
  journal={arXiv preprint arXiv:1605.07146},
  year={2016}
}

@inproceedings{huang2017densely,
  title={Densely connected convolutional networks},
  author={Huang, Gao and Liu, Zhuang and Van Der Maaten, Laurens and Weinberger, Kilian Q},
  booktitle={Proceedings of the IEEE conference on computer vision and pattern recognition},
  pages={4700--4708},
  year={2017}
}

@inproceedings{sandler2018mobilenetv2,
  title={Mobilenetv2: Inverted residuals and linear bottlenecks},
  author={Sandler, Mark and Howard, Andrew and Zhu, Menglong and Zhmoginov, Andrey and Chen, Liang-Chieh},
  booktitle={Proceedings of the IEEE conference on computer vision and pattern recognition},
  pages={4510--4520},
  year={2018}
}

@inproceedings{ma2018shufflenet,
  title={Shufflenet v2: Practical guidelines for efficient cnn architecture design},
  author={Ma, Ningning and Zhang, Xiangyu and Zheng, Hai-Tao and Sun, Jian},
  booktitle={Proceedings of the European conference on computer vision (ECCV)},
  pages={116--131},
  year={2018}
}

@inproceedings{ding2021repvgg,
  title={Repvgg: Making vgg-style convnets great again},
  author={Ding, Xiaohan and Zhang, Xiangyu and Ma, Ningning and Han, Jungong and Ding, Guiguang and Sun, Jian},
  booktitle={Proceedings of the IEEE/CVF conference on computer vision and pattern recognition},
  pages={13733--13742},
  year={2021}
}

@article{dosovitskiy2020image,
  title={An image is worth 16x16 words: Transformers for image recognition at scale},
  author={Dosovitskiy, Alexey and Beyer, Lucas and Kolesnikov, Alexander and Weissenborn, Dirk and Zhai, Xiaohua and Unterthiner, Thomas and Dehghani, Mostafa and Minderer, Matthias and Heigold, Georg and Gelly, Sylvain and others},
  journal={arXiv preprint arXiv:2010.11929},
  year={2020}
}

@inproceedings{liu2021swin,
  title={Swin transformer: Hierarchical vision transformer using shifted windows},
  author={Liu, Ze and Lin, Yutong and Cao, Yue and Hu, Han and Wei, Yixuan and Zhang, Zheng and Lin, Stephen and Guo, Baining},
  booktitle={Proceedings of the IEEE/CVF international conference on computer vision},
  pages={10012--10022},
  year={2021}
}

@inproceedings{liu2022convnet,
  title={A convnet for the 2020s},
  author={Liu, Zhuang and Mao, Hanzi and Wu, Chao-Yuan and Feichtenhofer, Christoph and Darrell, Trevor and Xie, Saining},
  booktitle={Proceedings of the IEEE/CVF conference on computer vision and pattern recognition},
  pages={11976--11986},
  year={2022}
}

@article{glenn1950verification,
  title={Verification of forecasts expressed in terms of probability},
  author={Glenn, W Brier and others},
  journal={Monthly weather review},
  volume={78},
  number={1},
  pages={1--3},
  year={1950},
  publisher={War Department, Office of the Chief Signal Officer}
}

@inproceedings{chidambaram2025reassessing,
  title={Reassessing how to compare and improve the calibration of machine learning models},
  author={Chidambaram, Muthu and Ge, Rong},
  booktitle={International Conference on Learning Representations},
  volume={2025},
  pages={61542--61570},
  year={2025}
}

@article{zhang2025instance,
  title={Instance-wise monotonic calibration by constrained transformation},
  author={Zhang, Yunrui and Batista, Gustavo and Kanhere, Salil S},
  journal={arXiv preprint arXiv:2507.06516},
  year={2025}
}

@inproceedings{joy2023sample,
  title={Sample-dependent adaptive temperature scaling for improved calibration},
  author={Joy, Tom and Pinto, Francesco and Lim, Ser-Nam and Torr, Philip HS and Dokania, Puneet K},
  booktitle={Proceedings of the AAAI Conference on Artificial Intelligence},
  volume={37},
  number={12},
  pages={14919--14926},
  year={2023}
}

@article{hendrycks2019benchmarking,
  title={Benchmarking neural network robustness to common corruptions and perturbations},
  author={Hendrycks, Dan and Dietterich, Thomas},
  journal={arXiv preprint arXiv:1903.12261},
  year={2019}
}

@inproceedings{roelofs2022mitigating,
  title={Mitigating bias in calibration error estimation},
  author={Roelofs, Rebecca and Cain, Nicholas and Shlens, Jonathon and Mozer, Michael C},
  booktitle={International Conference on Artificial Intelligence and Statistics},
  pages={4036--4054},
  year={2022},
  organization={PMLR}
}

@inproceedings{blasiok2023smooth,
  title={Smooth ECE: Principled reliability diagrams via kernel smoothing},
  author={Blasiok, Jaroslaw and Nakkiran, Preetum},
  booktitle={The Twelfth International Conference on Learning Representations},
  year={2023}
}

@article{chakraborty2026quantile,
  title={Quantile Adaptive Temperature Scaling for Confidence Calibration},
  author={Chakraborty, Omprakash and Fillioux, Leo and Ayed, Ismail Ben and Dolz, Jose},
  journal={arXiv preprint arXiv:2606.21749},
  year={2026}
}

@article{atarashi2025box,
  title={Box-Constrained Softmax Function and Its Application for Post-Hoc Calibration},
  author={Atarashi, Kyohei and Oyama, Satoshi and Arai, Hiromi and Kashima, Hisashi},
  journal={arXiv preprint arXiv:2506.10572},
  year={2025}
}

@inproceedings{panchenko2022class,
  title={Class-wise and reduced calibration methods},
  author={Panchenko, Michael and Benmerzoug, Anes and de Benito Delgado, Miguel},
  booktitle={2022 21st IEEE International Conference on Machine Learning and Applications (ICMLA)},
  pages={1093--1100},
  year={2022},
  organization={IEEE}
}
\clearpage
\onecolumn
\appendix

\begin{center}
    {\LARGE\bfseries
    Let Confidence Change, Not the Prediction:
    Prediction-Preserving Repair for Post-hoc Calibration}\\[8pt]
    {\Large\bfseries \textemdash\ Appendix \textemdash}
\end{center}

\section{Derivations and Numerical Details for CORD}
\label{app:cord-details}

This appendix supplies the derivations and fixed numerical conventions deferred from Section~3 (\emph{CORD: Post-Fit Prediction Preservation}) of the main paper.

\subsection{Boundary Outputs and Fixed Numerical Conventions}
\label{app:boundary-stabilization}

\noindent\textbf{Tie rule and numerical constants.}
Every exact tie is resolved by selecting the smallest maximizing class index, and the same rule is used for the original, direct, and repaired outputs. We use $\epsilon_{\mathrm{num}}=10^{-12}$ and $\delta_{\mathrm{stab}}=10^{-10}$ throughout.

\noindent\textbf{Order-preserving stabilization.}
For a vector $u$ on the boundary of $\Delta^{K-1}$, CORD forms the internal copy
\begin{equation}
\mathcal S_{\delta_{\mathrm{stab}}}(u)
=
(1-\delta_{\mathrm{stab}})u
+
\frac{\delta_{\mathrm{stab}}}{K}\mathbf 1.
\label{eq:app-stabilization}
\end{equation}
A vector already in $\operatorname{ri}\Delta^{K-1}$ is left unchanged, and $p^0$ and $q$ are stabilized independently when needed. Since
$[\mathcal S_{\delta_{\mathrm{stab}}}(u)]_j-[\mathcal S_{\delta_{\mathrm{stab}}}(u)]_k
=(1-\delta_{\mathrm{stab}})(u_j-u_k)$,
stabilization preserves every pairwise order relation and the complete maximizer set, while producing a normalized vector with all coordinates in $(0,1)$.
Consequently, the fixed tie rule selects the same original prediction from the supplied and stabilized copies of $p^0$.

Stabilization is internal to CORD; the supplied outputs remain unchanged. When needed, $p^0$ and $q$ in Section~3 of the main paper denote the stabilized internal copies. Preservation of the calibrated conditional distribution over the remaining classes in Proposition~1 of the main paper is therefore exact with respect to the internal copy of $q$. This copy equals the supplied direct output when the latter already lies in $\operatorname{ri}\Delta^{K-1}$; otherwise,
$\|\mathcal S_{\delta_{\mathrm{stab}}}(q)-q\|_1
=\delta_{\mathrm{stab}}\|K^{-1}\mathbf1-q\|_1
\le 2\delta_{\mathrm{stab}}$.
For numerical evaluation, CORD computes the denominator $1-q_a$ in the normalization over classes $j\ne a$ as the algebraically equivalent sum $\sum_{j\ne a}q_j$.

\subsection{KL Characterization of the Repair Family}
\label{app:kl-characterization}

This subsection derives Equation~(3) of the main paper. Fix an input and omit its index. Let $a$ be the originally predicted class, $b=q_a$, and $\alpha_a=0$, with $\alpha_j=q_j/(1-b)$ for $j\ne a$. For a fixed $s\in(0,1)$, consider any $u\in\Delta^{K-1}$ satisfying $u_a=s$, and define its conditional distribution over the remaining classes by $\alpha_a^u=0$ and $\alpha_j^u=u_j/(1-s)$ for $j\ne a$. All KL divergences below use the usual extended-value convention, and $D_{\mathrm{KL}}(\alpha\|\alpha^u)$ is taken over $j\ne a$. Substituting $q_j=(1-b)\alpha_j$ and $u_j=(1-s)\alpha_j^u$ gives
\begin{align}
D_{\mathrm{KL}}(q\|u)
&=
b\log\frac{b}{s}
+
\sum_{j\ne a}(1-b)\alpha_j
\log\frac{(1-b)\alpha_j}{(1-s)\alpha_j^u}
\notag\\
&=
d_{\mathrm B}(b\|s)
+
(1-b)D_{\mathrm{KL}}(\alpha\|\alpha^u).
\label{eq:app-kl-chain}
\end{align}
For fixed $s$, the first term is constant, while the second is nonnegative and vanishes only when $\alpha^u=\alpha$. The reconstruction $u=se_a+(1-s)\alpha$ in Equation~(2) of the main paper is therefore the unique minimizer of $D_{\mathrm{KL}}(q\|\cdot)$ on the simplex slice $\{u\in\Delta^{K-1}:u_a=s\}$. Once the mass assigned to the original prediction is fixed, the calibrated conditional distribution over the remaining classes is thus inherited rather than re-estimated.

\subsection{Shared-Scalar Characterization}
\label{app:shared-scalar-derivation}

This subsection establishes the characterization in Equation~(8) of the main paper. Write $I_i=[s_i^-,s_i^+]$ and $f_i(s)=d_{\mathrm B}(g_i\|s)$. Because $g_i,s_i^-,s_i^+\in(0,1)$,
\begin{equation}
f_i'(s)=\frac{s-g_i}{s(1-s)},
\qquad
f_i''(s)=\frac{g_i}{s^2}
+
\frac{1-g_i}{(1-s)^2}>0.
\label{eq:app-bernoulli-derivatives}
\end{equation}
The objective in Equation~(7) of the main paper is therefore strictly convex. Moreover,
$s_i^+-s_i^-=(1+\rho_i)^{-1}-2\epsilon_{\mathrm{num}}
\ge 1/2-2\epsilon_{\mathrm{num}}>0$,
and the attainable means are exactly
$[\bar s^-,\bar s^+]$, where
$\bar s^-=n^{-1}\sum_i s_i^-$ and
$\bar s^+=n^{-1}\sum_i s_i^+$.
Because Equation~(6) of the main paper places $\mu$ in this interval, the feasible set is nonempty and compact. The coordinated repair consequently has a unique minimizer.

\noindent\textbf{Lagrangian and KKT reduction.}
Rewrite the aggregate equality in Equation~(7) of the main paper as
$\sum_{i=1}^n s_i=n\mu$. Introduce $\eta\in\mathbb R$ for this
equality and $\nu_i^-,\nu_i^+\ge0$ for the lower and upper interval
bounds, respectively. Using the sign convention in
Equation~(8) of the main paper, the Lagrangian is
\begin{equation}
\begin{aligned}
\mathcal L(s,\eta,\nu^-,\nu^+)
={}&
\sum_{i=1}^n f_i(s_i)
-\eta\left(\sum_{i=1}^n s_i-n\mu\right)\\
&+
\sum_{i=1}^n\nu_i^-(s_i^--s_i)
+
\sum_{i=1}^n\nu_i^+(s_i-s_i^+).
\end{aligned}
\label{eq:app-lagrangian}
\end{equation}
Because the feasible set is a nonempty polyhedron and the objective
is convex and differentiable on a neighborhood of that set, the KKT
conditions below are necessary and sufficient for every
$\mu\in[\bar s^-,\bar s^+]$. In addition to primal and dual
feasibility, they require
\begin{equation}
\begin{aligned}
f_i'(s_i)-\eta-\nu_i^-+\nu_i^+&=0,\\
\nu_i^-(s_i-s_i^-)&=0,\qquad
\nu_i^+(s_i^+-s_i)=0,
\end{aligned}
\qquad i=1,\ldots,n.
\label{eq:app-kkt}
\end{equation}
Here $\eta$ is the only multiplier shared across repairs; the bound
multipliers are local to individual $s_i$ and are eliminated
by the interval-constrained minimization in Equation~(8) of the main
paper. Let $h_g:(0,1)\to\mathbb R$ be given by
$h_g(s)=(s-g)/\{s(1-s)\}$ and write $H(g,\eta)=h_g^{-1}(\eta)$.
Equation~\eqref{eq:app-bernoulli-derivatives}
makes $h_g$ continuous and strictly increasing from $-\infty$ to
$+\infty$. Stationarity and complementary slackness give
$\eta\le h_g(s^-)$ at the lower endpoint, $h_g(s)=\eta$ in the
interval interior, and $\eta\ge h_g(s^+)$ at the upper endpoint.
Consequently, the unique interval-constrained response is
\begin{equation}
\psi(\eta;g,[s^-,s^+])
=
\begin{cases}
s^-, & \eta\le h_g(s^-),\\
H(g,\eta), & h_g(s^-)<\eta<h_g(s^+),\\
s^+, & \eta\ge h_g(s^+).
\end{cases}
\label{eq:app-clipped-response}
\end{equation}
Thus $\psi$ is continuous and nondecreasing in $\eta$. Define
\begin{equation}
\bar\psi(\eta)=\frac1n\sum_{i=1}^n\psi(\eta;g_i,I_i),
\qquad
E_\mu=\{\eta\in\mathbb R:\bar\psi(\eta)=\mu\}.
\label{eq:app-mean-response}
\end{equation}
The mean response is continuous and nondecreasing, equals $\bar s^-$ and $\bar s^+$ beyond finite lower and upper thresholds, respectively, and hence attains every value in $[\bar s^-,\bar s^+]$. It follows that $E_\mu$ is a nonempty closed interval. It is bounded when $\mu\in(\bar s^-,\bar s^+)$, equals
$(-\infty,\min_i h_{g_i}(s_i^-)]$ when $\mu=\bar s^-$, and equals
$[\max_i h_{g_i}(s_i^+),\infty)$ when $\mu=\bar s^+$. For every
$\eta\in E_\mu$, the vector with coordinates
$s_i=\psi(\eta;g_i,I_i)$ satisfies the KKT conditions by taking
$\nu_i^-=h_{g_i}(s_i^-)-\eta$ at a lower endpoint,
$\nu_i^+=\eta-h_{g_i}(s_i^+)$ at an upper endpoint, and the inactive
bound multipliers equal to zero. By uniqueness, this vector is the
calibration-split repair. CORD makes the corresponding rule unique
for new inputs by selecting
\begin{equation}
\eta^\star=\Pi_{E_\mu}(0),
\label{eq:app-nearest-zero-multiplier}
\end{equation}
which exists and is unique because $E_\mu$ is a nonempty closed interval.

\subsection{Closed-Form Response and Numerical Solver}
\label{app:shared-scalar-solver}

\noindent\textbf{Cancellation-safe response.}
The interior stationarity equation $h_g(s)=\eta$ is quadratic.
A cancellation-safe form is obtained by defining
\begin{equation}
\begin{aligned}
R(g,\eta)
&=
\sqrt{(\eta+2g-1)^2+4g(1-g)}\\
&=
\operatorname{hypot}\!\left(
\eta+2g-1,\,
2\sqrt{g(1-g)}
\right).
\end{aligned}
\label{eq:app-stable-discriminant}
\end{equation}
The corresponding root in $(0,1)$ is
\begin{equation}
H(g,\eta)
=
\begin{cases}
\displaystyle
\frac{2g}{1-\eta+R(g,\eta)},
& \eta\le1,\\[3mm]
\displaystyle
\frac{\eta-1+R(g,\eta)}{2\eta},
& \eta>1.
\end{cases}
\label{eq:app-stable-response}
\end{equation}
At $\eta=0$, $H(g,0)=g$, and projecting $H(g,\eta)$ onto
$[s^-,s^+]$ gives the constrained response in
Equation~\eqref{eq:app-clipped-response}.

\noindent\textbf{Finite bracket.}
Let
\begin{equation}
\begin{aligned}
\eta_{\mathrm L}&=\min_i h_{g_i}(s_i^-),
&
\eta_{\mathrm U}&=\max_i h_{g_i}(s_i^+),\\
L&=\min\{0,\eta_{\mathrm L}\},
&
U&=\max\{0,\eta_{\mathrm U}\}.
\end{aligned}
\label{eq:app-finite-bracket}
\end{equation}
All responses equal their lower endpoints at $L$ and their upper endpoints at $U$, so
$\bar\psi(L)=\bar s^-\le\mu\le\bar s^+=\bar\psi(U)$.
The interval $[L,U]$ is therefore a finite bracket containing both zero and the nearest-zero multiplier $\eta^\star$.

\noindent\textbf{Plateau-aware bisection.}
When $E_\mu$ is non-singleton, an ordinary root finder can return an arbitrary point on the corresponding plateau. The nearest-zero convention is equivalently
\begin{equation}
\eta^\star
=
\begin{cases}
0, & \bar\psi(0)=\mu,\\[2pt]
\inf\{\eta\in[0,U]:\bar\psi(\eta)\ge\mu\},
& \bar\psi(0)<\mu,\\[2pt]
\sup\{\eta\in[L,0]:\bar\psi(\eta)\le\mu\},
& \bar\psi(0)>\mu.
\end{cases}
\label{eq:app-directional-inverse}
\end{equation}
These generalized inverses equal the projection in Equation~\eqref{eq:app-nearest-zero-multiplier}. CORD obtains the applicable value by bisection on the corresponding one-sided bracket, using the upper endpoint in the second case and the lower endpoint in the third as the target-side endpoint; equality at a midpoint updates this endpoint. In binary64 arithmetic, iteration continues until the bracket admits no distinct representable midpoint, and the target-side endpoint is returned.

\subsection{Proof of the Structural Guarantees}
\label{app:proof-cord-guarantees}

\begin{proof}[Proof of Proposition~1 of the main paper]
Existence and uniqueness of the calibration-split repair and the nearest-zero scalar follow from the preceding shared-scalar characterization. Fix an arbitrary input and suppress its dependence on $x$. Because $s\in I\subset(0,1)$ and $\alpha$ is a distribution over classes $j\ne a$, Equation~(2) of the main paper gives $\widetilde p\in\Delta^{K-1}$.

For every $j\ne a$, $\widetilde p_j=(1-s)\alpha_j\le(1-s)\rho$, whereas
\begin{equation}
s-(1-s)\rho
=(1+\rho)\left(
s-\frac{\rho}{1+\rho}
\right)
\ge
(1+\rho)\epsilon_{\mathrm{num}}>0.
\label{eq:app-strict-margin}
\end{equation}
Thus $a$ is uniquely top-ranked, establishing prediction preservation independently of the tie rule. Moreover, for every $j\ne a$,
\begin{equation}
\frac{\widetilde p_j}{1-\widetilde p_a}
=
\frac{(1-s)\alpha_j}{1-s}
=
\alpha_j
=
\frac{q_j}{1-q_a},
\label{eq:app-conditional-allocation}
\end{equation}
which proves preservation of the calibrated conditional distribution over the remaining classes. Finally, $n^{-1}\sum_i s_i^\star=\mu$ is the equality constraint in Equation~(7) of the main paper.
\end{proof}

\noindent\textbf{Identity characterization.}
On the calibration split, CORD returns $q_i$ for every $i$ if and only if $b_i\in I_i$ for every $i$. Under this condition, $q_i$ already has $a_i$ as its unique top-ranked class, so $g_i=b_i$; the vector $(b_1,\ldots,b_n)$ is feasible, Equation~(6) of the main paper gives $\mu=n^{-1}\sum_i b_i$, and $\psi(0;b_i,I_i)=b_i$. Hence $0\in E_\mu$, the nearest-zero rule selects $\eta^\star=0$, and Equation~(2) of the main paper recovers $q_i$. Conversely, $\widetilde p_i=q_i$ requires $s_i=b_i\in I_i$. With the resulting scalar $\eta^\star=0$ fixed, CORD returns $q$ for a new input if and only if $b=q_a\in I$. When boundary stabilization is used, identity is exact with respect to the stabilized internal copy of $q$, whose distance from the supplied output is bounded above by $2\delta_{\mathrm{stab}}$.
\clearpage
\section{Sensitivity to Local-Reference Weighting}
\label{app:design-sensitivity}

In the changed-prediction branch, the local reference in Equation~(5) of the main paper arises from assigning equal weight to the Bernoulli--KL departures from $b_i$ and $p^0_{i,a_i}$ to a candidate $s$. To examine variations of this weighting without changing any other component of CORD, consider the family indexed by $\lambda\in[0,1]$:
\begin{equation}
g_i^{(\lambda)}
=
\begin{cases}
b_i,
& \arg\max_j q_{ij}=a_i,\\[2pt]
\lambda b_i+(1-\lambda)p^0_{i,a_i},
& \arg\max_j q_{ij}\ne a_i.
\end{cases}
\label{eq:app-weighted-reference}
\end{equation}
Here $\lambda$ weights the probability mass $b_i$ that the direct output assigns to the original prediction; CORD uses the equal-weight setting $\lambda=1/2$. When $\arg\max_j q_{ij}\ne a_i$,
\begin{align}
&\lambda d_{\mathrm B}(b_i\|s)
+(1-\lambda)d_{\mathrm B}(p^0_{i,a_i}\|s)
\notag\\
&\quad=
d_{\mathrm B}(g_i^{(\lambda)}\|s)
+\lambda d_{\mathrm B}(b_i\|g_i^{(\lambda)})
\notag\\
&\qquad
+(1-\lambda)d_{\mathrm B}(p^0_{i,a_i}\|g_i^{(\lambda)}).
\label{eq:app-weighted-kl}
\end{align}
The final two terms do not depend on $s$, so $g_i^{(\lambda)}$ is the unique unconstrained minimizer of this weighted local objective. Within this fixed-weight family, $\lambda=1/2$ is the only fixed setting for which exchanging $b_i$ and $p^0_{i,a_i}$ leaves the local-reference rule unchanged for all admissible pairs, yielding the arithmetic mean in Equation~(5) of the main paper.

The sensitivity study changes only the local references through $\lambda$; the feasible intervals and inherited feasible mean $\mu$ remain fixed. For each $\lambda$, the same coordinated-repair program and shared-scalar solver apply, and the structural guarantees in Proposition~1 of the main paper remain unchanged. We evaluate seven fixed values over the 114 dataset--classifier--calibrator conditions underlying Figure~2 of the main paper, using five calibration/evaluation splits per condition. Each entry in Table~\ref{tab:local-reference-sensitivity} is the unweighted mean of the 114 condition-level paired reductions after first averaging over the five splits within each condition.

\begin{table}[ht!]
\renewcommand{\arraystretch}{1.0}
\centering
\setlength{\tabcolsep}{3.0pt}
\resizebox{0.5\linewidth}{!}{%
    \begin{tabular}{
        c
        !{\color{black!18}\vrule width 0.25pt}
        c
        !{\color{black!18}\vrule width 0.25pt}
        c
        !{\color{black!18}\vrule width 0.25pt}
        c
    }
        \toprule
        \textbf{$\lambda$}
        & \textbf{$\Delta_{\mathrm R}$ECE (pp)}
        & \textbf{$\Delta_{\mathrm R}$NLL}
        & \shortstack{
            \textbf{$\Delta_{\mathrm R}$Brier}
            \textbf{($\times10^{-3}$)}
        } \\
        \midrule

        0.0 & 1.888 & 0.2212 & 14.731 \\
        0.1 & 1.936 & 0.2220 & 14.761 \\
        0.3 & 2.147 & 0.2224 & 14.748 \\

        \rowcolor{gray!12}
        \textbf{0.5} & 2.140 & 0.2225 & 14.939 \\

        0.7 & 2.108 & 0.2225 & 14.959 \\
        0.9 & 2.061 & 0.2222 & 14.892 \\
        1.0 & 2.054 & 0.2222 & 14.878 \\

        \bottomrule
    \end{tabular}%
}
\caption{Sensitivity to local-reference weighting. The weight $\lambda$ multiplies $b_i$ in the changed-prediction branch of Equation~\eqref{eq:app-weighted-reference}; $\lambda=0.5$ is the equal-weight setting used by CORD. Following Table~1 of the main paper, paired reductions are defined as $\Delta_{\mathrm R}M=M(\mathrm{Base})-M(\mathrm{Base}+\mathrm{CORD})$; positive values indicate improvement.}
\label{tab:local-reference-sensitivity}
\end{table}

All seven fixed weights yield positive mean reductions in ECE, NLL, and Brier. Across $\lambda\in[0.3,0.9]$, the largest point estimate occurs at different weights across the three metrics; CORD fixes $\lambda=1/2$ by the symmetry criterion above.

\clearpage
\section{Python Implementation}
\label{app:cord-implementation}
The Python listing below compactly implements Algorithm~1 of the main paper; the preceding sections specify the fixed numerical conventions.

\begin{lstlisting}[style=cordpython]
import numpy as np


EPS_NUM = 1e-12
DELTA_STAB = 1e-10


class CORD:
    def fit(self, original: np.ndarray, calibrated: np.ndarray) -> "CORD":
        p0, q = _stabilize(original), _stabilize(calibrated)
        _, b, _, g, lower, upper = _quantities(p0, q)
        target = float(np.clip(b.mean(), lower.mean(), upper.mean()))
        self.eta_ = _solve_eta(g, lower, upper, target)
        return self

    def transform(self, original: np.ndarray, calibrated: np.ndarray) -> np.ndarray:
        p0, q = _stabilize(original), _stabilize(calibrated)
        a, _, tail, g, lower, upper = _quantities(p0, q)
        head = _response(self.eta_, g, lower, upper)
        repaired = (1.0 - head[:, None]) * tail
        repaired[np.arange(a.size), a] = head
        return repaired


def _stabilize(p: np.ndarray) -> np.ndarray:
    p = np.asarray(p, dtype=np.float64)
    boundary = np.any(p == 0.0, axis=1)

    if not np.any(boundary):
        return p

    stabilized = p.copy()
    stabilized[boundary] = (
        (1.0 - DELTA_STAB) * stabilized[boundary]
        + DELTA_STAB / p.shape[1]
    )
    return stabilized


def _quantities(p0: np.ndarray, q: np.ndarray):
    rows = np.arange(p0.shape[0])
    a = p0.argmax(axis=1)
    b = q[rows, a]
    tail = q.copy()
    tail[rows, a] = 0.0
    tail /= tail.sum(axis=1, keepdims=True)
    rho = tail.max(axis=1)
    lower = rho / (1.0 + rho) + EPS_NUM
    upper = np.full_like(lower, 1.0 - EPS_NUM)
    g = np.where(q.argmax(axis=1) == a, b, 0.5 * (b + p0[rows, a]))
    return a, b, tail, g, lower, upper


def _interior(eta: float, g: np.ndarray) -> np.ndarray:
    radius = np.hypot(
        eta + 2.0 * g - 1.0,
        2.0 * np.sqrt(g * (1.0 - g)),
    )

    if eta > 1.0:
        return (eta - 1.0 + radius) / (2.0 * eta)

    return 2.0 * g / (1.0 - eta + radius)


def _response(
    eta: float,
    g: np.ndarray,
    lower: np.ndarray,
    upper: np.ndarray,
) -> np.ndarray:
    return np.clip(_interior(eta, g), lower, upper)


def _stationarity(s: np.ndarray, g: np.ndarray) -> np.ndarray:
    return (s - g) / (s * (1.0 - s))


def _solve_eta(
    g: np.ndarray,
    lower: np.ndarray,
    upper: np.ndarray,
    target: float,
) -> float:
    mean_at_zero = float(_response(0.0, g, lower, upper).mean())
    if mean_at_zero == target:
        return 0.0

    if mean_at_zero < target:
        left, right = 0.0, float(_stationarity(upper, g).max())
        while True:
            midpoint = left + 0.5 * (right - left)
            if midpoint == left or midpoint == right:
                return right
            if _response(midpoint, g, lower, upper).mean() < target:
                left = midpoint
            else:
                right = midpoint

    left, right = float(_stationarity(lower, g).min()), 0.0
    while True:
        midpoint = left + 0.5 * (right - left)
        if midpoint == left or midpoint == right:
            return left
        if _response(midpoint, g, lower, upper).mean() <= target:
            left = midpoint
        else:
            right = midpoint
\end{lstlisting}

\clearpage
\section{Complete Classifier--Calibrator Results}
\label{app:complete-results}
The tables below report the classifier--calibrator-level results underlying the dataset-wise averages in Figure~2 of the main paper, with the corresponding reliability diagrams providing a complementary view of confidence calibration. Across all pairs, CORD exactly recovers the original top-1 predictions and hence the original accuracy, while the paired ECE, NLL, and Brier entries show how the effects of repair vary with the revisions induced by each Base direct output.

\vspace{-0.2cm}

\subsection{CIFAR-10}
\vspace{-0.3cm}
\begin{table*}[h!]
\renewcommand{\arraystretch}{1.0}
\centering
\setlength{\tabcolsep}{2.0pt}
\resizebox{\textwidth}{!}{%
\begin{tabular}{
    l
    !{\color{black!18}\vrule width 0.25pt}
    l
    !{\color{black!18}\vrule width 0.25pt}
    r
    !{\color{black!18}\vrule width 0.25pt}
    r
    !{\color{black!18}\vrule width 0.25pt}
    r
    !{\color{black!18}\vrule width 0.25pt}
    r
    !{\color{black!18}\vrule width 0.25pt}
    r
    !{\color{black!18}\vrule width 0.25pt}
    r
    !{\color{black!18}\vrule width 0.25pt}
    r
    !{\color{black!18}\vrule width 0.25pt}
    r
    !{\color{black!18}\vrule width 0.25pt}
    r
    !{\color{black!18}\vrule width 0.25pt}
    r
    !{\color{black!18}\vrule width 0.25pt}
    r
    !{\color{black!18}\vrule width 0.25pt}
    r
    !{\color{black!18}\vrule width 0.25pt}
    r
    !{\color{black!18}\vrule width 0.25pt}
    r
    !{\color{black!18}\vrule width 0.25pt}
    r
}
\toprule
\multirow{2}{*}{\textbf{Classifier}}
& \multirow{2}{*}{\textbf{Metric}}
& \multirow{2}{*}{\textbf{Uncal.}}
& \multicolumn{2}{c}{\textbf{VS}}
& \multicolumn{2}{c}{\textbf{SVS}}
& \multicolumn{2}{c}{\textbf{MS}}
& \multicolumn{2}{c}{\textbf{SMS}}
& \multicolumn{2}{c}{\textbf{Dir-ODIR}}
& \multicolumn{2}{c}{\textbf{IROvA}}
& \multicolumn{2}{c}{\textbf{IROvA-TS}} \\
\cmidrule(lr){4-5}
\cmidrule(lr){6-7}
\cmidrule(lr){8-9}
\cmidrule(lr){10-11}
\cmidrule(lr){12-13}
\cmidrule(lr){14-15}
\cmidrule(lr){16-17}
& & &
Base & \textbf{+ CORD}
& Base & \textbf{+ CORD}
& Base & \textbf{+ CORD}
& Base & \textbf{+ CORD}
& Base & \textbf{+ CORD}
& Base & \textbf{+ CORD}
& Base & \textbf{+ CORD} \\
\midrule

\multirow{5}{*}{VGG-16-BN}
& Acc. (\%) & 94.09 & -0.04 & \cellcolor{gray!12}\textbf{0.00} & 0.00 & \cellcolor{gray!12}\textbf{0.00} & -0.12 & \cellcolor{gray!12}\textbf{0.00} & -0.02 & \cellcolor{gray!12}\textbf{0.00} & -0.08 & \cellcolor{gray!12}\textbf{0.00} & -0.08 & \cellcolor{gray!12}\textbf{0.00} & -0.10 & \cellcolor{gray!12}\textbf{0.00} \\
& TPCR (\%) & 0.00 & 0.35 & \cellcolor{gray!12}\textbf{0.00} & 0.23 & \cellcolor{gray!12}\textbf{0.00} & 0.86 & \cellcolor{gray!12}\textbf{0.00} & 0.29 & \cellcolor{gray!12}\textbf{0.00} & 0.59 & \cellcolor{gray!12}\textbf{0.00} & 0.42 & \cellcolor{gray!12}\textbf{0.00} & 0.54 & \cellcolor{gray!12}\textbf{0.00} \\
& ECE (\%) & 4.79 & 1.46 & \cellcolor{green!8}\textbf{0.04~{\color{green!50!black}\ensuremath{\blacktriangledown}}} & 1.56 & \cellcolor{red!8}\textbf{-0.05~{\color{red!75!black}\ensuremath{\blacktriangle}}} & 1.52 & \cellcolor{green!8}\textbf{0.11~{\color{green!50!black}\ensuremath{\blacktriangledown}}} & 1.54 & \cellcolor{red!8}\textbf{-0.01~{\color{red!75!black}\ensuremath{\blacktriangle}}} & 1.80 & \cellcolor{green!8}\textbf{0.10~{\color{green!50!black}\ensuremath{\blacktriangledown}}} & 1.70 & \cellcolor{green!8}\textbf{0.02~{\color{green!50!black}\ensuremath{\blacktriangledown}}} & 1.59 & \cellcolor{gray!12}\textbf{0.00} \\
& NLL & 0.3353 & 0.2243 & \cellcolor{green!8}\textbf{0.0002~{\color{green!50!black}\ensuremath{\blacktriangledown}}} & 0.2253 & \cellcolor{gray!12}\textbf{0.0000} & 0.2275 & \cellcolor{green!8}\textbf{0.0019~{\color{green!50!black}\ensuremath{\blacktriangledown}}} & 0.2191 & \cellcolor{gray!12}\textbf{0.0000} & 0.2231 & \cellcolor{green!8}\textbf{0.0005~{\color{green!50!black}\ensuremath{\blacktriangledown}}} & 0.2759 & \cellcolor{green!8}\textbf{0.0036~{\color{green!50!black}\ensuremath{\blacktriangledown}}} & 0.2688 & \cellcolor{green!8}\textbf{0.0028~{\color{green!50!black}\ensuremath{\blacktriangledown}}} \\
& Brier & 0.1056 & 0.0941 & \cellcolor{green!8}\textbf{0.0001~{\color{green!50!black}\ensuremath{\blacktriangledown}}} & 0.0944 & \cellcolor{gray!12}\textbf{0.0000} & 0.0956 & \cellcolor{green!8}\textbf{0.0010~{\color{green!50!black}\ensuremath{\blacktriangledown}}} & 0.0939 & \cellcolor{gray!12}\textbf{0.0000} & 0.0947 & \cellcolor{green!8}\textbf{0.0004~{\color{green!50!black}\ensuremath{\blacktriangledown}}} & 0.0929 & \cellcolor{green!8}\textbf{0.0002~{\color{green!50!black}\ensuremath{\blacktriangledown}}} & 0.0924 & \cellcolor{green!8}\textbf{0.0002~{\color{green!50!black}\ensuremath{\blacktriangledown}}} \\

\cmidrule(lr){1-17}
\multirow{5}{*}{ResNet-56}
& Acc. (\%) & 94.38 & -0.10 & \cellcolor{gray!12}\textbf{0.00} & -0.04 & \cellcolor{gray!12}\textbf{0.00} & -0.26 & \cellcolor{gray!12}\textbf{0.00} & -0.11 & \cellcolor{gray!12}\textbf{0.00} & -0.09 & \cellcolor{gray!12}\textbf{0.00} & -0.04 & \cellcolor{gray!12}\textbf{0.00} & -0.06 & \cellcolor{gray!12}\textbf{0.00} \\
& TPCR (\%) & 0.00 & 0.48 & \cellcolor{gray!12}\textbf{0.00} & 0.24 & \cellcolor{gray!12}\textbf{0.00} & 1.25 & \cellcolor{gray!12}\textbf{0.00} & 0.58 & \cellcolor{gray!12}\textbf{0.00} & 1.01 & \cellcolor{gray!12}\textbf{0.00} & 0.62 & \cellcolor{gray!12}\textbf{0.00} & 0.76 & \cellcolor{gray!12}\textbf{0.00} \\
& ECE (\%) & 3.75 & 1.00 & \cellcolor{green!8}\textbf{0.05~{\color{green!50!black}\ensuremath{\blacktriangledown}}} & 0.94 & \cellcolor{green!8}\textbf{0.02~{\color{green!50!black}\ensuremath{\blacktriangledown}}} & 1.12 & \cellcolor{green!8}\textbf{0.27~{\color{green!50!black}\ensuremath{\blacktriangledown}}} & 1.00 & \cellcolor{green!8}\textbf{0.01~{\color{green!50!black}\ensuremath{\blacktriangledown}}} & 1.01 & \cellcolor{green!8}\textbf{0.07~{\color{green!50!black}\ensuremath{\blacktriangledown}}} & 1.33 & \cellcolor{red!8}\textbf{-0.01~{\color{red!75!black}\ensuremath{\blacktriangle}}} & 1.11 & \cellcolor{green!8}\textbf{0.01~{\color{green!50!black}\ensuremath{\blacktriangledown}}} \\
& NLL & 0.2525 & 0.1884 & \cellcolor{green!8}\textbf{0.0003~{\color{green!50!black}\ensuremath{\blacktriangledown}}} & 0.1889 & \cellcolor{green!8}\textbf{0.0001~{\color{green!50!black}\ensuremath{\blacktriangledown}}} & 0.1884 & \cellcolor{green!8}\textbf{0.0017~{\color{green!50!black}\ensuremath{\blacktriangledown}}} & 0.1841 & \cellcolor{green!8}\textbf{0.0003~{\color{green!50!black}\ensuremath{\blacktriangledown}}} & 0.1862 & \cellcolor{green!8}\textbf{0.0007~{\color{green!50!black}\ensuremath{\blacktriangledown}}} & 0.2302 & \cellcolor{green!8}\textbf{0.0033~{\color{green!50!black}\ensuremath{\blacktriangledown}}} & 0.2250 & \cellcolor{green!8}\textbf{0.0032~{\color{green!50!black}\ensuremath{\blacktriangledown}}} \\
& Brier & 0.0939 & 0.0866 & \cellcolor{green!8}\textbf{0.0002~{\color{green!50!black}\ensuremath{\blacktriangledown}}} & 0.0863 & \cellcolor{gray!12}\textbf{0.0000} & 0.0886 & \cellcolor{green!8}\textbf{0.0012~{\color{green!50!black}\ensuremath{\blacktriangledown}}} & 0.0861 & \cellcolor{green!8}\textbf{0.0002~{\color{green!50!black}\ensuremath{\blacktriangledown}}} & 0.0872 & \cellcolor{green!8}\textbf{0.0005~{\color{green!50!black}\ensuremath{\blacktriangledown}}} & 0.0861 & \cellcolor{gray!12}\textbf{0.0000} & 0.0858 & \cellcolor{green!8}\textbf{0.0001~{\color{green!50!black}\ensuremath{\blacktriangledown}}} \\

\cmidrule(lr){1-17}
\multirow{5}{*}{WRN-26-10}
& Acc. (\%) & 96.03 & 0.00 & \cellcolor{gray!12}\textbf{0.00} & +0.02 & \cellcolor{gray!12}\textbf{0.00} & -0.17 & \cellcolor{gray!12}\textbf{0.00} & -0.01 & \cellcolor{gray!12}\textbf{0.00} & +0.06 & \cellcolor{gray!12}\textbf{0.00} & +0.04 & \cellcolor{gray!12}\textbf{0.00} & +0.04 & \cellcolor{gray!12}\textbf{0.00} \\
& TPCR (\%) & 0.00 & 0.38 & \cellcolor{gray!12}\textbf{0.00} & 0.18 & \cellcolor{gray!12}\textbf{0.00} & 0.94 & \cellcolor{gray!12}\textbf{0.00} & 0.30 & \cellcolor{gray!12}\textbf{0.00} & 0.63 & \cellcolor{gray!12}\textbf{0.00} & 0.92 & \cellcolor{gray!12}\textbf{0.00} & 0.80 & \cellcolor{gray!12}\textbf{0.00} \\
& ECE (\%) & 3.36 & 1.14 & \cellcolor{green!8}\textbf{0.02~{\color{green!50!black}\ensuremath{\blacktriangledown}}} & 1.17 & \cellcolor{gray!12}\textbf{0.00} & 1.27 & \cellcolor{green!8}\textbf{0.20~{\color{green!50!black}\ensuremath{\blacktriangledown}}} & 1.24 & \cellcolor{green!8}\textbf{0.04~{\color{green!50!black}\ensuremath{\blacktriangledown}}} & 1.20 & \cellcolor{green!8}\textbf{0.05~{\color{green!50!black}\ensuremath{\blacktriangledown}}} & 1.21 & \cellcolor{green!8}\textbf{0.13~{\color{green!50!black}\ensuremath{\blacktriangledown}}} & 1.13 & \cellcolor{green!8}\textbf{0.11~{\color{green!50!black}\ensuremath{\blacktriangledown}}} \\
& NLL & 0.2775 & 0.1514 & \cellcolor{red!8}\textbf{-0.0001~{\color{red!75!black}\ensuremath{\blacktriangle}}} & 0.1539 & \cellcolor{red!8}\textbf{-0.0001~{\color{red!75!black}\ensuremath{\blacktriangle}}} & 0.1602 & \cellcolor{green!8}\textbf{0.0053~{\color{green!50!black}\ensuremath{\blacktriangledown}}} & 0.1479 & \cellcolor{green!8}\textbf{0.0001~{\color{green!50!black}\ensuremath{\blacktriangledown}}} & 0.1492 & \cellcolor{green!8}\textbf{0.0007~{\color{green!50!black}\ensuremath{\blacktriangledown}}} & 0.2019 & \cellcolor{green!8}\textbf{0.0036~{\color{green!50!black}\ensuremath{\blacktriangledown}}} & 0.1935 & \cellcolor{green!8}\textbf{0.0029~{\color{green!50!black}\ensuremath{\blacktriangledown}}} \\
& Brier & 0.0711 & 0.0621 & \cellcolor{gray!12}\textbf{0.0000} & 0.0626 & \cellcolor{red!8}\textbf{-0.0001~{\color{red!75!black}\ensuremath{\blacktriangle}}} & 0.0646 & \cellcolor{green!8}\textbf{0.0021~{\color{green!50!black}\ensuremath{\blacktriangledown}}} & 0.0621 & \cellcolor{green!8}\textbf{0.0001~{\color{green!50!black}\ensuremath{\blacktriangledown}}} & 0.0622 & \cellcolor{green!8}\textbf{0.0003~{\color{green!50!black}\ensuremath{\blacktriangledown}}} & 0.0622 & \cellcolor{gray!12}\textbf{0.0000} & 0.0615 & \cellcolor{red!8}\textbf{-0.0001~{\color{red!75!black}\ensuremath{\blacktriangle}}} \\

\cmidrule(lr){1-17}
\multirow{5}{*}{DenseNet-121}
& Acc. (\%) & 94.88 & +0.08 & \cellcolor{gray!12}\textbf{0.00} & +0.09 & \cellcolor{gray!12}\textbf{0.00} & -0.08 & \cellcolor{gray!12}\textbf{0.00} & +0.07 & \cellcolor{gray!12}\textbf{0.00} & +0.02 & \cellcolor{gray!12}\textbf{0.00} & +0.04 & \cellcolor{gray!12}\textbf{0.00} & -0.01 & \cellcolor{gray!12}\textbf{0.00} \\
& TPCR (\%) & 0.00 & 0.48 & \cellcolor{gray!12}\textbf{0.00} & 0.30 & \cellcolor{gray!12}\textbf{0.00} & 0.94 & \cellcolor{gray!12}\textbf{0.00} & 0.41 & \cellcolor{gray!12}\textbf{0.00} & 0.72 & \cellcolor{gray!12}\textbf{0.00} & 0.66 & \cellcolor{gray!12}\textbf{0.00} & 0.82 & \cellcolor{gray!12}\textbf{0.00} \\
& ECE (\%) & 4.66 & 1.48 & \cellcolor{green!8}\textbf{0.03~{\color{green!50!black}\ensuremath{\blacktriangledown}}} & 1.38 & \cellcolor{red!8}\textbf{-0.05~{\color{red!75!black}\ensuremath{\blacktriangle}}} & 1.51 & \cellcolor{green!8}\textbf{0.15~{\color{green!50!black}\ensuremath{\blacktriangledown}}} & 1.57 & \cellcolor{green!8}\textbf{0.03~{\color{green!50!black}\ensuremath{\blacktriangledown}}} & 1.56 & \cellcolor{green!8}\textbf{0.09~{\color{green!50!black}\ensuremath{\blacktriangledown}}} & 1.32 & \cellcolor{green!8}\textbf{0.02~{\color{green!50!black}\ensuremath{\blacktriangledown}}} & 1.34 & \cellcolor{green!8}\textbf{0.15~{\color{green!50!black}\ensuremath{\blacktriangledown}}} \\
& NLL & 0.4390 & 0.2095 & \cellcolor{gray!12}\textbf{0.0000} & 0.2117 & \cellcolor{red!8}\textbf{-0.0001~{\color{red!75!black}\ensuremath{\blacktriangle}}} & 0.2082 & \cellcolor{green!8}\textbf{0.0024~{\color{green!50!black}\ensuremath{\blacktriangledown}}} & 0.2021 & \cellcolor{green!8}\textbf{0.0001~{\color{green!50!black}\ensuremath{\blacktriangledown}}} & 0.2028 & \cellcolor{green!8}\textbf{0.0006~{\color{green!50!black}\ensuremath{\blacktriangledown}}} & 0.2524 & \cellcolor{green!8}\textbf{0.0022~{\color{green!50!black}\ensuremath{\blacktriangledown}}} & 0.2446 & \cellcolor{green!8}\textbf{0.0024~{\color{green!50!black}\ensuremath{\blacktriangledown}}} \\
& Brier & 0.0950 & 0.0838 & \cellcolor{red!8}\textbf{-0.0001~{\color{red!75!black}\ensuremath{\blacktriangle}}} & 0.0844 & \cellcolor{red!8}\textbf{-0.0001~{\color{red!75!black}\ensuremath{\blacktriangle}}} & 0.0848 & \cellcolor{green!8}\textbf{0.0013~{\color{green!50!black}\ensuremath{\blacktriangledown}}} & 0.0835 & \cellcolor{gray!12}\textbf{0.0000} & 0.0838 & \cellcolor{green!8}\textbf{0.0004~{\color{green!50!black}\ensuremath{\blacktriangledown}}} & 0.0825 & \cellcolor{gray!12}\textbf{0.0000} & 0.0820 & \cellcolor{green!8}\textbf{0.0002~{\color{green!50!black}\ensuremath{\blacktriangledown}}} \\

\cmidrule(lr){1-17}
\multirow{5}{*}{MobileNetV2-1.4$\times$}
& Acc. (\%) & 94.14 & -0.13 & \cellcolor{gray!12}\textbf{0.00} & -0.06 & \cellcolor{gray!12}\textbf{0.00} & -0.34 & \cellcolor{gray!12}\textbf{0.00} & -0.07 & \cellcolor{gray!12}\textbf{0.00} & -0.21 & \cellcolor{gray!12}\textbf{0.00} & -0.13 & \cellcolor{gray!12}\textbf{0.00} & -0.11 & \cellcolor{gray!12}\textbf{0.00} \\
& TPCR (\%) & 0.00 & 0.80 & \cellcolor{gray!12}\textbf{0.00} & 0.48 & \cellcolor{gray!12}\textbf{0.00} & 1.52 & \cellcolor{gray!12}\textbf{0.00} & 0.67 & \cellcolor{gray!12}\textbf{0.00} & 1.08 & \cellcolor{gray!12}\textbf{0.00} & 0.93 & \cellcolor{gray!12}\textbf{0.00} & 0.97 & \cellcolor{gray!12}\textbf{0.00} \\
& ECE (\%) & 3.73 & 1.22 & \cellcolor{green!8}\textbf{0.02~{\color{green!50!black}\ensuremath{\blacktriangledown}}} & 1.31 & \cellcolor{green!8}\textbf{0.05~{\color{green!50!black}\ensuremath{\blacktriangledown}}} & 1.50 & \cellcolor{green!8}\textbf{0.34~{\color{green!50!black}\ensuremath{\blacktriangledown}}} & 1.20 & \cellcolor{red!8}\textbf{-0.02~{\color{red!75!black}\ensuremath{\blacktriangle}}} & 1.42 & \cellcolor{green!8}\textbf{0.09~{\color{green!50!black}\ensuremath{\blacktriangledown}}} & 1.45 & \cellcolor{red!8}\textbf{-0.03~{\color{red!75!black}\ensuremath{\blacktriangle}}} & 1.34 & \cellcolor{green!8}\textbf{0.04~{\color{green!50!black}\ensuremath{\blacktriangledown}}} \\
& NLL & 0.2509 & 0.1972 & \cellcolor{green!8}\textbf{0.0006~{\color{green!50!black}\ensuremath{\blacktriangledown}}} & 0.1984 & \cellcolor{green!8}\textbf{0.0001~{\color{green!50!black}\ensuremath{\blacktriangledown}}} & 0.1987 & \cellcolor{green!8}\textbf{0.0028~{\color{green!50!black}\ensuremath{\blacktriangledown}}} & 0.1939 & \cellcolor{green!8}\textbf{0.0004~{\color{green!50!black}\ensuremath{\blacktriangledown}}} & 0.1969 & \cellcolor{green!8}\textbf{0.0012~{\color{green!50!black}\ensuremath{\blacktriangledown}}} & 0.2320 & \cellcolor{green!8}\textbf{0.0024~{\color{green!50!black}\ensuremath{\blacktriangledown}}} & 0.2284 & \cellcolor{green!8}\textbf{0.0020~{\color{green!50!black}\ensuremath{\blacktriangledown}}} \\
& Brier & 0.0977 & 0.0909 & \cellcolor{green!8}\textbf{0.0005~{\color{green!50!black}\ensuremath{\blacktriangledown}}} & 0.0906 & \cellcolor{green!8}\textbf{0.0001~{\color{green!50!black}\ensuremath{\blacktriangledown}}} & 0.0933 & \cellcolor{green!8}\textbf{0.0019~{\color{green!50!black}\ensuremath{\blacktriangledown}}} & 0.0907 & \cellcolor{green!8}\textbf{0.0004~{\color{green!50!black}\ensuremath{\blacktriangledown}}} & 0.0920 & \cellcolor{green!8}\textbf{0.0010~{\color{green!50!black}\ensuremath{\blacktriangledown}}} & 0.0914 & \cellcolor{green!8}\textbf{0.0002~{\color{green!50!black}\ensuremath{\blacktriangledown}}} & 0.0910 & \cellcolor{green!8}\textbf{0.0002~{\color{green!50!black}\ensuremath{\blacktriangledown}}} \\

\cmidrule(lr){1-17}
\multirow{5}{*}{ShuffleNetV2-1.0$\times$}
& Acc. (\%) & 93.35 & -0.03 & \cellcolor{gray!12}\textbf{0.00} & +0.02 & \cellcolor{gray!12}\textbf{0.00} & -0.25 & \cellcolor{gray!12}\textbf{0.00} & -0.03 & \cellcolor{gray!12}\textbf{0.00} & -0.20 & \cellcolor{gray!12}\textbf{0.00} & -0.08 & \cellcolor{gray!12}\textbf{0.00} & -0.08 & \cellcolor{gray!12}\textbf{0.00} \\
& TPCR (\%) & 0.00 & 0.65 & \cellcolor{gray!12}\textbf{0.00} & 0.37 & \cellcolor{gray!12}\textbf{0.00} & 1.56 & \cellcolor{gray!12}\textbf{0.00} & 0.78 & \cellcolor{gray!12}\textbf{0.00} & 1.38 & \cellcolor{gray!12}\textbf{0.00} & 1.01 & \cellcolor{gray!12}\textbf{0.00} & 0.92 & \cellcolor{gray!12}\textbf{0.00} \\
& ECE (\%) & 4.03 & 1.06 & \cellcolor{green!8}\textbf{0.02~{\color{green!50!black}\ensuremath{\blacktriangledown}}} & 1.10 & \cellcolor{green!8}\textbf{0.02~{\color{green!50!black}\ensuremath{\blacktriangledown}}} & 1.22 & \cellcolor{green!8}\textbf{0.10~{\color{green!50!black}\ensuremath{\blacktriangledown}}} & 1.10 & \cellcolor{green!8}\textbf{0.02~{\color{green!50!black}\ensuremath{\blacktriangledown}}} & 1.11 & \cellcolor{red!8}\textbf{-0.05~{\color{red!75!black}\ensuremath{\blacktriangle}}} & 1.52 & \cellcolor{green!8}\textbf{0.11~{\color{green!50!black}\ensuremath{\blacktriangledown}}} & 1.20 & \cellcolor{green!8}\textbf{0.03~{\color{green!50!black}\ensuremath{\blacktriangledown}}} \\
& NLL & 0.2839 & 0.2222 & \cellcolor{green!8}\textbf{0.0002~{\color{green!50!black}\ensuremath{\blacktriangledown}}} & 0.2226 & \cellcolor{gray!12}\textbf{0.0000} & 0.2245 & \cellcolor{green!8}\textbf{0.0015~{\color{green!50!black}\ensuremath{\blacktriangledown}}} & 0.2194 & \cellcolor{green!8}\textbf{0.0001~{\color{green!50!black}\ensuremath{\blacktriangledown}}} & 0.2248 & \cellcolor{green!8}\textbf{0.0011~{\color{green!50!black}\ensuremath{\blacktriangledown}}} & 0.2766 & \cellcolor{green!8}\textbf{0.0022~{\color{green!50!black}\ensuremath{\blacktriangledown}}} & 0.2738 & \cellcolor{green!8}\textbf{0.0018~{\color{green!50!black}\ensuremath{\blacktriangledown}}} \\
& Brier & 0.1091 & 0.1024 & \cellcolor{green!8}\textbf{0.0001~{\color{green!50!black}\ensuremath{\blacktriangledown}}} & 0.1023 & \cellcolor{gray!12}\textbf{0.0000} & 0.1048 & \cellcolor{green!8}\textbf{0.0011~{\color{green!50!black}\ensuremath{\blacktriangledown}}} & 0.1023 & \cellcolor{green!8}\textbf{0.0001~{\color{green!50!black}\ensuremath{\blacktriangledown}}} & 0.1044 & \cellcolor{green!8}\textbf{0.0008~{\color{green!50!black}\ensuremath{\blacktriangledown}}} & 0.1033 & \cellcolor{green!8}\textbf{0.0003~{\color{green!50!black}\ensuremath{\blacktriangledown}}} & 0.1028 & \cellcolor{green!8}\textbf{0.0002~{\color{green!50!black}\ensuremath{\blacktriangledown}}} \\

\cmidrule(lr){1-17}
\multirow{5}{*}{RepVGG-A1}
& Acc. (\%) & 94.78 & -0.08 & \cellcolor{gray!12}\textbf{0.00} & -0.04 & \cellcolor{gray!12}\textbf{0.00} & -0.15 & \cellcolor{gray!12}\textbf{0.00} & -0.08 & \cellcolor{gray!12}\textbf{0.00} & -0.18 & \cellcolor{gray!12}\textbf{0.00} & -0.05 & \cellcolor{gray!12}\textbf{0.00} & -0.04 & \cellcolor{gray!12}\textbf{0.00} \\
& TPCR (\%) & 0.00 & 0.39 & \cellcolor{gray!12}\textbf{0.00} & 0.20 & \cellcolor{gray!12}\textbf{0.00} & 1.20 & \cellcolor{gray!12}\textbf{0.00} & 0.46 & \cellcolor{gray!12}\textbf{0.00} & 1.00 & \cellcolor{gray!12}\textbf{0.00} & 0.62 & \cellcolor{gray!12}\textbf{0.00} & 0.71 & \cellcolor{gray!12}\textbf{0.00} \\
& ECE (\%) & 3.49 & 1.18 & \cellcolor{green!8}\textbf{0.04~{\color{green!50!black}\ensuremath{\blacktriangledown}}} & 1.15 & \cellcolor{green!8}\textbf{0.02~{\color{green!50!black}\ensuremath{\blacktriangledown}}} & 1.32 & \cellcolor{green!8}\textbf{0.19~{\color{green!50!black}\ensuremath{\blacktriangledown}}} & 1.13 & \cellcolor{green!8}\textbf{0.04~{\color{green!50!black}\ensuremath{\blacktriangledown}}} & 1.25 & \cellcolor{green!8}\textbf{0.25~{\color{green!50!black}\ensuremath{\blacktriangledown}}} & 1.49 & \cellcolor{green!8}\textbf{0.05~{\color{green!50!black}\ensuremath{\blacktriangledown}}} & 1.28 & \cellcolor{green!8}\textbf{0.04~{\color{green!50!black}\ensuremath{\blacktriangledown}}} \\
& NLL & 0.2319 & 0.1807 & \cellcolor{green!8}\textbf{0.0004~{\color{green!50!black}\ensuremath{\blacktriangledown}}} & 0.1812 & \cellcolor{green!8}\textbf{0.0001~{\color{green!50!black}\ensuremath{\blacktriangledown}}} & 0.1804 & \cellcolor{green!8}\textbf{0.0020~{\color{green!50!black}\ensuremath{\blacktriangledown}}} & 0.1768 & \cellcolor{green!8}\textbf{0.0003~{\color{green!50!black}\ensuremath{\blacktriangledown}}} & 0.1800 & \cellcolor{green!8}\textbf{0.0012~{\color{green!50!black}\ensuremath{\blacktriangledown}}} & 0.2328 & \cellcolor{green!8}\textbf{0.0044~{\color{green!50!black}\ensuremath{\blacktriangledown}}} & 0.2269 & \cellcolor{green!8}\textbf{0.0038~{\color{green!50!black}\ensuremath{\blacktriangledown}}} \\
& Brier & 0.0871 & 0.0806 & \cellcolor{green!8}\textbf{0.0002~{\color{green!50!black}\ensuremath{\blacktriangledown}}} & 0.0806 & \cellcolor{green!8}\textbf{0.0001~{\color{green!50!black}\ensuremath{\blacktriangledown}}} & 0.0819 & \cellcolor{green!8}\textbf{0.0011~{\color{green!50!black}\ensuremath{\blacktriangledown}}} & 0.0802 & \cellcolor{green!8}\textbf{0.0002~{\color{green!50!black}\ensuremath{\blacktriangledown}}} & 0.0813 & \cellcolor{green!8}\textbf{0.0008~{\color{green!50!black}\ensuremath{\blacktriangledown}}} & 0.0806 & \cellcolor{green!8}\textbf{0.0002~{\color{green!50!black}\ensuremath{\blacktriangledown}}} & 0.0805 & \cellcolor{green!8}\textbf{0.0001~{\color{green!50!black}\ensuremath{\blacktriangledown}}} \\
\bottomrule
\end{tabular}%
}
\vspace{-0.2cm}
\caption{Complete classifier--calibrator results on CIFAR-10 corresponding to Figure~2 of the main paper. Values are five-split means. Uncal. reports absolute values. Within each Base/+CORD pair, Acc. reports $\Delta\mathrm{Acc}$ from Uncal. in percentage points, TPCR reports absolute rates in percent, and $M\in\{\mathrm{ECE},\mathrm{NLL},\mathrm{Brier}\}$ is reported as $M(\mathrm{Base})/\Delta_{\mathrm R}M$, where $\Delta_{\mathrm R}M:=M(\mathrm{Base})-M(\mathrm{Base}+\mathrm{CORD})$; positive values indicate improvement.}
\label{tab:complete-results-cifar10}
\end{table*}

\begin{figure*}[h!]
    \centering
    \includegraphics[width=\linewidth]
    {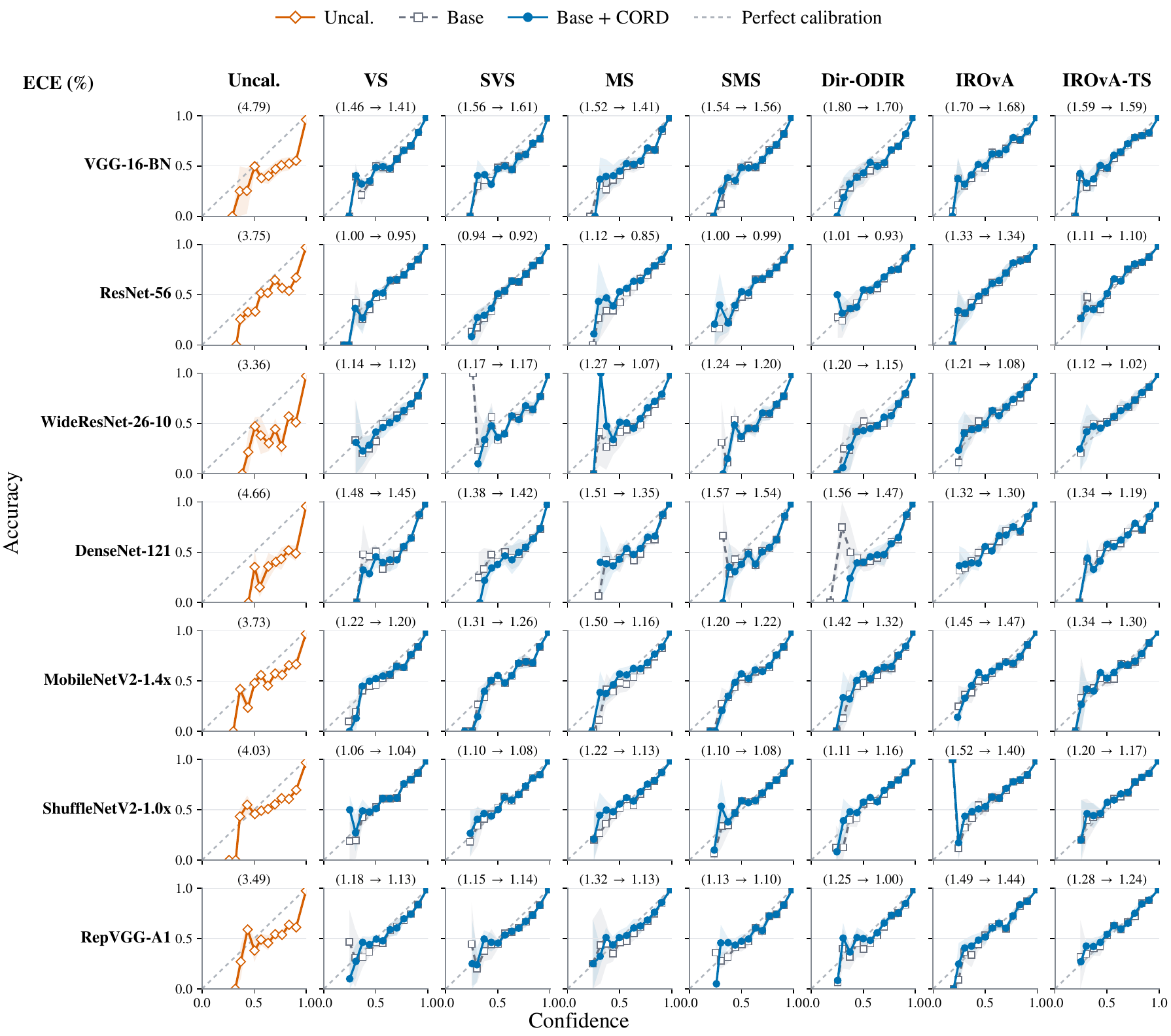}
    \caption{Reliability diagrams on CIFAR-10 corresponding to Table~\ref{tab:complete-results-cifar10}. Rows denote classifiers and columns denote Uncal. and Base calibrators; parenthetical values report ECE (\%) for Uncal. or Base $\rightarrow$ Base + CORD. Curves are averaged over five splits; shading shows pointwise 95\% CIs.}
    \label{fig:reliability-cifar10}
\end{figure*}

\clearpage
\subsection{CIFAR-100}
\vspace{-0.3cm}
\begin{table*}[h!]
\renewcommand{\arraystretch}{1.0}
\centering
\setlength{\tabcolsep}{2.0pt}
\resizebox{1.0\textwidth}{!}{%
\begin{tabular}{
    l
    !{\color{black!18}\vrule width 0.25pt}
    l
    !{\color{black!18}\vrule width 0.25pt}
    r
    !{\color{black!18}\vrule width 0.25pt}
    r
    !{\color{black!18}\vrule width 0.25pt}
    r
    !{\color{black!18}\vrule width 0.25pt}
    r
    !{\color{black!18}\vrule width 0.25pt}
    r
    !{\color{black!18}\vrule width 0.25pt}
    r
    !{\color{black!18}\vrule width 0.25pt}
    r
    !{\color{black!18}\vrule width 0.25pt}
    r
    !{\color{black!18}\vrule width 0.25pt}
    r
    !{\color{black!18}\vrule width 0.25pt}
    r
    !{\color{black!18}\vrule width 0.25pt}
    r
    !{\color{black!18}\vrule width 0.25pt}
    r
    !{\color{black!18}\vrule width 0.25pt}
    r
    !{\color{black!18}\vrule width 0.25pt}
    r
    !{\color{black!18}\vrule width 0.25pt}
    r
}
\toprule
\multirow{2}{*}{\textbf{Classifier}}
& \multirow{2}{*}{\textbf{Metric}}
& \multirow{2}{*}{\textbf{Uncal.}}
& \multicolumn{2}{c}{\textbf{VS}}
& \multicolumn{2}{c}{\textbf{SVS}}
& \multicolumn{2}{c}{\textbf{MS}}
& \multicolumn{2}{c}{\textbf{SMS}}
& \multicolumn{2}{c}{\textbf{Dir-ODIR}}
& \multicolumn{2}{c}{\textbf{IROvA}}
& \multicolumn{2}{c}{\textbf{IROvA-TS}} \\
\cmidrule(lr){4-5}
\cmidrule(lr){6-7}
\cmidrule(lr){8-9}
\cmidrule(lr){10-11}
\cmidrule(lr){12-13}
\cmidrule(lr){14-15}
\cmidrule(lr){16-17}
& & &
Base & \textbf{+ CORD}
& Base & \textbf{+ CORD}
& Base & \textbf{+ CORD}
& Base & \textbf{+ CORD}
& Base & \textbf{+ CORD}
& Base & \textbf{+ CORD}
& Base & \textbf{+ CORD} \\
\midrule

\multirow{5}{*}{VGG-16-BN}
& Acc. (\%) & 71.27 & +0.26 & \cellcolor{gray!12}\textbf{0.00} & +0.04 & \cellcolor{gray!12}\textbf{0.00} & -2.14 & \cellcolor{gray!12}\textbf{0.00} & +0.14 & \cellcolor{gray!12}\textbf{0.00} & +0.20 & \cellcolor{gray!12}\textbf{0.00} & -0.09 & \cellcolor{gray!12}\textbf{0.00} & -0.53 & \cellcolor{gray!12}\textbf{0.00} \\
& TPCR (\%) & 0.00 & 5.08 & \cellcolor{gray!12}\textbf{0.00} & 1.01 & \cellcolor{gray!12}\textbf{0.00} & 16.55 & \cellcolor{gray!12}\textbf{0.00} & 3.24 & \cellcolor{gray!12}\textbf{0.00} & 6.27 & \cellcolor{gray!12}\textbf{0.00} & 6.24 & \cellcolor{gray!12}\textbf{0.00} & 7.41 & \cellcolor{gray!12}\textbf{0.00} \\
& ECE (\%) & 20.66 & 4.43 & \cellcolor{green!8}\textbf{0.38~{\color{green!50!black}\ensuremath{\blacktriangledown}}} & 4.09 & \cellcolor{green!8}\textbf{0.06~{\color{green!50!black}\ensuremath{\blacktriangledown}}} & 9.77 & \cellcolor{green!8}\textbf{6.91~{\color{green!50!black}\ensuremath{\blacktriangledown}}} & 3.91 & \cellcolor{red!8}\textbf{-0.02~{\color{red!75!black}\ensuremath{\blacktriangle}}} & 3.45 & \cellcolor{green!8}\textbf{0.20~{\color{green!50!black}\ensuremath{\blacktriangledown}}} & 7.22 & \cellcolor{green!8}\textbf{0.82~{\color{green!50!black}\ensuremath{\blacktriangledown}}} & 4.23 & \cellcolor{green!8}\textbf{1.08~{\color{green!50!black}\ensuremath{\blacktriangledown}}} \\
& NLL & 1.8368 & 1.2272 & \cellcolor{green!8}\textbf{0.0006~{\color{green!50!black}\ensuremath{\blacktriangledown}}} & 1.2261 & \cellcolor{red!8}\textbf{-0.0002~{\color{red!75!black}\ensuremath{\blacktriangle}}} & 1.7526 & \cellcolor{green!8}\textbf{0.0659~{\color{green!50!black}\ensuremath{\blacktriangledown}}} & 1.2167 & \cellcolor{red!8}\textbf{-0.0007~{\color{red!75!black}\ensuremath{\blacktriangle}}} & 1.2100 & \cellcolor{green!8}\textbf{0.0016~{\color{green!50!black}\ensuremath{\blacktriangledown}}} & 1.7156 & \cellcolor{green!8}\textbf{0.0103~{\color{green!50!black}\ensuremath{\blacktriangledown}}} & 1.6109 & \cellcolor{green!8}\textbf{0.0083~{\color{green!50!black}\ensuremath{\blacktriangledown}}} \\
& Brier & 0.4788 & 0.4003 & \cellcolor{green!8}\textbf{0.0001~{\color{green!50!black}\ensuremath{\blacktriangledown}}} & 0.4034 & \cellcolor{gray!12}\textbf{0.0000} & 0.4487 & \cellcolor{green!8}\textbf{0.0229~{\color{green!50!black}\ensuremath{\blacktriangledown}}} & 0.3997 & \cellcolor{red!8}\textbf{-0.0002~{\color{red!75!black}\ensuremath{\blacktriangle}}} & 0.4000 & \cellcolor{green!8}\textbf{0.0005~{\color{green!50!black}\ensuremath{\blacktriangledown}}} & 0.4099 & \cellcolor{green!8}\textbf{0.0023~{\color{green!50!black}\ensuremath{\blacktriangledown}}} & 0.4017 & \cellcolor{green!8}\textbf{0.0020~{\color{green!50!black}\ensuremath{\blacktriangledown}}} \\

\cmidrule(lr){1-17}
\multirow{5}{*}{ResNet-56}
& Acc. (\%) & 67.64 & +0.27 & \cellcolor{gray!12}\textbf{0.00} & +0.15 & \cellcolor{gray!12}\textbf{0.00} & -13.06 & \cellcolor{gray!12}\textbf{0.00} & +0.32 & \cellcolor{gray!12}\textbf{0.00} & +0.15 & \cellcolor{gray!12}\textbf{0.00} & -0.19 & \cellcolor{gray!12}\textbf{0.00} & -0.27 & \cellcolor{gray!12}\textbf{0.00} \\
& TPCR (\%) & 0.00 & 9.34 & \cellcolor{gray!12}\textbf{0.00} & 2.04 & \cellcolor{gray!12}\textbf{0.00} & 39.96 & \cellcolor{gray!12}\textbf{0.00} & 6.51 & \cellcolor{gray!12}\textbf{0.00} & 11.58 & \cellcolor{gray!12}\textbf{0.00} & 8.30 & \cellcolor{gray!12}\textbf{0.00} & 9.78 & \cellcolor{gray!12}\textbf{0.00} \\
& ECE (\%) & 15.97 & 3.30 & \cellcolor{green!8}\textbf{0.77~{\color{green!50!black}\ensuremath{\blacktriangledown}}} & 2.60 & \cellcolor{green!8}\textbf{0.14~{\color{green!50!black}\ensuremath{\blacktriangledown}}} & 39.18 & \cellcolor{green!8}\textbf{37.26~{\color{green!50!black}\ensuremath{\blacktriangledown}}} & 2.62 & \cellcolor{green!8}\textbf{0.55~{\color{green!50!black}\ensuremath{\blacktriangledown}}} & 2.30 & \cellcolor{green!8}\textbf{0.61~{\color{green!50!black}\ensuremath{\blacktriangledown}}} & 6.46 & \cellcolor{green!8}\textbf{1.08~{\color{green!50!black}\ensuremath{\blacktriangledown}}} & 3.46 & \cellcolor{green!8}\textbf{0.90~{\color{green!50!black}\ensuremath{\blacktriangledown}}} \\
& NLL & 1.5343 & 1.2323 & \cellcolor{green!8}\textbf{0.0006~{\color{green!50!black}\ensuremath{\blacktriangledown}}} & 1.2381 & \cellcolor{red!8}\textbf{-0.0003~{\color{red!75!black}\ensuremath{\blacktriangle}}} & 6.9181 & \cellcolor{green!8}\textbf{2.4006~{\color{green!50!black}\ensuremath{\blacktriangledown}}} & 1.2254 & \cellcolor{red!8}\textbf{-0.0011~{\color{red!75!black}\ensuremath{\blacktriangle}}} & 1.2214 & \cellcolor{green!8}\textbf{0.0020~{\color{green!50!black}\ensuremath{\blacktriangledown}}} & 1.7049 & \cellcolor{green!8}\textbf{0.0118~{\color{green!50!black}\ensuremath{\blacktriangledown}}} & 1.6676 & \cellcolor{green!8}\textbf{0.0070~{\color{green!50!black}\ensuremath{\blacktriangledown}}} \\
& Brier & 0.4801 & 0.4342 & \cellcolor{green!8}\textbf{0.0005~{\color{green!50!black}\ensuremath{\blacktriangledown}}} & 0.4367 & \cellcolor{red!8}\textbf{-0.0001~{\color{red!75!black}\ensuremath{\blacktriangle}}} & 0.8248 & \cellcolor{green!8}\textbf{0.2684~{\color{green!50!black}\ensuremath{\blacktriangledown}}} & 0.4323 & \cellcolor{red!8}\textbf{-0.0005~{\color{red!75!black}\ensuremath{\blacktriangle}}} & 0.4338 & \cellcolor{green!8}\textbf{0.0010~{\color{green!50!black}\ensuremath{\blacktriangledown}}} & 0.4448 & \cellcolor{green!8}\textbf{0.0026~{\color{green!50!black}\ensuremath{\blacktriangledown}}} & 0.4399 & \cellcolor{green!8}\textbf{0.0015~{\color{green!50!black}\ensuremath{\blacktriangledown}}} \\

\cmidrule(lr){1-17}
\multirow{5}{*}{WRN-26-10}
& Acc. (\%) & 79.40 & -0.15 & \cellcolor{gray!12}\textbf{0.00} & +0.04 & \cellcolor{gray!12}\textbf{0.00} & -9.13 & \cellcolor{gray!12}\textbf{0.00} & -0.06 & \cellcolor{gray!12}\textbf{0.00} & -0.20 & \cellcolor{gray!12}\textbf{0.00} & -0.22 & \cellcolor{gray!12}\textbf{0.00} & -0.48 & \cellcolor{gray!12}\textbf{0.00} \\
& TPCR (\%) & 0.00 & 3.14 & \cellcolor{gray!12}\textbf{0.00} & 0.62 & \cellcolor{gray!12}\textbf{0.00} & 24.02 & \cellcolor{gray!12}\textbf{0.00} & 1.45 & \cellcolor{gray!12}\textbf{0.00} & 5.06 & \cellcolor{gray!12}\textbf{0.00} & 4.27 & \cellcolor{gray!12}\textbf{0.00} & 5.14 & \cellcolor{gray!12}\textbf{0.00} \\
& ECE (\%) & 15.31 & 5.02 & \cellcolor{green!8}\textbf{0.52~{\color{green!50!black}\ensuremath{\blacktriangledown}}} & 4.68 & \cellcolor{green!8}\textbf{0.12~{\color{green!50!black}\ensuremath{\blacktriangledown}}} & 26.99 & \cellcolor{green!8}\textbf{25.12~{\color{green!50!black}\ensuremath{\blacktriangledown}}} & 4.41 & \cellcolor{green!8}\textbf{0.21~{\color{green!50!black}\ensuremath{\blacktriangledown}}} & 3.74 & \cellcolor{green!8}\textbf{0.80~{\color{green!50!black}\ensuremath{\blacktriangledown}}} & 6.42 & \cellcolor{green!8}\textbf{0.51~{\color{green!50!black}\ensuremath{\blacktriangledown}}} & 5.22 & \cellcolor{green!8}\textbf{0.66~{\color{green!50!black}\ensuremath{\blacktriangledown}}} \\
& NLL & 1.4049 & 0.9629 & \cellcolor{green!8}\textbf{0.0048~{\color{green!50!black}\ensuremath{\blacktriangledown}}} & 0.9494 & \cellcolor{green!8}\textbf{0.0002~{\color{green!50!black}\ensuremath{\blacktriangledown}}} & 5.7379 & \cellcolor{green!8}\textbf{2.4294~{\color{green!50!black}\ensuremath{\blacktriangledown}}} & 0.9478 & \cellcolor{green!8}\textbf{0.0010~{\color{green!50!black}\ensuremath{\blacktriangledown}}} & 0.8947 & \cellcolor{green!8}\textbf{0.0054~{\color{green!50!black}\ensuremath{\blacktriangledown}}} & 1.3913 & \cellcolor{green!8}\textbf{0.0125~{\color{green!50!black}\ensuremath{\blacktriangledown}}} & 1.2901 & \cellcolor{green!8}\textbf{0.0126~{\color{green!50!black}\ensuremath{\blacktriangledown}}} \\
& Brier & 0.3511 & 0.3031 & \cellcolor{green!8}\textbf{0.0016~{\color{green!50!black}\ensuremath{\blacktriangledown}}} & 0.3026 & \cellcolor{gray!12}\textbf{0.0000} & 0.5617 & \cellcolor{green!8}\textbf{0.1824~{\color{green!50!black}\ensuremath{\blacktriangledown}}} & 0.3016 & \cellcolor{green!8}\textbf{0.0003~{\color{green!50!black}\ensuremath{\blacktriangledown}}} & 0.3005 & \cellcolor{green!8}\textbf{0.0022~{\color{green!50!black}\ensuremath{\blacktriangledown}}} & 0.3075 & \cellcolor{green!8}\textbf{0.0019~{\color{green!50!black}\ensuremath{\blacktriangledown}}} & 0.2998 & \cellcolor{green!8}\textbf{0.0021~{\color{green!50!black}\ensuremath{\blacktriangledown}}} \\

\cmidrule(lr){1-17}
\multirow{5}{*}{DenseNet-121}
& Acc. (\%) & 76.12 & +0.29 & \cellcolor{gray!12}\textbf{0.00} & +0.04 & \cellcolor{gray!12}\textbf{0.00} & -9.08 & \cellcolor{gray!12}\textbf{0.00} & +0.20 & \cellcolor{gray!12}\textbf{0.00} & -0.06 & \cellcolor{gray!12}\textbf{0.00} & -0.04 & \cellcolor{gray!12}\textbf{0.00} & -0.15 & \cellcolor{gray!12}\textbf{0.00} \\
& TPCR (\%) & 0.00 & 2.93 & \cellcolor{gray!12}\textbf{0.00} & 0.40 & \cellcolor{gray!12}\textbf{0.00} & 25.88 & \cellcolor{gray!12}\textbf{0.00} & 1.47 & \cellcolor{gray!12}\textbf{0.00} & 4.67 & \cellcolor{gray!12}\textbf{0.00} & 4.16 & \cellcolor{gray!12}\textbf{0.00} & 5.56 & \cellcolor{gray!12}\textbf{0.00} \\
& ECE (\%) & 20.29 & 5.23 & \cellcolor{green!8}\textbf{0.52~{\color{green!50!black}\ensuremath{\blacktriangledown}}} & 4.98 & \cellcolor{green!8}\textbf{0.04~{\color{green!50!black}\ensuremath{\blacktriangledown}}} & 30.97 & \cellcolor{green!8}\textbf{28.61~{\color{green!50!black}\ensuremath{\blacktriangledown}}} & 4.72 & \cellcolor{green!8}\textbf{0.13~{\color{green!50!black}\ensuremath{\blacktriangledown}}} & 4.31 & \cellcolor{green!8}\textbf{0.58~{\color{green!50!black}\ensuremath{\blacktriangledown}}} & 8.09 & \cellcolor{green!8}\textbf{0.68~{\color{green!50!black}\ensuremath{\blacktriangledown}}} & 6.61 & \cellcolor{green!8}\textbf{0.60~{\color{green!50!black}\ensuremath{\blacktriangledown}}} \\
& NLL & 2.0066 & 1.1468 & \cellcolor{green!8}\textbf{0.0017~{\color{green!50!black}\ensuremath{\blacktriangledown}}} & 1.1554 & \cellcolor{gray!12}\textbf{0.0000} & 6.8986 & \cellcolor{green!8}\textbf{3.0624~{\color{green!50!black}\ensuremath{\blacktriangledown}}} & 1.1422 & \cellcolor{green!8}\textbf{0.0001~{\color{green!50!black}\ensuremath{\blacktriangledown}}} & 1.0452 & \cellcolor{green!8}\textbf{0.0047~{\color{green!50!black}\ensuremath{\blacktriangledown}}} & 1.5937 & \cellcolor{green!8}\textbf{0.0166~{\color{green!50!black}\ensuremath{\blacktriangledown}}} & 1.4711 & \cellcolor{green!8}\textbf{0.0127~{\color{green!50!black}\ensuremath{\blacktriangledown}}} \\
& Brier & 0.4329 & 0.3561 & \cellcolor{green!8}\textbf{0.0001~{\color{green!50!black}\ensuremath{\blacktriangledown}}} & 0.3624 & \cellcolor{gray!12}\textbf{0.0000} & 0.6318 & \cellcolor{green!8}\textbf{0.2044~{\color{green!50!black}\ensuremath{\blacktriangledown}}} & 0.3579 & \cellcolor{red!8}\textbf{-0.0001~{\color{red!75!black}\ensuremath{\blacktriangle}}} & 0.3503 & \cellcolor{green!8}\textbf{0.0013~{\color{green!50!black}\ensuremath{\blacktriangledown}}} & 0.3636 & \cellcolor{green!8}\textbf{0.0021~{\color{green!50!black}\ensuremath{\blacktriangledown}}} & 0.3530 & \cellcolor{green!8}\textbf{0.0018~{\color{green!50!black}\ensuremath{\blacktriangledown}}} \\

\cmidrule(lr){1-17}
\multirow{5}{*}{MobileNetV2-1.4$\times$}
& Acc. (\%) & 71.26 & -0.02 & \cellcolor{gray!12}\textbf{0.00} & +0.10 & \cellcolor{gray!12}\textbf{0.00} & -12.73 & \cellcolor{gray!12}\textbf{0.00} & +0.17 & \cellcolor{gray!12}\textbf{0.00} & +0.04 & \cellcolor{gray!12}\textbf{0.00} & -0.37 & \cellcolor{gray!12}\textbf{0.00} & -0.36 & \cellcolor{gray!12}\textbf{0.00} \\
& TPCR (\%) & 0.00 & 7.88 & \cellcolor{gray!12}\textbf{0.00} & 1.65 & \cellcolor{gray!12}\textbf{0.00} & 35.95 & \cellcolor{gray!12}\textbf{0.00} & 5.06 & \cellcolor{gray!12}\textbf{0.00} & 9.47 & \cellcolor{gray!12}\textbf{0.00} & 7.49 & \cellcolor{gray!12}\textbf{0.00} & 8.61 & \cellcolor{gray!12}\textbf{0.00} \\
& ECE (\%) & 11.42 & 3.60 & \cellcolor{green!8}\textbf{0.53~{\color{green!50!black}\ensuremath{\blacktriangledown}}} & 3.14 & \cellcolor{green!8}\textbf{0.02~{\color{green!50!black}\ensuremath{\blacktriangledown}}} & 40.81 & \cellcolor{green!8}\textbf{36.77~{\color{green!50!black}\ensuremath{\blacktriangledown}}} & 2.66 & \cellcolor{green!8}\textbf{0.04~{\color{green!50!black}\ensuremath{\blacktriangledown}}} & 2.81 & \cellcolor{green!8}\textbf{0.51~{\color{green!50!black}\ensuremath{\blacktriangledown}}} & 5.48 & \cellcolor{green!8}\textbf{0.49~{\color{green!50!black}\ensuremath{\blacktriangledown}}} & 3.45 & \cellcolor{green!8}\textbf{0.34~{\color{green!50!black}\ensuremath{\blacktriangledown}}} \\
& NLL & 1.2337 & 1.1040 & \cellcolor{green!8}\textbf{0.0021~{\color{green!50!black}\ensuremath{\blacktriangledown}}} & 1.1068 & \cellcolor{gray!12}\textbf{0.0000} & 10.6075 & \cellcolor{green!8}\textbf{5.7217~{\color{green!50!black}\ensuremath{\blacktriangledown}}} & 1.0964 & \cellcolor{red!8}\textbf{-0.0001~{\color{red!75!black}\ensuremath{\blacktriangle}}} & 1.0901 & \cellcolor{green!8}\textbf{0.0039~{\color{green!50!black}\ensuremath{\blacktriangledown}}} & 1.5513 & \cellcolor{green!8}\textbf{0.0110~{\color{green!50!black}\ensuremath{\blacktriangledown}}} & 1.5129 & \cellcolor{green!8}\textbf{0.0091~{\color{green!50!black}\ensuremath{\blacktriangledown}}} \\
& Brier & 0.4177 & 0.3921 & \cellcolor{green!8}\textbf{0.0012~{\color{green!50!black}\ensuremath{\blacktriangledown}}} & 0.3944 & \cellcolor{gray!12}\textbf{0.0000} & 0.8195 & \cellcolor{green!8}\textbf{0.3119~{\color{green!50!black}\ensuremath{\blacktriangledown}}} & 0.3902 & \cellcolor{red!8}\textbf{-0.0002~{\color{red!75!black}\ensuremath{\blacktriangle}}} & 0.3906 & \cellcolor{green!8}\textbf{0.0018~{\color{green!50!black}\ensuremath{\blacktriangledown}}} & 0.3982 & \cellcolor{green!8}\textbf{0.0018~{\color{green!50!black}\ensuremath{\blacktriangledown}}} & 0.3944 & \cellcolor{green!8}\textbf{0.0010~{\color{green!50!black}\ensuremath{\blacktriangledown}}} \\

\cmidrule(lr){1-17}
\multirow{5}{*}{ShuffleNetV2-1.0$\times$}
& Acc. (\%) & 66.82 & +0.34 & \cellcolor{gray!12}\textbf{0.00} & +0.12 & \cellcolor{gray!12}\textbf{0.00} & -13.16 & \cellcolor{gray!12}\textbf{0.00} & +0.46 & \cellcolor{gray!12}\textbf{0.00} & +0.51 & \cellcolor{gray!12}\textbf{0.00} & -0.17 & \cellcolor{gray!12}\textbf{0.00} & -0.20 & \cellcolor{gray!12}\textbf{0.00} \\
& TPCR (\%) & 0.00 & 9.55 & \cellcolor{gray!12}\textbf{0.00} & 2.25 & \cellcolor{gray!12}\textbf{0.00} & 40.51 & \cellcolor{gray!12}\textbf{0.00} & 6.34 & \cellcolor{gray!12}\textbf{0.00} & 11.39 & \cellcolor{gray!12}\textbf{0.00} & 9.59 & \cellcolor{gray!12}\textbf{0.00} & 10.03 & \cellcolor{gray!12}\textbf{0.00} \\
& ECE (\%) & 11.83 & 4.37 & \cellcolor{green!8}\textbf{0.62~{\color{green!50!black}\ensuremath{\blacktriangledown}}} & 3.95 & \cellcolor{green!8}\textbf{0.08~{\color{green!50!black}\ensuremath{\blacktriangledown}}} & 45.54 & \cellcolor{green!8}\textbf{38.92~{\color{green!50!black}\ensuremath{\blacktriangledown}}} & 3.53 & \cellcolor{green!8}\textbf{0.21~{\color{green!50!black}\ensuremath{\blacktriangledown}}} & 3.39 & \cellcolor{green!8}\textbf{0.79~{\color{green!50!black}\ensuremath{\blacktriangledown}}} & 6.11 & \cellcolor{green!8}\textbf{0.74~{\color{green!50!black}\ensuremath{\blacktriangledown}}} & 3.89 & \cellcolor{green!8}\textbf{0.31~{\color{green!50!black}\ensuremath{\blacktriangledown}}} \\
& NLL & 1.4419 & 1.3073 & \cellcolor{green!8}\textbf{0.0010~{\color{green!50!black}\ensuremath{\blacktriangledown}}} & 1.3148 & \cellcolor{red!8}\textbf{-0.0001~{\color{red!75!black}\ensuremath{\blacktriangle}}} & 11.8949 & \cellcolor{green!8}\textbf{6.0656~{\color{green!50!black}\ensuremath{\blacktriangledown}}} & 1.3015 & \cellcolor{red!8}\textbf{-0.0006~{\color{red!75!black}\ensuremath{\blacktriangle}}} & 1.2856 & \cellcolor{green!8}\textbf{0.0038~{\color{green!50!black}\ensuremath{\blacktriangledown}}} & 1.7878 & \cellcolor{green!8}\textbf{0.0123~{\color{green!50!black}\ensuremath{\blacktriangledown}}} & 1.7594 & \cellcolor{green!8}\textbf{0.0089~{\color{green!50!black}\ensuremath{\blacktriangledown}}} \\
& Brier & 0.4716 & 0.4474 & \cellcolor{green!8}\textbf{0.0003~{\color{green!50!black}\ensuremath{\blacktriangledown}}} & 0.4502 & \cellcolor{gray!12}\textbf{0.0000} & 0.9153 & \cellcolor{green!8}\textbf{0.3384~{\color{green!50!black}\ensuremath{\blacktriangledown}}} & 0.4453 & \cellcolor{red!8}\textbf{-0.0005~{\color{red!75!black}\ensuremath{\blacktriangle}}} & 0.4447 & \cellcolor{green!8}\textbf{0.0012~{\color{green!50!black}\ensuremath{\blacktriangledown}}} & 0.4549 & \cellcolor{green!8}\textbf{0.0021~{\color{green!50!black}\ensuremath{\blacktriangledown}}} & 0.4522 & \cellcolor{green!8}\textbf{0.0013~{\color{green!50!black}\ensuremath{\blacktriangledown}}} \\

\cmidrule(lr){1-17}
\multirow{5}{*}{RepVGG-A1}
& Acc. (\%) & 70.59 & +0.46 & \cellcolor{gray!12}\textbf{0.00} & +0.14 & \cellcolor{gray!12}\textbf{0.00} & -11.08 & \cellcolor{gray!12}\textbf{0.00} & +0.34 & \cellcolor{gray!12}\textbf{0.00} & +0.36 & \cellcolor{gray!12}\textbf{0.00} & +0.08 & \cellcolor{gray!12}\textbf{0.00} & -0.16 & \cellcolor{gray!12}\textbf{0.00} \\
& TPCR (\%) & 0.00 & 8.17 & \cellcolor{gray!12}\textbf{0.00} & 1.89 & \cellcolor{gray!12}\textbf{0.00} & 34.50 & \cellcolor{gray!12}\textbf{0.00} & 5.02 & \cellcolor{gray!12}\textbf{0.00} & 10.06 & \cellcolor{gray!12}\textbf{0.00} & 8.78 & \cellcolor{gray!12}\textbf{0.00} & 9.36 & \cellcolor{gray!12}\textbf{0.00} \\
& ECE (\%) & 8.11 & 5.16 & \cellcolor{green!8}\textbf{0.71~{\color{green!50!black}\ensuremath{\blacktriangledown}}} & 4.50 & \cellcolor{green!8}\textbf{0.12~{\color{green!50!black}\ensuremath{\blacktriangledown}}} & 39.62 & \cellcolor{green!8}\textbf{33.97~{\color{green!50!black}\ensuremath{\blacktriangledown}}} & 4.35 & \cellcolor{green!8}\textbf{0.38~{\color{green!50!black}\ensuremath{\blacktriangledown}}} & 4.24 & \cellcolor{green!8}\textbf{0.93~{\color{green!50!black}\ensuremath{\blacktriangledown}}} & 5.67 & \cellcolor{green!8}\textbf{0.65~{\color{green!50!black}\ensuremath{\blacktriangledown}}} & 4.46 & \cellcolor{green!8}\textbf{0.47~{\color{green!50!black}\ensuremath{\blacktriangledown}}} \\
& NLL & 1.2381 & 1.1934 & \cellcolor{green!8}\textbf{0.0027~{\color{green!50!black}\ensuremath{\blacktriangledown}}} & 1.1999 & \cellcolor{red!8}\textbf{-0.0001~{\color{red!75!black}\ensuremath{\blacktriangle}}} & 10.2162 & \cellcolor{green!8}\textbf{5.1115~{\color{green!50!black}\ensuremath{\blacktriangledown}}} & 1.1873 & \cellcolor{green!8}\textbf{0.0002~{\color{green!50!black}\ensuremath{\blacktriangledown}}} & 1.1496 & \cellcolor{green!8}\textbf{0.0025~{\color{green!50!black}\ensuremath{\blacktriangledown}}} & 1.5807 & \cellcolor{green!8}\textbf{0.0161~{\color{green!50!black}\ensuremath{\blacktriangledown}}} & 1.5541 & \cellcolor{green!8}\textbf{0.0147~{\color{green!50!black}\ensuremath{\blacktriangledown}}} \\
& Brier & 0.4106 & 0.3997 & \cellcolor{green!8}\textbf{0.0012~{\color{green!50!black}\ensuremath{\blacktriangledown}}} & 0.4026 & \cellcolor{red!8}\textbf{-0.0001~{\color{red!75!black}\ensuremath{\blacktriangle}}} & 0.7976 & \cellcolor{green!8}\textbf{0.2846~{\color{green!50!black}\ensuremath{\blacktriangledown}}} & 0.3986 & \cellcolor{gray!12}\textbf{0.0000} & 0.3972 & \cellcolor{green!8}\textbf{0.0014~{\color{green!50!black}\ensuremath{\blacktriangledown}}} & 0.4032 & \cellcolor{green!8}\textbf{0.0015~{\color{green!50!black}\ensuremath{\blacktriangledown}}} & 0.4005 & \cellcolor{green!8}\textbf{0.0011~{\color{green!50!black}\ensuremath{\blacktriangledown}}} \\
\bottomrule
\end{tabular}%
}
\vspace{-0.3cm}
\caption{Complete classifier--calibrator results on CIFAR-100 corresponding to Figure~2 of the main paper. Values are five-split means. Uncal. reports absolute values. Within each Base/+CORD pair, Acc. reports $\Delta\mathrm{Acc}$ from Uncal. in percentage points, TPCR reports absolute rates in percent, and $M\in\{\mathrm{ECE},\mathrm{NLL},\mathrm{Brier}\}$ is reported as $M(\mathrm{Base})/\Delta_{\mathrm R}M$, where $\Delta_{\mathrm R}M:=M(\mathrm{Base})-M(\mathrm{Base}+\mathrm{CORD})$; positive values indicate improvement.}
\label{tab:complete-results-cifar100}
\end{table*}

\begin{figure*}[t!]
    \centering
    \includegraphics[width=\linewidth]
    {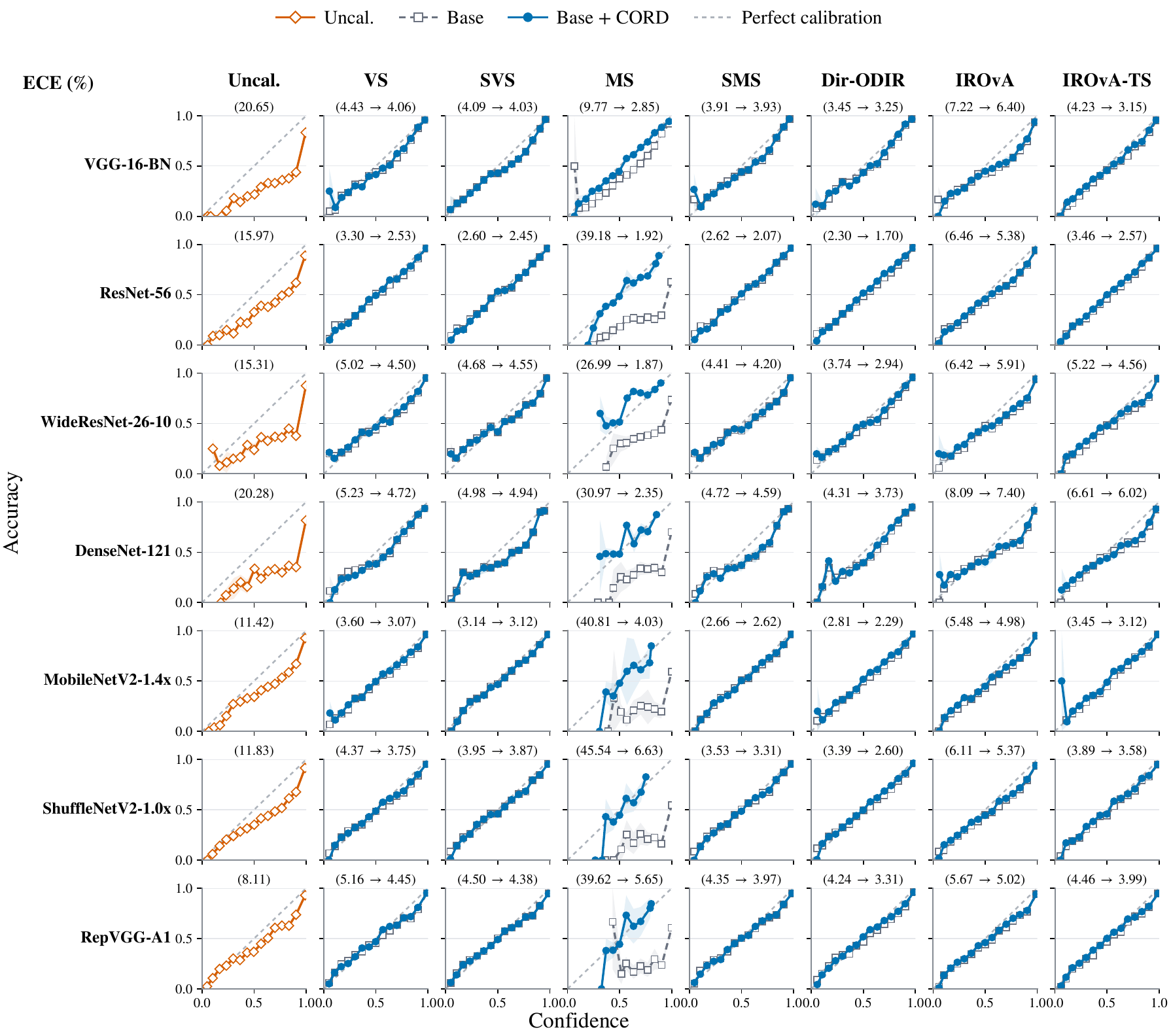}
    \caption{Reliability diagrams on CIFAR-100 corresponding to Table~\ref{tab:complete-results-cifar100}. Rows denote classifiers and columns denote Uncal. and Base calibrators; parenthetical values report ECE (\%) for Uncal. or Base $\rightarrow$ Base + CORD. Curves are averaged over five splits; shading shows pointwise 95\% CIs.}
    \label{fig:reliability-cifar100}
\end{figure*}

\clearpage
\subsection{ImageNet-1K}
\begin{table*}[h!]
\renewcommand{\arraystretch}{1.0}
\centering
\setlength{\tabcolsep}{3.0pt}
\resizebox{1.0\textwidth}{!}{%
\begin{tabular}{
    l
    !{\color{black!18}\vrule width 0.25pt}
    l
    !{\color{black!18}\vrule width 0.25pt}
    r
    !{\color{black!18}\vrule width 0.25pt}
    r
    !{\color{black!18}\vrule width 0.25pt}
    r
    !{\color{black!18}\vrule width 0.25pt}
    r
    !{\color{black!18}\vrule width 0.25pt}
    r
    !{\color{black!18}\vrule width 0.25pt}
    r
    !{\color{black!18}\vrule width 0.25pt}
    r
    !{\color{black!18}\vrule width 0.25pt}
    r
    !{\color{black!18}\vrule width 0.25pt}
    r
}
\toprule
\multirow{2}{*}{\textbf{Classifier}}
& \multirow{2}{*}{\textbf{Metric}}
& \multirow{2}{*}{\textbf{Uncal.}}
& \multicolumn{2}{c}{\textbf{VS}}
& \multicolumn{2}{c}{\textbf{SVS}}
& \multicolumn{2}{c}{\textbf{IROvA}}
& \multicolumn{2}{c}{\textbf{IROvA-TS}} \\
\cmidrule(lr){4-5}
\cmidrule(lr){6-7}
\cmidrule(lr){8-9}
\cmidrule(lr){10-11}
& & &
Base & \textbf{+ CORD}
& Base & \textbf{+ CORD}
& Base & \textbf{+ CORD}
& Base & \textbf{+ CORD} \\
\midrule

\multirow{5}{*}{ResNet-50}
& Acc. (\%) & 80.81 & -0.24 & \cellcolor{gray!12}\textbf{0.00} & 0.00 & \cellcolor{gray!12}\textbf{0.00} & -0.62 & \cellcolor{gray!12}\textbf{0.00} & -0.49 & \cellcolor{gray!12}\textbf{0.00} \\
& TPCR (\%) & 0.00 & 5.74 & \cellcolor{gray!12}\textbf{0.00} & 0.07 & \cellcolor{gray!12}\textbf{0.00} & 7.02 & \cellcolor{gray!12}\textbf{0.00} & 6.53 & \cellcolor{gray!12}\textbf{0.00} \\
& ECE (\%) & 41.15 & 4.16 & \cellcolor{green!8}\textbf{1.17~{\color{green!50!black}\ensuremath{\blacktriangledown}}} & 3.21 & \cellcolor{gray!12}\textbf{0.00} & 1.36 & \cellcolor{green!8}\textbf{0.06~{\color{green!50!black}\ensuremath{\blacktriangledown}}} & 3.72 & \cellcolor{green!8}\textbf{1.12~{\color{green!50!black}\ensuremath{\blacktriangledown}}} \\
& NLL & 1.3980 & 0.8002 & \cellcolor{green!8}\textbf{0.0048~{\color{green!50!black}\ensuremath{\blacktriangledown}}} & 0.7712 & \cellcolor{gray!12}\textbf{0.0000} & 1.4339 & \cellcolor{green!8}\textbf{0.0339~{\color{green!50!black}\ensuremath{\blacktriangledown}}} & 1.4560 & \cellcolor{green!8}\textbf{0.0253~{\color{green!50!black}\ensuremath{\blacktriangledown}}} \\
& Brier & 0.4711 & 0.2826 & \cellcolor{green!8}\textbf{0.0028~{\color{green!50!black}\ensuremath{\blacktriangledown}}} & 0.2773 & \cellcolor{gray!12}\textbf{0.0000} & 0.2830 & \cellcolor{green!8}\textbf{0.0023~{\color{green!50!black}\ensuremath{\blacktriangledown}}} & 0.2830 & \cellcolor{green!8}\textbf{0.0030~{\color{green!50!black}\ensuremath{\blacktriangledown}}} \\

\cmidrule(lr){1-11}
\multirow{5}{*}{ViT-B/16}
& Acc. (\%) & 81.01 & -0.29 & \cellcolor{gray!12}\textbf{0.00} & -0.01 & \cellcolor{gray!12}\textbf{0.00} & -0.60 & \cellcolor{gray!12}\textbf{0.00} & -0.58 & \cellcolor{gray!12}\textbf{0.00} \\
& TPCR (\%) & 0.00 & 5.13 & \cellcolor{gray!12}\textbf{0.00} & 0.06 & \cellcolor{gray!12}\textbf{0.00} & 6.20 & \cellcolor{gray!12}\textbf{0.00} & 6.16 & \cellcolor{gray!12}\textbf{0.00} \\
& ECE (\%) & 5.61 & 4.76 & \cellcolor{green!8}\textbf{1.34~{\color{green!50!black}\ensuremath{\blacktriangledown}}} & 3.82 & \cellcolor{green!8}\textbf{0.01~{\color{green!50!black}\ensuremath{\blacktriangledown}}} & 3.70 & \cellcolor{green!8}\textbf{1.36~{\color{green!50!black}\ensuremath{\blacktriangledown}}} & 4.00 & \cellcolor{green!8}\textbf{1.30~{\color{green!50!black}\ensuremath{\blacktriangledown}}} \\
& NLL & 0.8412 & 0.8227 & \cellcolor{green!8}\textbf{0.0048~{\color{green!50!black}\ensuremath{\blacktriangledown}}} & 0.7897 & \cellcolor{gray!12}\textbf{0.0000} & 1.4668 & \cellcolor{green!8}\textbf{0.0264~{\color{green!50!black}\ensuremath{\blacktriangledown}}} & 1.4742 & \cellcolor{green!8}\textbf{0.0265~{\color{green!50!black}\ensuremath{\blacktriangledown}}} \\
& Brier & 0.2773 & 0.2824 & \cellcolor{green!8}\textbf{0.0032~{\color{green!50!black}\ensuremath{\blacktriangledown}}} & 0.2766 & \cellcolor{gray!12}\textbf{0.0000} & 0.2807 & \cellcolor{green!8}\textbf{0.0032~{\color{green!50!black}\ensuremath{\blacktriangledown}}} & 0.2811 & \cellcolor{green!8}\textbf{0.0034~{\color{green!50!black}\ensuremath{\blacktriangledown}}} \\

\cmidrule(lr){1-11}
\multirow{5}{*}{Swin-T}
& Acc. (\%) & 81.49 & -0.16 & \cellcolor{gray!12}\textbf{0.00} & 0.00 & \cellcolor{gray!12}\textbf{0.00} & -0.45 & \cellcolor{gray!12}\textbf{0.00} & -0.40 & \cellcolor{gray!12}\textbf{0.00} \\
& TPCR (\%) & 0.00 & 5.59 & \cellcolor{gray!12}\textbf{0.00} & 0.09 & \cellcolor{gray!12}\textbf{0.00} & 6.41 & \cellcolor{gray!12}\textbf{0.00} & 6.27 & \cellcolor{gray!12}\textbf{0.00} \\
& ECE (\%) & 6.82 & 3.90 & \cellcolor{green!8}\textbf{1.23~{\color{green!50!black}\ensuremath{\blacktriangledown}}} & 2.99 & \cellcolor{green!8}\textbf{0.01~{\color{green!50!black}\ensuremath{\blacktriangledown}}} & 2.91 & \cellcolor{green!8}\textbf{1.17~{\color{green!50!black}\ensuremath{\blacktriangledown}}} & 3.34 & \cellcolor{green!8}\textbf{1.20~{\color{green!50!black}\ensuremath{\blacktriangledown}}} \\
& NLL & 0.7982 & 0.7626 & \cellcolor{green!8}\textbf{0.0042~{\color{green!50!black}\ensuremath{\blacktriangledown}}} & 0.7363 & \cellcolor{gray!12}\textbf{0.0000} & 1.4055 & \cellcolor{green!8}\textbf{0.0321~{\color{green!50!black}\ensuremath{\blacktriangledown}}} & 1.4153 & \cellcolor{green!8}\textbf{0.0307~{\color{green!50!black}\ensuremath{\blacktriangledown}}} \\
& Brier & 0.2719 & 0.2720 & \cellcolor{green!8}\textbf{0.0026~{\color{green!50!black}\ensuremath{\blacktriangledown}}} & 0.2674 & \cellcolor{gray!12}\textbf{0.0000} & 0.2729 & \cellcolor{green!8}\textbf{0.0027~{\color{green!50!black}\ensuremath{\blacktriangledown}}} & 0.2728 & \cellcolor{green!8}\textbf{0.0027~{\color{green!50!black}\ensuremath{\blacktriangledown}}} \\

\cmidrule(lr){1-11}
\multirow{5}{*}{ConvNeXt-T}
& Acc. (\%) & 82.53 & -0.25 & \cellcolor{gray!12}\textbf{0.00} & -0.01 & \cellcolor{gray!12}\textbf{0.00} & -0.58 & \cellcolor{gray!12}\textbf{0.00} & -0.62 & \cellcolor{gray!12}\textbf{0.00} \\
& TPCR (\%) & 0.00 & 5.27 & \cellcolor{gray!12}\textbf{0.00} & 0.07 & \cellcolor{gray!12}\textbf{0.00} & 5.92 & \cellcolor{gray!12}\textbf{0.00} & 5.74 & \cellcolor{gray!12}\textbf{0.00} \\
& ECE (\%) & 16.93 & 4.04 & \cellcolor{green!8}\textbf{1.20~{\color{green!50!black}\ensuremath{\blacktriangledown}}} & 3.06 & \cellcolor{green!8}\textbf{0.01~{\color{green!50!black}\ensuremath{\blacktriangledown}}} & 2.72 & \cellcolor{green!8}\textbf{1.15~{\color{green!50!black}\ensuremath{\blacktriangledown}}} & 3.67 & \cellcolor{green!8}\textbf{1.24~{\color{green!50!black}\ensuremath{\blacktriangledown}}} \\
& NLL & 0.8779 & 0.7268 & \cellcolor{green!8}\textbf{0.0054~{\color{green!50!black}\ensuremath{\blacktriangledown}}} & 0.7007 & \cellcolor{gray!12}\textbf{0.0000} & 1.3463 & \cellcolor{green!8}\textbf{0.0251~{\color{green!50!black}\ensuremath{\blacktriangledown}}} & 1.3790 & \cellcolor{green!8}\textbf{0.0281~{\color{green!50!black}\ensuremath{\blacktriangledown}}} \\
& Brier & 0.2935 & 0.2616 & \cellcolor{green!8}\textbf{0.0030~{\color{green!50!black}\ensuremath{\blacktriangledown}}} & 0.2572 & \cellcolor{gray!12}\textbf{0.0000} & 0.2608 & \cellcolor{green!8}\textbf{0.0025~{\color{green!50!black}\ensuremath{\blacktriangledown}}} & 0.2619 & \cellcolor{green!8}\textbf{0.0031~{\color{green!50!black}\ensuremath{\blacktriangledown}}} \\
\bottomrule
\end{tabular}%
}
\caption{Complete classifier--calibrator results on ImageNet-1K corresponding to Figure~2 of the main paper. Values are five-split means. Uncal. reports absolute values. Within each Base/+CORD pair, Acc. reports $\Delta\mathrm{Acc}$ from Uncal. in percentage points, TPCR reports absolute rates in percent, and $M\in\{\mathrm{ECE},\mathrm{NLL},\mathrm{Brier}\}$ is reported as $M(\mathrm{Base})/\Delta_{\mathrm R}M$, where $\Delta_{\mathrm R}M:=M(\mathrm{Base})-M(\mathrm{Base}+\mathrm{CORD})$; positive values indicate improvement.}
\label{tab:complete-results-imagenet}
\end{table*}

\begin{figure*}[h!]
    \centering
    \includegraphics[width=1.0\linewidth]
    {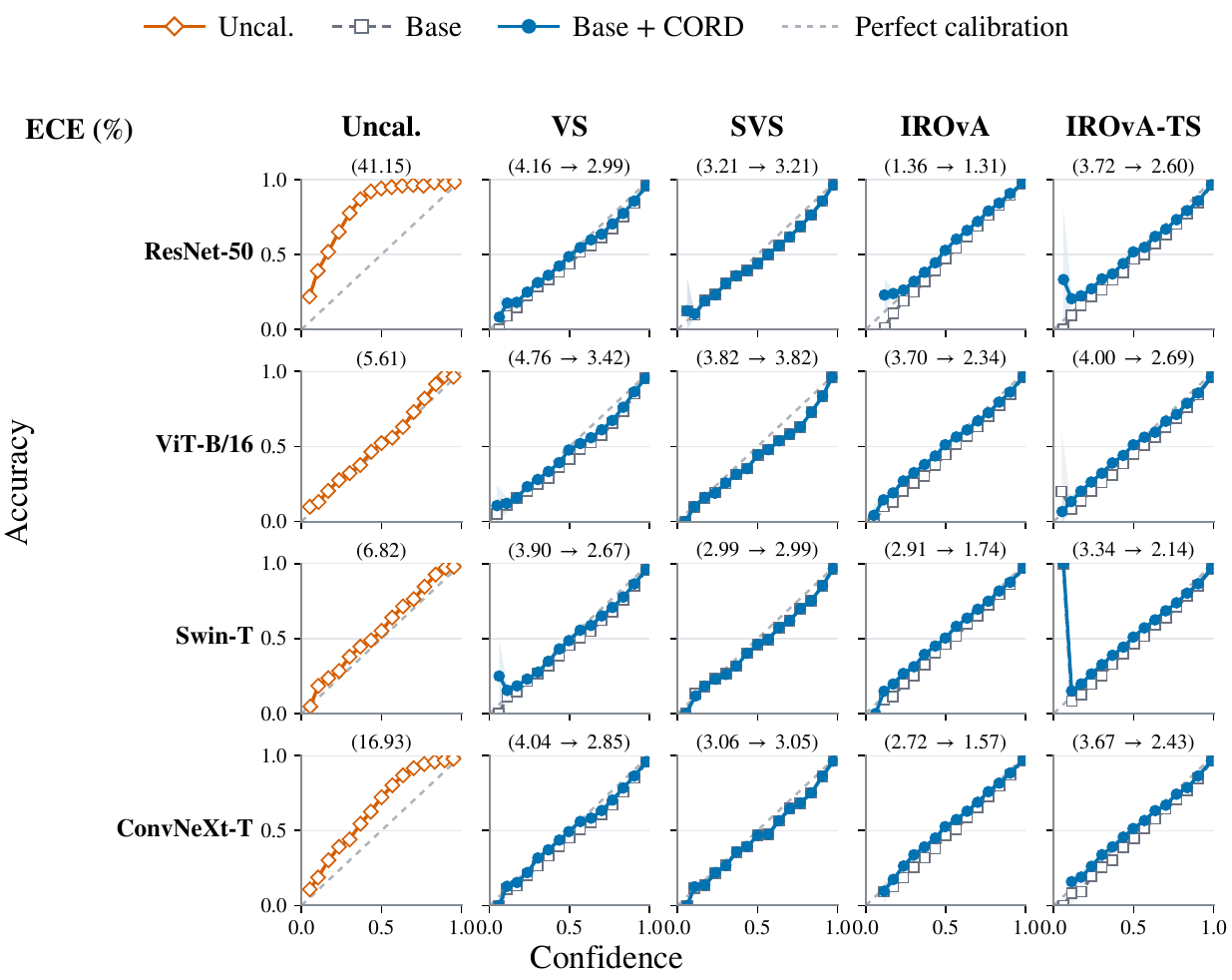}
    \caption{Reliability diagrams on ImageNet-1K corresponding to
    Table~\ref{tab:complete-results-imagenet}. Rows denote classifiers
    and columns denote Uncal. and Base calibrators; parenthetical
    values report ECE (\%) for Uncal. or Base $\rightarrow$
    Base + CORD. Curves are averaged over five splits; shading shows
    pointwise 95\% CIs.}
    \label{fig:reliability-imagenet1k}
\end{figure*}

\clearpage

\vspace{-0.3cm}

\section{Mean-Target Ablation on CIFAR-10 and ImageNet-1K}
\label{app:mean-target-additional}
The additional datasets preserve the overall pattern observed on CIFAR-100 in Table~3 of the main paper. CORD attains the lowest ECE and NLL and the smallest held-out mean discrepancy on both datasets, while matching the lowest Brier value at the reported precision. The inherited feasible mean thus most closely retains the corresponding Base mean beyond the calibration split, including in the low-TPCR setting of CIFAR-10.

\begin{table}[h!]
\renewcommand{\arraystretch}{1.0}
\centering
\resizebox{1.0\linewidth}{!}{%
    \begin{tabular}{
        l
        !{\color{black!18}\vrule width 0.25pt}
        c
        !{\color{black!18}\vrule width 0.25pt}
        c
        !{\color{black!18}\vrule width 0.25pt}
        c
        !{\color{black!18}\vrule width 0.25pt}
        c
        !{\color{black!18}\vrule width 0.25pt}
        c
    }
        \toprule
        \textbf{Dataset}
        & \textbf{Mean target}
        & \textbf{ECE (\%)}
        & \textbf{NLL}
        & \textbf{Brier}
        & \shortstack{
            \textbf{$\lvert\bar{s}_{\mathrm{eval}}
            -\bar{b}_{\mathrm{eval}}\rvert$}
            \textbf{(pp)}
        } \\
        \midrule

        \multirow{4}{*}{CIFAR-10}
        & None (independent)
        & 1.271
        & 0.2073
        & 0.0858
        & 0.104 \\

        & Local-reference
        & 1.270
        & 0.2073
        & 0.0858
        & 0.098 \\

        & Pointwise-projected Base
        & 1.257
        & 0.2072
        & 0.0857
        & 0.058 \\

        & \cellcolor{gray!12}\textbf{CORD}
        & \cellcolor{gray!12}\textbf{1.240}
        & \cellcolor{gray!12}\textbf{0.2071}
        & \cellcolor{gray!12}\textbf{0.0857}
        & \cellcolor{gray!12}\textbf{0.019} \\

        \midrule

        \multirow{4}{*}{ImageNet-1K}
        & None (independent)
        & 2.974
        & 1.0793
        & 0.2713
        & 0.569 \\

        & Local-reference
        & 2.725
        & 1.0780
        & 0.2712
        & 0.199 \\

        & Pointwise-projected Base
        & 2.958
        & 1.0792
        & 0.2713
        & 0.548 \\

        & \cellcolor{gray!12}\textbf{CORD}
        & \cellcolor{gray!12}\textbf{2.614}
        & \cellcolor{gray!12}\textbf{1.0775}
        & \cellcolor{gray!12}\textbf{0.2712}
        & \cellcolor{gray!12}\textbf{0.104} \\

        \bottomrule
    \end{tabular}%
}
\caption{Mean-target ablation on CIFAR-10 and ImageNet-1K, extending Table~3 of the main paper. Values are averaged over classifier--calibrator pairs and five splits.}
\label{tab:mean-target-additional}
\end{table}

\section{Robustness to CIFAR-C Corruptions}
\label{app:cifar-c-results}
Figure~\ref{fig:app-cifar-c-absolute} complements the paired reductions in Figure~3 of the main paper with absolute Base and Base + CORD results under corruption. Averaged over classifiers, calibrators, and 15 corruptions, CORD maintains zero TPCR and lower ECE, NLL, and Brier than Base at every severity on both datasets. The widening separation as corruption severity and direct-output TPCR increase extends the TPCR-dependent pattern observed on clean data to distribution shift.

\begin{figure}[h!]
    \centering
    \includegraphics[width=1.0\linewidth]
    {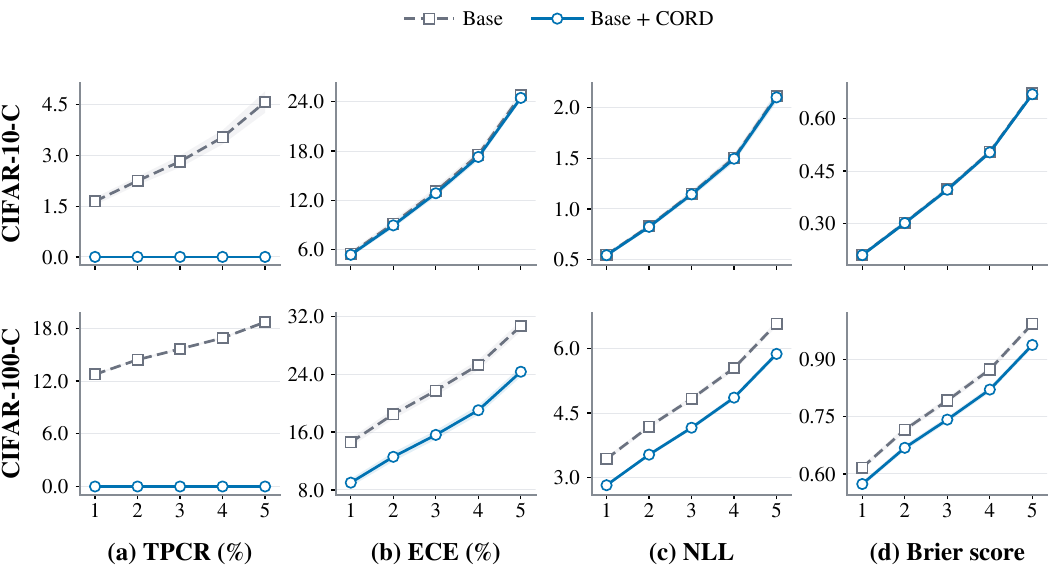}
    \caption{Absolute Base and Base + CORD results corresponding to Figure~3 of the main paper. Rows show CIFAR-10-C and CIFAR-100-C; curves average over classifiers, calibrators, 15 corruptions, and five splits.}
    \label{fig:app-cifar-c-absolute}
\end{figure}

\clearpage
\section{Calibration-Size Sensitivity on CIFAR-10/100}
\label{app:calibration-size-cifar}

Figure~\ref{fig:app-calibration-size-cifar} extends the calibration-size analysis in Figure~4 of the main paper to CIFAR-10/100. Across calibration-set fractions, CORD maintains zero TPCR by construction, and the mean paired reductions in ECE, NLL, and Brier persist on both datasets. The smaller paired changes on CIFAR-10 and larger gains on CIFAR-100 remain consistent with their respective direct-output TPCR levels.

\begin{figure}[h!]
    \centering
    \includegraphics[width=1.0\linewidth]
    {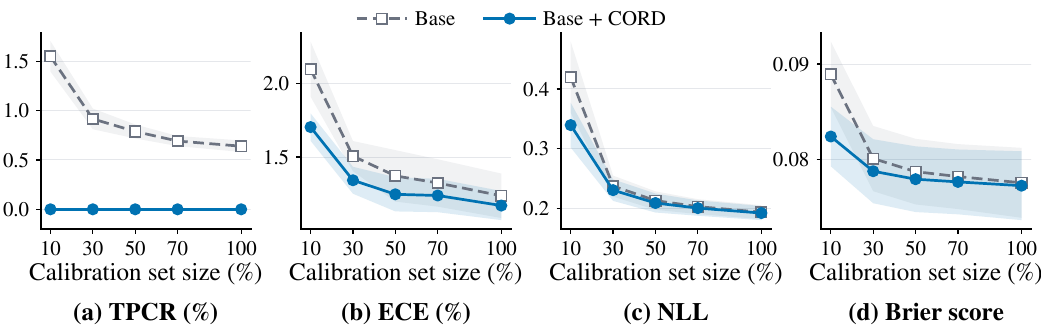}
    \includegraphics[width=1.0\linewidth]
    {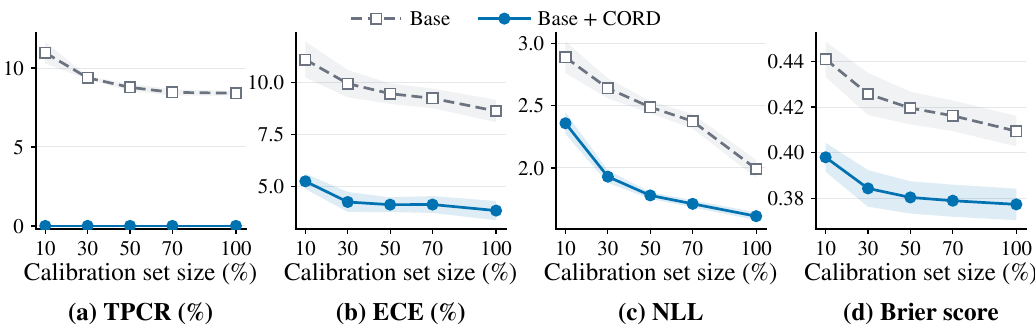}

    \caption{Calibration-size sensitivity on CIFAR-10 (top) and
    CIFAR-100 (bottom), extending Figure~4 of the main paper.
    Curves average over classifiers and calibrators; shading shows
    pointwise 95\% CIs across five splits.}
    \label{fig:app-calibration-size-cifar}
\end{figure}

\clearpage

\section{Additional Comparisons with Prediction-Preserving Methods}
\label{app:fit-time-post-fit}

These results extend Table~4 of the main paper. NLL and Brier are reported for the Base$^{\dagger}$ selected by standard ECE, while the complete CIFAR-10 comparisons report the Base outputs underlying the metric-specific selections. In both views, CORD is applied after fitting to the same direct outputs, leaving each fitted calibrator unchanged.

\subsection{NLL and Brier Results}
\label{app:fit-time-nll-brier}

\begin{table*}[h!]
\renewcommand{\arraystretch}{0.9}
\centering
\setlength{\tabcolsep}{3.0pt}

\newcommand{\metricsep}{%
    \hspace{1pt}{\color{black!45}/}\hspace{1pt}}
\newcommand{\improved}{%
    \,{\color{green!50!black}\ensuremath{\blacktriangledown}}}
\newcommand{\degraded}{%
    \,{\color{red!75!black}\ensuremath{\blacktriangle}}}

\resizebox{\linewidth}{!}{%
    \begin{tabular}{
        l
        !{\color{black!18}\vrule width 0.25pt}
        r@{\metricsep}r@{\metricsep}r
        !{\color{black!18}\vrule width 0.25pt}
        r@{\metricsep}r@{\metricsep}r
        !{\color{black!18}\vrule width 0.25pt}
        r@{\metricsep}r@{\metricsep}r
        !{\color{black!18}\vrule width 0.25pt}
        r@{\metricsep}r@{\metricsep}r
        !{\color{black!18}\vrule width 0.25pt}
        r@{\metricsep}r@{\metricsep}r
        !{\color{black!18}\vrule width 0.25pt}
        r@{\metricsep}r@{\metricsep}r
        !{\color{black!18}\vrule width 0.25pt}
        r@{\metricsep}r@{\metricsep}r
        !{\color{black!18}\vrule width 0.25pt}
        >{\columncolor{gray!12}}r
        @{\metricsep}
        >{\columncolor{gray!12}}r
        @{\metricsep}
        >{\columncolor{gray!12}}r
    }
        \toprule
        \textbf{Classifier}
        & \multicolumn{3}{
            c!{\color{black!18}\vrule width 0.25pt}
          }{\textbf{Uncal.}}
        & \multicolumn{3}{
            c!{\color{black!18}\vrule width 0.25pt}
          }{\textbf{TS}}
        & \multicolumn{3}{
            c!{\color{black!18}\vrule width 0.25pt}
          }{\textbf{IRM}}
        & \multicolumn{3}{
            c!{\color{black!18}\vrule width 0.25pt}
          }{\textbf{AdaTS}}
        & \multicolumn{3}{
            c!{\color{black!18}\vrule width 0.25pt}
          }{\textbf{TS--TvA}}
        & \multicolumn{3}{
            c!{\color{black!18}\vrule width 0.25pt}
          }{\textbf{MCCT-I}}
        & \multicolumn{3}{
            c!{\color{black!18}\vrule width 0.25pt}
          }{\textbf{Base$^{\dagger}$}}
        & \multicolumn{3}{
            >{\columncolor{gray!12}}c
          }{\textbf{Base$^{\dagger}$ + CORD}} \\
        \midrule

        VGG-16-BN
        & 4.79 & 0.3353 & 10.56
        & 1.54 & 0.2306 & 9.52
        & 1.62 & 0.2286 & \textbf{9.27}
        & 1.47 & 0.2282 & 9.28
        & 1.73 & 0.2306 & 9.47
        & 1.74 & \textbf{0.2209} & 9.46
        & 1.46 & 0.2243 & 9.41
        & \textbf{1.41}\improved
        & 0.2241\improved
        & 9.40\improved \\

        ResNet-56
        & 3.75 & 0.2525 & 9.39
        & 0.94 & 0.1910 & 8.65
        & 1.19 & 0.1917 & 8.60
        & 1.59 & 0.1951 & 8.65
        & 1.02 & 0.1912 & 8.63
        & 1.02 & 0.1955 & \textbf{8.59}
        & 0.94 & 0.1889 & 8.63
        & \textbf{0.92}\improved
        & \textbf{0.1888}\improved
        & 8.63\phantom{\improved} \\

        WRN-26-10
        & 3.36 & 0.2775 & 7.11
        & 1.15 & 0.1581 & 6.32
        & 1.19 & 0.1548 & 6.17
        & 1.34 & 0.1576 & 6.21
        & 1.16 & 0.1581 & 6.30
        & 1.22 & \textbf{0.1490} & 6.25
        & 1.13 & 0.1935 & 6.15
        & \textbf{1.02}\improved
        & 0.1905\improved
        & \textbf{6.16}\degraded \\

        DenseNet-121
        & 4.66 & 0.4390 & 9.50
        & \textbf{1.24} & 0.2211 & 8.56
        & 1.25 & 0.2128 & \textbf{8.24}
        & 2.06 & 0.2111 & 8.31
        & 1.31 & 0.2211 & 8.55
        & 1.68 & \textbf{0.2089} & 8.45
        & 1.32 & 0.2524 & 8.25
        & 1.30\improved
        & 0.2502\improved
        & 8.25\phantom{\improved} \\

        MobileNetV2-1.4$\times$
        & 3.73 & 0.2509 & 9.77
        & 1.25 & 0.2013 & 9.10
        & 1.38 & 0.2008 & 9.04
        & 1.55 & 0.1992 & 9.08
        & 1.22 & 0.2013 & 9.08
        & 1.30 & 0.1951 & 9.04
        & 1.20 & 0.1939 & 9.07
        & \textbf{1.22}\degraded
        & \textbf{0.1935}\improved
        & \textbf{9.03}\improved \\

        ShuffleNetV2-1.0$\times$
        & 4.03 & 0.2839 & 10.91
        & 1.07 & 0.2244 & 10.24
        & 1.46 & 0.2305 & 10.29
        & 1.39 & 0.2233 & 10.23
        & 1.02 & 0.2246 & 10.23
        & \textbf{1.00} & \textbf{0.2214} & \textbf{10.20}
        & 1.06 & 0.2222 & 10.24
        & 1.04\improved
        & 0.2220\improved
        & 10.23\improved \\

        RepVGG-A1
        & 3.49 & 0.2319 & 8.71
        & 1.14 & 0.1832 & 8.08
        & 1.30 & 0.1845 & 8.02
        & 1.34 & 0.1864 & 8.07
        & 1.12 & 0.1832 & 8.05
        & 1.13 & 0.1792 & 8.02
        & 1.13 & 0.1768 & 8.02
        & \textbf{1.10}\improved
        & \textbf{0.1765}\improved
        & \textbf{8.00}\improved \\

        \bottomrule
    \end{tabular}%
}
\caption{Extension of Table~4 of the main paper to NLL and Brier for the Base$^{\dagger}$ selected by standard ECE. Values are five-split means reported as ECE (\%) / NLL / Brier ($\times10^2$). Per classifier, Base$^{\dagger}$ denotes the prediction-non-preserving Base with the lowest direct-output standard ECE, and Base$^{\dagger}$+CORD its paired repair. Fit-time baselines and CORD repairs have zero TPCR; bold marks the lowest prediction-preserving value per classifier and metric, and ${\color{green!50!black}\blacktriangledown}$ and ${\color{red!75!black}\blacktriangle}$ denote paired improvement and degradation from Base$^{\dagger}$, respectively.}
\label{tab:fit-time-preservation-nll-brier}
\end{table*}

\subsection{Complete CIFAR-10 Comparisons across Base Calibrators}
\label{app:fit-time-complete-base}
\begin{table}[h!]
\renewcommand{\arraystretch}{1.0}
\centering
\setlength{\tabcolsep}{2.5pt}

\newcommand{\improved}{%
    \,{\color{green!50!black}\ensuremath{\blacktriangledown}}}
\newcommand{\degraded}{%
    \,{\color{red!75!black}\ensuremath{\blacktriangle}}}

\resizebox{1.0\linewidth}{!}{%
    \begin{tabular}{
        l
        !{\color{black!18}\vrule width 0.25pt}
        l
        !{\color{black!18}\vrule width 0.25pt}
        r >{\columncolor{gray!12}}r
        !{\color{black!18}\vrule width 0.25pt}
        r >{\columncolor{gray!12}}r
        !{\color{black!18}\vrule width 0.25pt}
        r >{\columncolor{gray!12}}r
        !{\color{black!18}\vrule width 0.25pt}
        r >{\columncolor{gray!12}}r
        !{\color{black!18}\vrule width 0.25pt}
        r >{\columncolor{gray!12}}r
        !{\color{black!18}\vrule width 0.25pt}
        r >{\columncolor{gray!12}}r
        !{\color{black!18}\vrule width 0.25pt}
        r >{\columncolor{gray!12}}r
    }
        \toprule
        \multirow{2}{*}{\textbf{Classifier}}
        & \multirow{2}{*}{\textbf{Metric (\%)}}
        & \multicolumn{2}{
            c!{\color{black!18}\vrule width 0.25pt}
          }{\textbf{VS}}
        & \multicolumn{2}{
            c!{\color{black!18}\vrule width 0.25pt}
          }{\textbf{SVS}}
        & \multicolumn{2}{
            c!{\color{black!18}\vrule width 0.25pt}
          }{\textbf{MS}}
        & \multicolumn{2}{
            c!{\color{black!18}\vrule width 0.25pt}
          }{\textbf{SMS}}
        & \multicolumn{2}{
            c!{\color{black!18}\vrule width 0.25pt}
          }{\textbf{Dir-ODIR}}
        & \multicolumn{2}{
            c!{\color{black!18}\vrule width 0.25pt}
          }{\textbf{IROvA}}
        & \multicolumn{2}{c}{\textbf{IROvA-TS}} \\
        \cmidrule(lr){3-4}
        \cmidrule(lr){5-6}
        \cmidrule(lr){7-8}
        \cmidrule(lr){9-10}
        \cmidrule(lr){11-12}
        \cmidrule(lr){13-14}
        \cmidrule(lr){15-16}
        &
        & \textbf{Base} & \textbf{+ CORD}
        & \textbf{Base} & \textbf{+ CORD}
        & \textbf{Base} & \textbf{+ CORD}
        & \textbf{Base} & \textbf{+ CORD}
        & \textbf{Base} & \textbf{+ CORD}
        & \textbf{Base} & \textbf{+ CORD}
        & \textbf{Base} & \textbf{+ CORD} \\
        \midrule

        \multirow{3}{*}{VGG-16-BN}
        & ECE
        & \textbf{1.455} & 1.413\improved
        & 1.563 & 1.609\degraded
        & 1.524 & 1.409\improved
        & 1.545 & 1.558\degraded
        & 1.799 & 1.695\improved
        & 1.703 & 1.681\improved
        & 1.589 & 1.587\improved \\
        & $\mathrm{ECE}_{\mathrm{EM}}$
        & 1.914 & 1.870\improved
        & 2.475 & 2.471\improved
        & 2.012 & 1.784\improved
        & 2.040 & 2.007\improved
        & 2.147 & 2.023\improved
        & 1.492 & 1.434\improved
        & \textbf{1.360} & 1.295\improved \\
        & smECE
        & 1.640 & 1.603\improved
        & 1.995 & 1.995\phantom{\improved}
        & 1.682 & 1.558\improved
        & 1.719 & 1.709\improved
        & 1.848 & 1.780\improved
        & 1.695 & 1.679\improved
        & \textbf{1.628} & 1.606\improved \\

        \cmidrule(lr){1-16}

        \multirow{3}{*}{ResNet-56}
        & ECE
        & 1.003 & 0.948\improved
        & \textbf{0.943} & 0.924\improved
        & 1.120 & 0.849\improved
        & 0.998 & 0.988\improved
        & 1.007 & 0.934\improved
        & 1.329 & 1.340\degraded
        & 1.110 & 1.104\improved \\
        & $\mathrm{ECE}_{\mathrm{EM}}$
        & 1.215 & 1.081\improved
        & 1.332 & 1.282\improved
        & 1.379 & 0.965\improved
        & 1.252 & 1.119\improved
        & 1.224 & 1.057\improved
        & 0.962 & 0.959\improved
        & \textbf{0.884} & 0.875\improved \\
        & smECE
        & \textbf{1.196} & 1.132\improved
        & 1.246 & 1.230\improved
        & 1.304 & 1.148\improved
        & 1.229 & 1.197\improved
        & 1.204 & 1.136\improved
        & 1.315 & 1.309\improved
        & 1.225 & 1.225\phantom{\degraded} \\

        \cmidrule(lr){1-16}

        \multirow{3}{*}{WRN-26-10}
        & ECE
        & 1.141 & 1.123\improved
        & 1.165 & 1.165\phantom{\improved}
        & 1.266 & 1.066\improved
        & 1.242 & 1.199\improved
        & 1.197 & 1.147\improved
        & 1.205 & 1.076\improved
        & \textbf{1.125} & 1.017\improved \\
        & $\mathrm{ECE}_{\mathrm{EM}}$
        & 1.533 & 1.494\improved
        & 1.824 & 1.850\degraded
        & 1.661 & 1.391\improved
        & 1.644 & 1.621\improved
        & 1.527 & 1.542\degraded
        & \textbf{0.908} & 0.878\improved
        & 0.953 & 0.930\improved \\
        & smECE
        & 1.354 & 1.325\improved
        & 1.480 & 1.478\improved
        & 1.450 & 1.257\improved
        & 1.449 & 1.423\improved
        & 1.407 & 1.391\improved
        & 1.273 & 1.161\improved
        & \textbf{1.205} & 1.161\improved \\

        \cmidrule(lr){1-16}

        \multirow{3}{*}{DenseNet-121}
        & ECE
        & 1.477 & 1.452\improved
        & 1.378 & 1.423\degraded
        & 1.506 & 1.354\improved
        & 1.569 & 1.542\improved
        & 1.560 & 1.469\improved
        & \textbf{1.320} & 1.300\improved
        & 1.337 & 1.191\improved \\
        & $\mathrm{ECE}_{\mathrm{EM}}$
        & 1.754 & 1.712\improved
        & 2.512 & 2.508\improved
        & 1.757 & 1.537\improved
        & 1.853 & 1.909\degraded
        & 1.897 & 1.867\improved
        & 1.280 & 1.232\improved
        & \textbf{1.190} & 1.139\improved \\
        & smECE
        & 1.460 & 1.459\improved
        & 1.896 & 1.845\improved
        & 1.547 & 1.355\improved
        & 1.569 & 1.551\improved
        & 1.543 & 1.514\improved
        & 1.418 & 1.322\improved
        & \textbf{1.373} & 1.270\improved \\

        \cmidrule(lr){1-16}

        \multirow{3}{*}{MobileNetV2-1.4$\times$}
        & ECE
        & 1.224 & 1.201\improved
        & 1.309 & 1.259\improved
        & 1.501 & 1.162\improved
        & \textbf{1.198} & 1.220\degraded
        & 1.416 & 1.321\improved
        & 1.446 & 1.471\degraded
        & 1.344 & 1.304\improved \\
        & $\mathrm{ECE}_{\mathrm{EM}}$
        & 1.434 & 1.270\improved
        & 1.541 & 1.467\improved
        & 1.698 & 1.193\improved
        & 1.485 & 1.339\improved
        & 1.603 & 1.338\improved
        & 1.331 & 1.329\improved
        & \textbf{1.126} & 1.183\degraded \\
        & smECE
        & \textbf{1.357} & 1.345\improved
        & 1.470 & 1.466\improved
        & 1.568 & 1.382\improved
        & 1.429 & 1.411\improved
        & 1.526 & 1.429\improved
        & 1.541 & 1.531\improved
        & 1.398 & 1.365\improved \\

        \cmidrule(lr){1-16}

        \multirow{3}{*}{ShuffleNetV2-1.0$\times$}
        & ECE
        & \textbf{1.058} & 1.035\improved
        & 1.101 & 1.084\improved
        & 1.225 & 1.129\improved
        & 1.095 & 1.076\improved
        & 1.110 & 1.163\degraded
        & 1.518 & 1.404\improved
        & 1.204 & 1.170\improved \\
        & $\mathrm{ECE}_{\mathrm{EM}}$
        & 1.085 & 1.105\degraded
        & 1.190 & 1.163\improved
        & 1.232 & 1.191\improved
        & 1.087 & 1.119\degraded
        & 1.114 & 1.099\improved
        & 1.244 & 1.231\improved
        & \textbf{1.007} & 1.046\degraded \\
        & smECE
        & 1.352 & 1.346\improved
        & 1.326 & 1.300\improved
        & 1.389 & 1.315\improved
        & 1.290 & 1.279\improved
        & 1.304 & 1.282\improved
        & 1.506 & 1.497\improved
        & \textbf{1.273} & 1.252\improved \\

        \cmidrule(lr){1-16}

        \multirow{3}{*}{RepVGG-A1}
        & ECE
        & 1.177 & 1.134\improved
        & 1.155 & 1.140\improved
        & 1.321 & 1.128\improved
        & \textbf{1.134} & 1.097\improved
        & 1.250 & 1.004\improved
        & 1.488 & 1.435\improved
        & 1.285 & 1.244\improved \\
        & $\mathrm{ECE}_{\mathrm{EM}}$
        & 1.384 & 1.269\improved
        & 1.431 & 1.381\improved
        & 1.528 & 1.146\improved
        & 1.368 & 1.256\improved
        & 1.553 & 1.234\improved
        & 1.143 & 1.067\improved
        & \textbf{0.981} & 0.949\improved \\
        & smECE
        & 1.282 & 1.217\improved
        & 1.287 & 1.284\improved
        & 1.391 & 1.239\improved
        & \textbf{1.266} & 1.231\improved
        & 1.359 & 1.219\improved
        & 1.463 & 1.406\improved
        & 1.344 & 1.303\improved \\

        \bottomrule
    \end{tabular}%
}
\caption{Complete CIFAR-10 comparisons across Base calibrators
underlying Table~4 of the main paper. Base and Base + CORD
entries are absolute five-split means (\%); bold identifies the
metric-specific Base$^{\dagger}$, and
${\color{green!50!black}\blacktriangledown}$ and
${\color{red!75!black}\blacktriangle}$ denote paired improvement
and degradation from the corresponding Base, respectively.}
\label{tab:fit-time-complete-base}
\end{table}

\end{document}